\pdfoutput=1
\documentclass{article}
\PassOptionsToPackage{numbers, compress}{natbib}

\usepackage[final]{neurips_2022}

\usepackage{hyperref}       %
\usepackage{url}            %
\usepackage{booktabs}       %
\usepackage{amsfonts}       %
\usepackage{nicefrac}       %
\usepackage{microtype}      %

\usepackage{amsmath,amsfonts,bm}

\def\1{\mathds{1}}

\def\vv{{\bm{v}}}

\def\vx{{\bm{x}}}

\def\mI{{\bm{I}}}

\DeclareMathAlphabet{\mathsfit}{\encodingdefault}{\sfdefault}{m}{sl}
\SetMathAlphabet{\mathsfit}{bold}{\encodingdefault}{\sfdefault}{bx}{n}

\newcommand{\E}{\mathbb{E}}

\newcommand{\R}{\mathbb{R}}

\DeclareMathOperator*{\argmax}{arg\,max}
\DeclareMathOperator*{\argmin}{arg\,min}

\usepackage[utf8]{inputenc} %
\usepackage{lipsum} %

\usepackage[T1]{fontenc}

\usepackage{etoolbox}
\newbool{includeappendix}
\setbool{includeappendix}{true}

\ifdefined\isoverfull
\else
\fi

\usepackage{listings}

\usepackage{textcomp}

\usepackage{xcolor}

\usepackage[scaled=0.8]{beramono}

\definecolor{ckeyword}{HTML}{7F0055}
\definecolor{ccomment}{HTML}{3F7F5F}
\definecolor{cstring}{HTML}{2A0099}

\lstdefinestyle{numbers}{
	numbers=left,
	framexleftmargin=20pt,
	numberstyle=\tiny,
	firstnumber=auto,
	numbersep=1em,
	xleftmargin=2em
}

\lstdefinestyle{layout}{
	frame=none,
	captionpos=b,
}

\lstdefinestyle{comment-style}{
	morecomment=[l]//,
	morecomment=[s]{/*}{*/},
	commentstyle={\color{ccomment}\itshape},
}

\lstdefinestyle{string-style}{
	morestring=[b]",%
	morestring=[b]',%
	stringstyle={\color{cstring}},
	showstringspaces=false,%
}

\lstdefinestyle{keyword-style}{
	keywordstyle={\ttfamily\bfseries},
	morekeywords={
		function,
		constructor,
		int,
		bool,
		return,
		returns,
		uint
	},
	morekeywords = [2]{},
	keywordstyle = [2]{\text},
	sensitive=true,
}

\lstdefinestyle{input-encoding}{
	inputencoding=utf8,
	extendedchars=true,
	literate=
	{ℝ}{$\reals$}1%
	{→}{$\rightarrow$}1%
	{α}{$\alpha$}1%
	{β}{$\beta$}1%
	{λ}{$\lambda$}1%
	{θ}{$\theta$}1%
	{ϕ}{$\phi$}1%
}

\lstdefinestyle{escaping}{
	moredelim={**[is][\color{blue}]{\%}{\%}},
	escapechar=|,
	mathescape=true
}

\lstdefinestyle{default-style}{
	basicstyle=\fontencoding{T1}\ttfamily\footnotesize,
	style=numbers,
	style=layout,
	style=comment-style,
	style=string-style,
	style=keyword-style,
	style=input-encoding,
	style=escaping,
	tabsize=2,
	upquote=true
}

\lstdefinelanguage{BASIC}{
	language=C++,
	style=default-style
}[keywords,comments,strings]%

\usepackage[capitalize]{cleveref}

\crefrangeformat{section}{\S#3#1#4\crefrangeconjunction\S#5#2#6}

\crefmultiformat{section}{\S#2#1#3}{\crefpairconjunction\S#2#1#3}{\crefmiddleconjunction\S#2#1#3}{\creflastconjunction\S#2#1#3}

\newcommand{\crefrangeconjunction}{--}

\crefname{listing}{Lst.}{listings}
\crefname{line}{Lin.}{Lin.}
\crefname{appendix}{App.}{App.}

\newcommand{\app}[1]{%
	\ifbool{includeappendix}{\cref{#1}}{the appendix}%
}
\newcommand{\App}[1]{%
	\ifbool{includeappendix}{\cref{#1}}{The appendix}%
}

\usepackage{tikz}

\usetikzlibrary{arrows}
\usetikzlibrary{automata}
\usetikzlibrary{calc}
\usetikzlibrary{backgrounds}
\usetikzlibrary{decorations.markings}
\usetikzlibrary{decorations.pathmorphing}
\usetikzlibrary{decorations.pathreplacing}
\usetikzlibrary{fit}
\usetikzlibrary{patterns}
\usetikzlibrary{positioning}
\usetikzlibrary{shadows}
\usetikzlibrary{shapes}
\usetikzlibrary{shapes.geometric}

\usepackage{algorithm}
\usepackage[noend]{algpseudocode} %
\usepackage{xspace}
\usepackage{amsfonts}
\usepackage{amssymb}
\usepackage{dsfont}
\usepackage{wrapfig}
\usepackage{graphicx}
\usepackage{varwidth}
\usepackage{subcaption}
\usepackage{caption}
\usepackage{multirow, makecell}
\usepackage{thmtools}

\usepackage{amsthm}

\theoremstyle{plain} %

\newtheorem{theorem}{Theorem}[section]
\newtheorem*{theorem*}{Theorem}
\newtheorem{lemma}{Lemma}
\newtheorem*{lemma*}{Lemma}

\newtheorem*{corollary*}{Corollary}

\newtheorem{proposition*}{Proposition*}
\theoremstyle{definition}

\allowdisplaybreaks
\renewcommand{\paragraph}[1]{\textbf{#1}$\:$}

\newcommand{\bb}[1]{\mathbb{#1}}
\newcommand{\bs}[1]{\mathds{#1}}
\newcommand{\bc}[1]{\mathcal{#1}}
\newcommand{\Z}{\bb{Z}}
\renewcommand{\P}{\bb{P}}

\newcommand{\N}{\bc{N}}

\newcommand{\imp}{H}

\newcommand{\dist}{\phi}

\newcommand{\fmnist}{\textsc{FMNIST}\xspace}

\newcommand{\diabetes}{\textsc{Diabetes}\xspace}
\newcommand{\breast}{\textsc{BreastCancer}\xspace}
\newcommand{\fmnists}{\fmnist-\textsc{Shoes}\xspace}
\newcommand{\mnistof}{\textsc{MNIST} $1$ \textsc{vs.} $5$\xspace}
\newcommand{\mnistts}{\textsc{MNIST} $2$ \textsc{vs.} $6$\xspace}
\newcommand{\adult}{\textsc{Adult}\xspace}
\newcommand{\credit}{\textsc{Credit}\xspace}
\newcommand{\mammal}{\textsc{Mammo}\xspace}
\newcommand{\bank}{\textsc{Bank}\xspace}
\newcommand{\mushroom}{\textsc{Mushroom}\xspace}
\newcommand{\spambase}{\textsc{Spambase}\xspace}

\newcommand{\tool}{\textsc{DRS}\xspace}
\newcommand{\RS}{\textsc{RS}\xspace}

\newcommand{\treeboost}{\textsc{TreeBoost}\xspace}
\newcommand{\robtreeboost}{\textsc{RobTreeBoost}\xspace}
\newcommand{\ada}{\textsc{AdaBoost}\xspace}
\newcommand{\robada}{\textsc{RobAdaBoost}\xspace}
\renewcommand{\th}{$^\text{th}$\xspace}

\title{(De-)Randomized Smoothing for \\ Decision Stump Ensembles}

\newcommand*\samethanks[1][\value{footnote}]{\footnotemark[#1]}
\author{%
Miklós Z. Horváth\thanks{Equal contribution},\, Mark Niklas Müller\samethanks,\, Marc Fischer,\, Martin Vechev\\
Department of Computer Science\\
ETH Zurich\\
Switzerland\\
\texttt{mihorvat@ethz.ch}, \texttt{\{mark.mueller,marc.fischer,martin.vechev\}@inf.ethz.ch}
}

\begin{document}

\maketitle

\begin{abstract}

Tree-based models are used in many high-stakes application domains such as finance and medicine, where robustness and interpretability are of utmost importance. Yet, methods for improving and certifying their robustness are severely under-explored, in contrast to those focusing on neural networks. Targeting this important challenge, we propose deterministic smoothing for decision stump ensembles. Whereas most prior work on randomized smoothing focuses on evaluating arbitrary base models approximately under input randomization, the key insight of our work is that decision stump ensembles enable exact yet efficient evaluation via dynamic programming. Importantly, we obtain deterministic robustness certificates, even jointly over numerical and categorical features, a setting ubiquitous in the real world. Further, we derive an MLE-optimal training method for smoothed decision stumps under randomization and propose two boosting approaches to improve their provable robustness. An extensive experimental evaluation on computer vision and tabular data tasks shows that our approach yields significantly higher certified accuracies than the state-of-the-art for tree-based models. We release all code and trained models at \url{https://github.com/eth-sri/drs}.
\end{abstract}

\section{Introduction}
Tree-based models have long been a favourite for making decisions in high-stakes domains such as medicine and finance, due to their interpretability and exceptional performance on tabular data \cite{SHWARTZZIV202284}. However, recent results have highlighted that tree-based models are, similarly to other machine learning models \citep{BiggioCMNSLGR13, szegedy2013intriguing}, also highly susceptible to adversarial examples \citep{ChenZBH19, CartellaAFYAE21, MathovLKSE22}, raising concerns about their use in high-stakes domains where errors can have dire consequences.

While the robustness of neural models has received considerable attention \citep{singh2019abstract,xu2020automatic, GehrMDTCV18, WangPWYJ18, WengZCSHDBD18, WongK18, singh2018fast,muller2021prima,tjeng2017evaluating,dathathri2020enabling,Ehlers17,MirmanGV18, BalunovicV20, RaghunathanSL18b, GowalDSBQUAMK18}, the challenge of obtaining robustness guarantees for ensembles of tree-based models has only been investigated recently \citep{ChenZBH19, Andriushchenko019, WangZCBH20}. However, these initial works only consider numerical features and are based on worst-case approximations, which do not scale well to the difficult $\ell_p$-norm setting.

\paragraph{This Work} In this work, we address this challenge and present \tool, a novel (\textbf{D}e-)\textbf{R}andomized \textbf{S}moothing approach, for constructing robust tree-based models with deterministic $\ell_p$-norm guarantees while supporting both categorical \emph{and} numerical variables. Unlike prior work, our method is based on Randomized Smoothing (\RS) \citep{CohenRK19}, an approach that obtains robustness guarantees by evaluating a general base model under an input randomization $\dist(\vx)$. However, in contrast to standard applications of \RS, which use costly and imprecise approximations via sampling and only obtain probabilistic certificates, we leverage the structure of decision stump ensembles to compute their exact output distributions for a given input randomization scheme and thus obtain deterministic certificates. Our key insight is that this distribution can be efficiently computed by aggregating independent distributions associated with the individual features used by the ensemble.

We illustrate this idea in \cref{fig:overview}: In (a), we show an ensemble of decision stumps over three features ($x_1, x_2, x_3$), aggregated to piecewise constant functions over one feature each (discussed in \cref{sec:det_smoothing}) and evaluated under the input randomization $\dist(\vx)$, here a Gaussian. We can compute the independent probability density functions of their outputs (PDFs) (shown in (b)) directly, by evaluating the (Gaussian) cumulative density function (CDF) over the constant regions. Aggregating the individual PDFs (discussed in \cref{sec:det_smoothing}), we can efficiently compute the exact PDF (c) and CDF (d) of the ensemble's output. To evaluate and certify the smoothed model, we can now simply look up the median prediction and success probability, respectively, in the CDF, without requiring sampling.

\tool combines $\ell_p$-norm certificates over numerical features, computed as described above, with an efficient worst-case analysis for $\ell_0$-perturbations of categorical features in order to, for the first time, provide joint certificates. To train models amenable to certification with \tool, we propose a robust MLE optimality criterion for training individual stumps and two boosting schemes targeting the certified robustness of the whole ensemble.
We show empirically that \tool significantly improves on the state-of-the-art, increasing certified accuracies on established benchmarks up to \emph{4-fold}.

\textbf{Main Contributions} Our key contributions are:
\vspace{-2mm}
\begin{itemize}
	\item \tool, a novel and efficient (De-)Randomized Smoothing approach for robustness certification, enabling joint deterministic certificates over numerical and categorical variables (\cref{sec:det_smoothing}).
	\item A novel MLE optimality criterion for training decision stumps robust under input randomization and two boosting approaches for certifiably robust stump ensembles (\cref{sec:training}).
	\item An extensive empirical evaluation, demonstrating the effectiveness of our approach and establishing a new state-of-the-art in a wide range of settings (\cref{sec:eval}).
\end{itemize}
\vspace{-2mm}
\begin{figure}[t]
\centering
\scalebox{1.00}{
\begin{tikzpicture}
	\def\linesep{-1.3}
	\def\colsep{0.5}
	\def\lineheight{3em}

	\foreach \i in {0,...,2}
	{
		\node (line\i) at (0,\i*\linesep) {};
	}

	\foreach \i in {1,...,3}
	{
		\pgfmathtruncatemacro{\j}{\i - 1};
		\node[inner sep=0pt, right=-0.6+0*\colsep of line\j] (fx\j) {$\tilde{f}_\i(x_\i)$};
		\node[inner sep=0pt, right=\colsep of line\j] (d\j) {\includegraphics[height=\lineheight]{figures/overview_dist_feature_\j}};
		\node[inner sep=0pt, right=-0.2  + \colsep of d\j ] (p\j) {\includegraphics[height=\lineheight]{figures/overview_pdf_feature_\j}};
	}

 	\node[right=1.6*\colsep of p0.north] (brace0) {};
 	\node[right=1.6*\colsep of p2.south] (brace1) {};
 	\draw [decorate,decoration={brace,amplitude=10pt},yshift=0pt] ($(brace0)+(0,0.0)$) -- ($(brace1)+(0,-0.1)$) node [black,midway,xshift=0.8cm] {};

  	\node[inner sep=0pt, anchor=south west] (pdf_final) at ($(p2.south east)+(2*\colsep,-0.35)$) {\includegraphics[width=.3\textwidth]{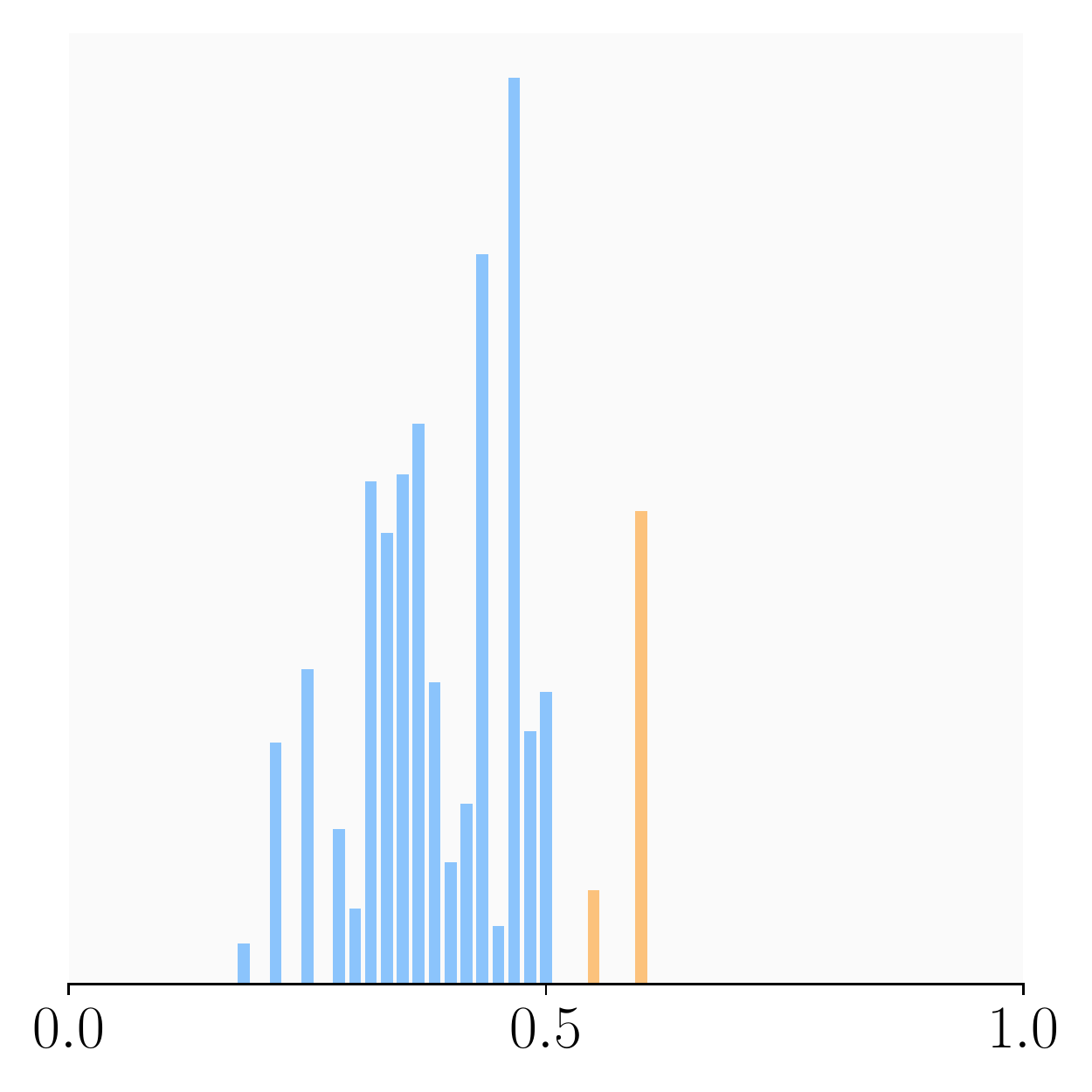}};
  	\node[inner sep=0pt, right=\colsep of pdf_final.south east, anchor=south west] (cdf_final) {\includegraphics[width=.3\textwidth]{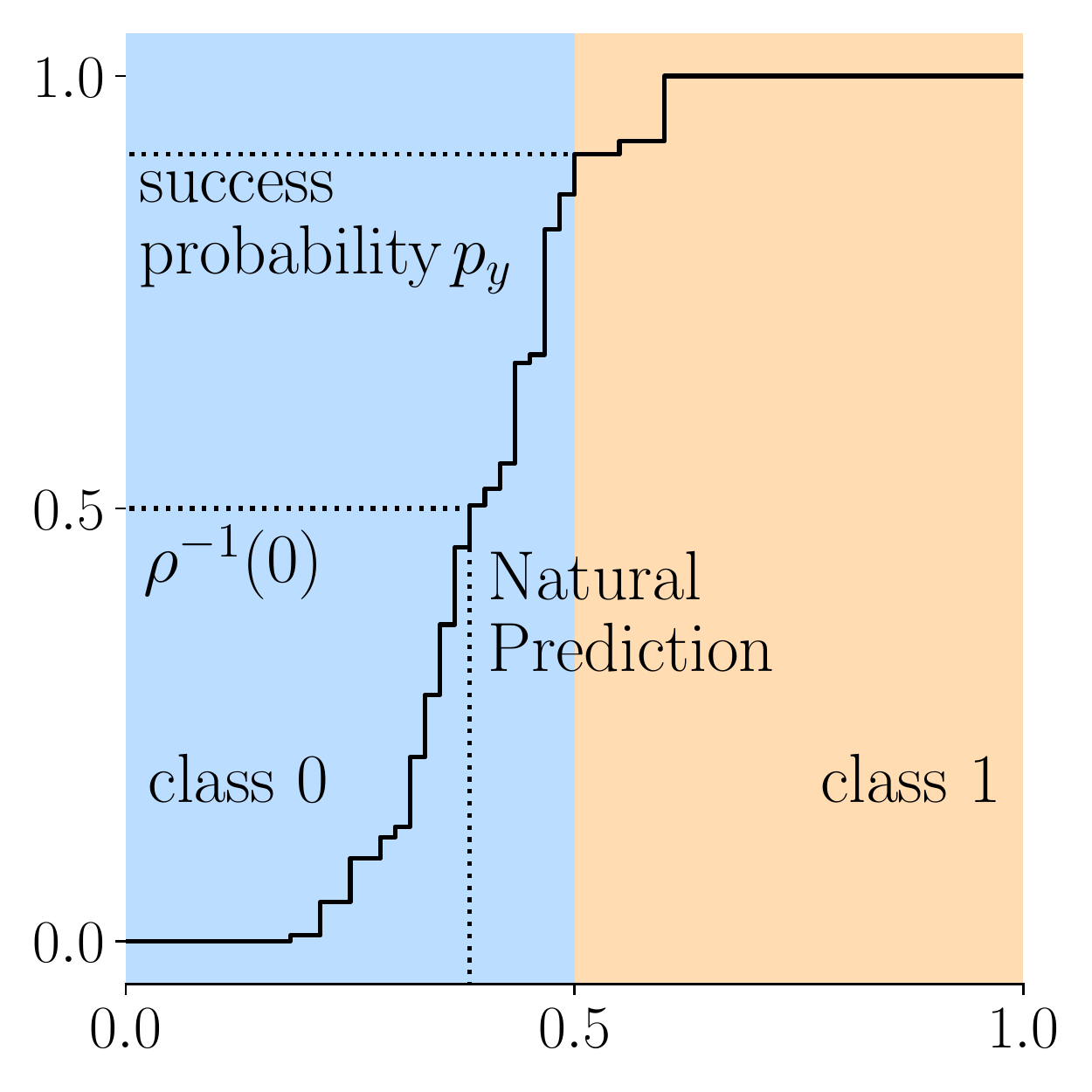}};

	\node[anchor=north] (a) at ($(d2.south)+(-0.5,-0.2)$) {(a) Meta-Stumps};
	\node[] (b) at (a -| p2) {(b) PDFs};
	\node[] (c) at (a -| pdf_final) {(c) Ensemble PDF};
	\node[] (d) at (a -| cdf_final) {(d) Ensemble CDF};

\end{tikzpicture}}
	\caption{Given an ensemble of $3$ meta-stumps $\tilde{f}_i$ (piecewise constant univariate functions), each operating on a different feature $x_i$ of an input $\vx$, we calculate the probability of every output under input randomization (a) to obtain a distribution over their outputs (b). 
	We aggregate these individual PDFs via dynamic programming to obtain the probability distribution over the ensemble's outputs (c). 
	We can then compute the corresponding CDF (d) to evaluate the smoothed stump ensemble exactly.
	}
	\label{fig:overview}
\end{figure}

\vspace{-1mm}
\section{Background on Randomized Smoothing} \label{sec:background}
\vspace{-1.5mm}
For a given base model $F \colon \R^d \to [C]$, classifying inputs to one of $C \in \Z^{\geq 2}$ classes, Randomized Smoothing (\RS) is a method to construct a classifier $G \colon \R^d \to [C]$ with robustness guarantees. For a randomization scheme $\dist \colon \R^d \!\to\! \R^d\!$, we define the success probability $p_y \!:=\! \P_{\vx' \sim \dist(\vx)}[F(\vx') \!= \!y]$ and $G(\vx) \! := \! \argmax_{c \in [C]} p_y $. Depending on the choice of $\dist$, we obtain different certificates of the form:
\begin{theorem}[Adapted from \citet{CohenRK19,YangDHSR020}] \label{thm:l1}
	If
	$\P(F(\dist(\vx)) = y) := p_y
	\geq
	\underline{p_y}$ and
	$\underline{p_y}> 0.5$,
	then $G(\vx + \delta) = y$ for all $\delta$ satisfying $\|\delta\|_p < R$
	with $R := \rho(\underline{p_y})$.
\end{theorem}
\begin{wrapfigure}[7]{r}{0.42\textwidth}
	\vspace{-0mm}
	\centering
	\captionof{table}{\footnotesize Randomized Smoothing guarantees.} \label{tab:rs}
	\vspace{-1mm}
	\begin{tabular}{@{}cll@{}}
		\toprule
		& $\dist(\vx)$ & $R := \rho(\underline{p_y})$\\
		\midrule
		$\ell_1$ & $\vx + \textit{Unif}([-\lambda, \lambda]^d)$ & $2\lambda (\underline{p_y} - \tfrac{1}{2})$\\[1.2ex]
		$\ell_2$ & $\vx + \N(\mathbf{0}, \sigma \boldsymbol{\mathds{I}})$  & $\sigma\Phi^{-1}(\underline{p_y})$\\
		\bottomrule
	\end{tabular}
\end{wrapfigure}
In particular, we present two instantiations that we utilize throughout this paper in \cref{tab:rs}, where $\Phi^{-1}$ is the inverse Gaussian CDF.
Similar results, yielding other $\ell_p$-norm certificates, can be derived for a wide range of input randomization schemes \citep{YangDHSR020,ZhangYGZ020}. Note that, by using more information than just $\underline{p_y}$, e.g., $p_{c}$ for the runner-up class $c$, tighter certificates can be obtained \citep{CohenRK19,DvijothamHBKQGX20}.
Once $\underline{p_y}$ is computed, we can directly calculate the certifiable radius $R := \rho(p_y)$.
For a broader overview of variants of Randomized Smoothing, please refer to \cref{sec:related}.

For most choices of $F$ and $\dist$, the exact success probability $p_y$ can not be computed efficiently. Thus a lower bound $\underline{p_y}$ is estimated with confidence $1-\alpha$ (typically $\alpha=10^{-3}$) using Monte Carlo sampling and the Neyman-Pearson lemma \citep{neyman1933ix}. Not only is this extremely computationally expensive, as typically $100\,000$ samples have to be evaluated per data point, but this also severely limits the maximum certifiable radius (see \cref{fig:ablation-rs-ds}) and only yields probabilistic guarantees. Additionally, if the number of samples is not sufficient for the statistical test, the procedure will abstain from classifying.

In the following, we will show how considering a specific class of models $F$ allows us to compute the success probability $p_y$ exactly, overcoming these drawbacks, and thus invoke $\rho(p_y)$ to compute deterministic certificates over larger radii, orders of magnitude faster than \RS.

\vspace{-1.0mm}
\section{(De-)Randomized Smoothing for Decision Stump Ensembles} \label{sec:det_smoothing}
\vspace{-1.5mm}
Tree-based models such as decision stump ensembles often combine exceptional performance on tabular data \cite{SHWARTZZIV202284} with good interpretability, making them ideal for many real-world high-stakes applications. Here, we propose a (De-)Randomized Smoothing approach, \tool, to equip them with deterministic robustness guarantees.
For this, we first revisit decision stump ensembles and then show that their structure permits an exact evaluation under isotropic input randomization schemes, such as those discussed in \cref{sec:background}.
Finally, we propose joint certification over numerical and categorical variables, as many practical tabular datasets often contain both variable types. 

\paragraph{Stump Ensembles}
We define a decision stump as $f_m(\vx) = \gamma_{l,m} + (\gamma_{r,m}-\gamma_{l,m}) \bs{1}_{x_{j_m}>v_m}$, with leaf predictions $\gamma_{l,m}, \gamma_{r,m} \in [0,1]$, split position $v_m$, and split variable $j_m$.
We construct unweighted ensembles, particularly suitable for Smoothing \citep{Horvath2022Boosting}, of $M$ such stumps $\bar{f}_M \colon \R^d \mapsto [0,1]$ as
\begin{equation}
\bar{f}_M(\vx) := \frac{1}{M} \sum_{m=1}^M f_m(\vx), \label{eq:ens}
\end{equation}
and treat them as a binary classifiers $\1_{\bar{f}_M(\vx) > 0.5}$. While our approach is extensible to multi-class classification by replacing the scalar leaf predictions $\gamma$ with prediction-vectors, assigning a score per class, we focus on the binary case in this work.

\paragraph{Smoothed Stump Ensemble}
We now define a smoothed stump ensemble $\bar{g}_M$ along the lines of Randomized Smoothing as discussed in \cref{sec:background}, by evaluating $\bar{f}_M$ not only on the original input $\vx$ but rather on a whole distribution of $\vx' \sim \dist(\vx)$:\\
\begin{wrapfigure}[15]{r}{0.35 \textwidth}
	\vspace{-10mm}
	\centering
	\includegraphics[width=0.97\linewidth]{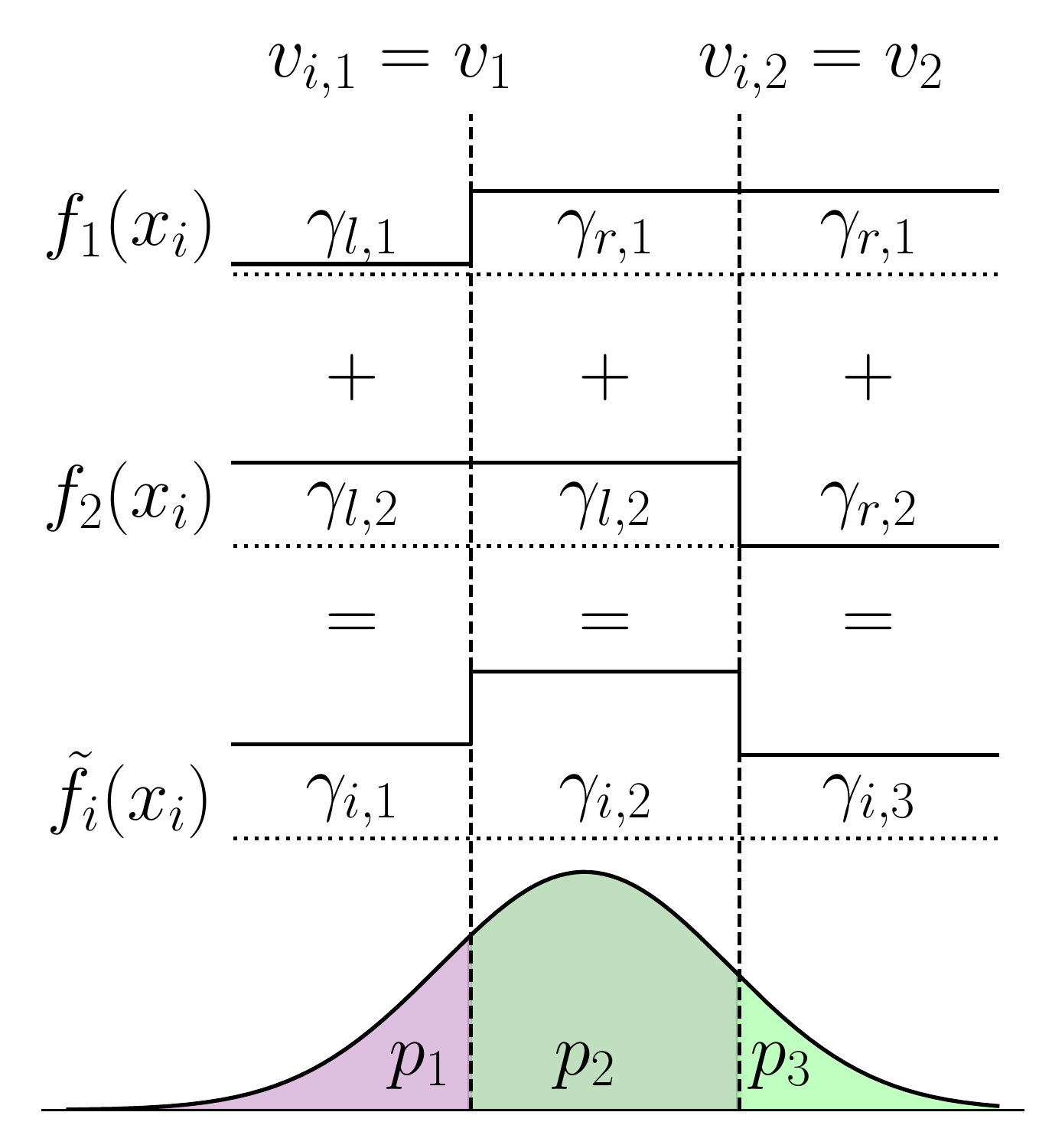}
	\vspace{-1mm}
	\caption{ A meta-stump constructed from two stumps.} \label{fig:metastump}
\end{wrapfigure}
\vspace{-6mm}
\begin{equation*}
\bar{g}_M(\vx) := \P_{\vx' \sim \dist(\vx)}[\bar{f}_M(\vx') > 0.5].
\end{equation*}
In this work, we consider randomization schemes $\dist(\vx)$ that are (i) isotropic, i.e., the dimensions of $\vx' \sim \dist(\vx)$ are independently distributed, and (ii) permit an efficient computation of their marginal cumulative distribution functions (CDF).
This includes a wide range of distributions, e.g., the Gaussian and Uniform distributions used in \cref{tab:rs} and others commonly used for \RS \citep{YangDHSR020}. 

By denoting the model CDF as $\bar{\bc{F}}_{M,\vx}(z) = \P_{\vx' \sim \dist(\vx)}[\bar{f}(\vx') \leq z]$, we can alternatively define $\bar{g}_M$ as
	$\bar{g}_M(\vx) := 1 - \bar{\bc{F}}_{M,\vx}(0.5)$,
which will become useful later.
For a label $y \in \{0, 1\}$ we obtain the success probability $p_y = |y - \bar{\bc{F}}_{M,\vx}(0.5)|$ of predicting $y$ for a sample from $\dist(\vx)$.

\paragraph{Meta-Stumps}
To evaluate $p_y$ exactly as illustrated in \cref{fig:overview}, we group the stumps constituting an ensemble by their split variable $j_m$ to obtain one \emph{meta-stump} $\tilde{f}_i$ per feature $i$. The key idea is that outputs of these meta-stumps are now independently distributed under isotropic input randomization (illustrated in \cref{fig:overview} (b)), allowing us to aggregate them efficiently later on.

We showcase this in \cref{fig:metastump}, where two stumps ($f_1$ and $f_2$) are combined into the meta-stump $\tilde{f}_i$. Formally, we have
\begin{equation}\label{eqn:meta_stump}
	\tilde{f}_i(\vx) := \sum_{m \in \bc{I}_i} f_m(\vx), \qquad \bc{I}_i := \{m \in [M] \mid j_m = i\}, 
\end{equation}
define $M_i = |\bc{I}_i|$ and rewrite our ensemble as
	$\bar{f}_M(\vx) := \frac{1}{M} \sum_{i=1}^d \tilde{f}_i(x_i)$.
Every meta-stump can be represented by its split positions $v_{i,j}$, sorted such that $v_{i,j} \leq v_{i,j+1}$, and its predictions $\gamma_{i,j} = \sum_{t=1}^{j-1} \gamma_{r,t} + \sum_{t=j}^{|\bc{I}_i|} \gamma_{l,m}$ on each of the resulting $|\bc{I}_i| + 1$ regions, written as $(\boldsymbol{\gamma}, \vv)_i$. 

\vspace{-1mm}
\paragraph{CDF Computation}
Now we leverage the independence of our meta-stumps' output distributions under an isotropic input randomization scheme $\dist$ to compute the PDF of their ensemble efficiently via dynamic programming (DP) (illustrated in \cref{fig:overview} (c) and explained below). Given its PDF, we can trivially compute the ensemble's CDF $\bar{\bc{F}}_{M,\vx}$, allowing us to evaluate the smoothed model exactly (illustrated in \cref{fig:overview} (d)). This efficient CDF computation constitutes the core of \tool.

\begin{figure}
	\begin{algorithm}[H]
		\caption{Stump Ensemble PDF computation via Dynamic Programming}
		\label{alg:dp}
		\begin{algorithmic}
			\vspace{-1mm}
			\Function{ComputePDF}{$\{(\boldsymbol{\Gamma}, \vv)_i\}_{i=1}^d, \vx, \dist$}
			\State $\texttt{pdf}[i][t]=0 \text{ for } t \in [M \cdot \Delta + 1 ], i \in [d]$
			\State $\texttt{pdf}[0][0] = 1$ \Comment{For 0 stumps all probability mass is on $0$}
			\For{$i=1$ to $d$}
			\For{$j = 1$ to $M_i$}
			\For{$t = 0$ to $M \cdot \Delta + 1 - \Gamma_{i,j}$}
			\State $\texttt{pdf}[i][t+\Gamma_{i,j}] = \texttt{pdf}[i][t+\Gamma_{i,j}] + \texttt{pdf}[i-1][t] \cdot \P_{x'_i \sim \dist(\vx)}[v_{i,j-1}  < x'_i \leq v_{i,j}]$
			\EndFor 
			\EndFor 
			\EndFor
			\State \textbf{return} $\texttt{pdf}$
			\EndFunction
			\vspace{-1mm}
		\end{algorithmic}
	\end{algorithm}
	\vspace{-5mm}
\end{figure}

In more detail, we observe that the PDF of a stump ensemble is the convex sum of exponentially many ($\bc{O}( (\max_i \bc{I}_i)^d)$) Dirac-delta distributions. 
To avoid this exponential blow-up, we discretize all leaf predictions $\gamma$ to a grid of $\Delta$ values (typically $\Delta = 100$), when constructing the smoothed model $\bar{g}_M$. 
For each $\gamma_{i,j}$, we define a corresponding $\Gamma_{i,j} \in \{0, \dots, M_i\cdot\Delta \}$ such that  $\gamma_{i,j} = \frac{\Gamma_{i,j}}{\Delta}$. 
Now, we construct a DP-table, where every entry $\texttt{pdf[i][t]}$ corresponds to the weight of the Dirac-delta associated with an output of $\tfrac{t}{\Delta M}$ after considering the first $i$ meta-stumps (in any arbitrary but fixed order).
We show the PDF computation in \cref{alg:dp} and provide an intuition below. We initialize $\texttt{pdf[0][*]}$ by allocating all probability mass to $\texttt{t}=0$ ($\texttt{pdf[0][0]=1}$). 
Now, we compute $\texttt{pdf[i][*]}$ from $\texttt{pdf[i-1][*]}$ by accounting for the effect of the $i$\th meta-stump as follows: 
The weight of the Dirac-delta at $\texttt{t}$ after considering $\texttt{i}$ meta-stumps is exactly the sum over the weights of the Dirac-deltas at $\texttt{t-}\Gamma_{i,j}$ after $\texttt{i-1}$ meta-stumps, weighted with the probability $p_{i,j}:=\P_{x'_i \sim \dist(\vx)}[v_{i,j-1} < x'_i \leq v_{i,j}]$ of the $\texttt{i}$\th meta stump predicting $\Gamma_{i,j}$. 
We compute $p_{i,j}$ as the probability of the randomized $x'_i$ lying between $v_{i,j-1}$ and $v_{i,j}$ (padded with $-\infty$ and $\infty$ on the left and right, respectively), as illustrated in \cref{fig:metastump}. 
After termination, the last line of the DP-table $\texttt{pdf[d][*]}$ contains the full PDF (see \cref{fig:overview}(c)). Formally we summarize this in the theorem below, delaying a formal proof to \cref{sec:pdf_computation}:

\begin{restatable}{theorem}{correctness}
	\label{thm:correctness}
	For $z \in [0, 1]$, $\bar{\bc{F}}_{M,\vx}(z) = \sum_{t=0}^{ \lfloor z M\Delta \rfloor  } \texttt{pdf}[d][t]$ describes the exact CDF and thus success probability $p_y = \P_{\vx' \sim \dist(\vx)}[\bar{f}_M(\vx') = y] = |y - \bar{\bc{F}}_{M,\vx}(0.5)|$ for $y \in \{ 0, 1\}$.
\end{restatable}
\vspace{-1mm}

Note that the presented algorithm is slightly simplified, and we actually only have to track the range of non-zero entries of one row of the DP-table. %
This allows us to compute the full PDF and thus certificates for smoothed stump ensembles very efficiently, e.g., taking only around $1.2$ s total for the \mnistts task (around $2.000$ data points and over $500$ stumps).

\paragraph{Certification}
Recall from \cref{sec:background} that, given the success probability $p_y$, robustness certification for $\ell_p$-norm  bounded perturbations
reduces to computing the maximal certifiable robustness radius $R = \rho(p_y)$.
For all popular $\ell_p$-norms, $\rho$ (and its inverse $\rho^{-1}$; used shortly) can be either evaluated symbolically \citep{CohenRK19,YangDHSR020} or precomputed efficiently \citep{LeeYCJ19,BojchevskiKG20}, such that the core challenge of certification becomes computing (a lower bound to) $p_y$, which we solve efficiently via \cref{thm:correctness}.
Alternatively, for a given target radius $r$, we need to check whether $p_y \geq \rho^{-1}(r)$ by equivalently calculating 
\begin{equation}\label{eqn:cert_g}
\bar{g}_{M,r}(\vx) = \bar{\bc{F}}_{M,\vx}^{-1}(z) \qquad\qquad z = \begin{cases} 1 - \rho^{-1}(r) \quad &\text{ if } y = 1\\
\rho^{-1}(r) & \text{ if } y = 0
\end{cases},
\end{equation}
and checking $\bar{g}_{M,r}(\vx)>0.5$. This corresponds to asserting that class $y$ is predicted at least $z$ of the time.
Here, the inverse CDF $\bar{\bc{F}}_{M,\vx}^{-1}(z)$ can be efficiently evaluated using the step-wise $\bar{\bc{F}}_{M,\vx}$ computed via \cref{thm:correctness}.
We will see in \cref{sec:training} that this view is useful when training stump ensembles for certifiability. Finally, we want to highlight that this approach can be used with all common randomization schemes yielding certificates for different $\ell_p$-norm bounded adversaries.

\paragraph{Categorical Variables \& Joint Certificates}
For practical applications, it is essential to handle both numerical and categorical features jointly.
To consider a categorical feature $x_i \in \{1, \dots, d_i\}$ in our stump ensemble, we construct a $d_i$-ary stump $\tilde{f}_i \colon [d_i] \to [0, 1]$ returning a value $\gamma_{i,j}$ corresponding to each of the $d_i$ categorical values and treated as a meta-stump with $M_i = 1$ for normalization.

To provide certificates in this setting, we propose a novel scheme combining an arbitrary $\ell_p$-norm certificate of radius $r_p$ over all numerical features, computed as discussed above, with an $\ell_0$ certificate of radius $r_0$ over all categorical features $\bc{C}$, computed using an approach adapted from \citet{WangZCBH20}. 
Conceptually, we compute the worst-case effect of every individual categorical variable independently, greedily aggregate these worst-case effects, and account for them in our ensemble's CDF.

Given a meta-stump's prediction on a concrete sample $q_i = \tilde{f}_i(x_i)$ as well as its maximal and minimal output $u_i$ and $l_i$, respectively, we compute the maximum and minimum perturbation effect to $\overline{\delta}_i = \frac{u_i - q_i}{M}$ and $\underline{\delta}_i = \frac{l_i - q_i}{M}$, respectively. Given the set of categorical features $\bc{C}$, we can compute the worst-case effect when perturbing at most $r_0$ samples as
\begin{equation*}
	\overline{\delta}_{r_0} = \max_\bc{R} \sum_{i \in \bc{R}} \overline{\delta}_i, \quad s.t. |\bc{R}| \leq r_0, \bc{R} \subseteq \bc{C}
\end{equation*}
\begin{wrapfigure}[10]{r}{0.344 \textwidth}
	\vspace{-6mm}
	\centering
	\includegraphics[width=1.0\linewidth]{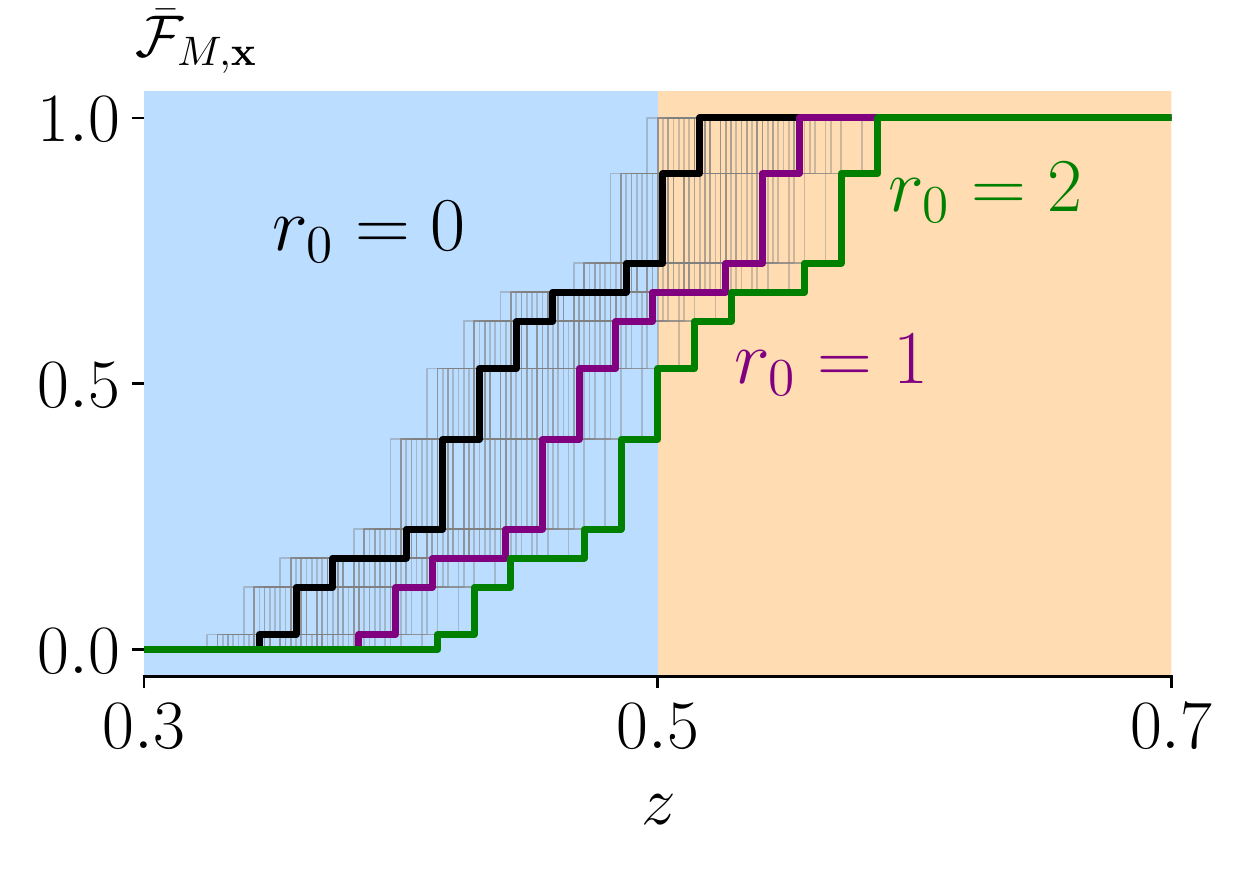}
	\vspace{-6mm}
	\caption{ CDF shifted by the effect of categorical feature perturbations.} \label{fig:categorical}
\end{wrapfigure}
by greedily picking the $r_0$ largest $\overline{\delta}_i$. For $\underline{\delta}_{r_0}$ we proceed analogously.
Shifting the CDF, computed as above, by $\overline{\delta}$ and $\underline{\delta}$ for samples with labels $y=0$ and $y=1$, respectively, before computing the success probability $p_y$, allows us to account for the worst-case categorical perturbations exactly.
We illustrate this for a sample with $y=0$ in \cref{fig:categorical}, where we show the CDFs obtained by all possible perturbations of at most $r_0$ categorical variables, bounded to the right by those obtained by shifting the original by $\overline{\delta}_{r_0}$.
Note that here no smoothing over the categorical variables is done or required, making inference trivial.

\section{Training for and with (De-)Randomized Smoothing} \label{sec:training}
To obtain large certified radii via smoothing, the base model has to be robust to the chosen randomization scheme.
To train robust decision stump ensembles, we propose a robust MLE optimality criterion for individual stumps (\cref{sec:training_indp}) and two boosting schemes for whole ensembles (\cref{sec:training_boosted}).

\subsection{Independently MLE-Optimal Stumps} \label{sec:training_indp}
To train an individual stump $f_m(\vx) = \gamma_{l,m} + (\gamma_{r,m}-\gamma_{l,m}) \bs{1}_{x_{j_m}>v_m}$, its split feature $j_m$, split position $v_m$, and leaf predictions $\gamma_{l,m}, \gamma_{r,m}$ have to be determined. 
We choose them in an MLE-optimal fashion with respect to the randomization scheme $\dist$, starting with $v_m$, as follows:
We consider the probabilities $p_{l,i}(v_m) = \P_{\vx' \sim \dist(\vx_i)}[x'_{j_m} \leq v_m]$ and $p_{r,i} = 1 - p_{l,i}(v_m)$ of $\vx'_i$ lying to the left or the right of $v_m$, respectively, under the input randomization scheme $\dist$. To avoid clutter, we drop the explicit dependence on $v_m$ in the following. For an i.i.d. dataset with $n$ samples $(\vx_i, y_i) \sim (\bc{X},\bc{Y})$, we define the probabilities  $p^y_{j} = \frac{1}{n}\sum_{\{i | y_i=y\}} p_{j,i}$ of picking the $j \in \{l, r\}$ leaf, conditioned on the target label, and $p_j = p_j^0 +p_j^1$ as their sum to compute the entropy impurity $\imp_\text{entropy}$ \citep{BustosKSSV04} as
\begin{align*}
	\imp_\text{entropy} &= - \sum_{j \in \{l,r\}} p_{j}  \sum_{y \in \{0,1\}}  \frac{p^y_{j}}{p_{j}} \log\left(\frac{p^y_{j}}{p_{j}}\right).
\end{align*}
We then choose the $v_m$ approximately minimizing $\imp_\text{entropy}$ via line-search.
After fixing $v_m$ this way, we compute the MLE-optimal leaf predictions $\gamma_l^{\dist,\text{MLE}}$ and $\gamma_r^{\dist,\text{MLE}}$ as:
\begin{align*}
\gamma_l^{\dist\text{MLE}}, \gamma_r^{\dist\text{MLE}} &= \argmax_{\gamma_l, \gamma_r} \P[\bc{Y} \mid \dist(\bc{X}), f_m] = \argmax_{\gamma_l, \gamma_r} \sum_{i=1}^{n} \E_{\vx' \sim \dist(\vx_i)} \left[\log \P[y_i \mid \vx', f_m]\right]\\
&= \argmax_{\gamma_l, \gamma_r} \sum_{i \in \{i \mid y_i = 0\}}^{n} p_{l,i}  \log(1 - \gamma_l) + p_{r,i}  \log(1 - \gamma_r) \\
& \qquad\quad\;\;\,+ \sum_{i \in \{i \mid y_i = 1\}}^{n} p_{l,i} \log(\gamma_l) + p_{r,i} \log(\gamma_r) \\
&= \argmax_{\gamma_l, \gamma_r} p^0_{l}  \log(1 - \gamma_l) + p^0_{r}  \log(1 - \gamma_r) + p^1_{l} \log(\gamma_l) + p^1_{r} \log(\gamma_r),
\end{align*}
where the second line is obtained by splitting the sum over samples by class and explicitly computing the expectation. %
We solve the maximization problem by setting the first derivatives $\frac{\partial}{\partial \gamma_l}$ and $\frac{\partial}{\partial \gamma_r}$ of our optimization objective to zero and checking its Hessian to confirm that
\begin{equation}
	\gamma_l^{\dist\text{MLE}} = \frac{p^1_{l}}{p^1_{l} + p^0_{l}} \qquad \gamma_r^{\dist\text{MLE}} = \frac{p^1_{r}}{p^1_{r} + p^0_{r}}
\end{equation}
are indeed maxima.
We show in \cref{app:mle_proof} that $\gamma_l^{\dist,\text{MLE}}$, $\gamma_l^{\dist,\text{MLE}}$, and $v_m$ are even jointly MLE-optimal, when $v_m$ is chosen as the exact instead of an approximate minimizer of the entropy impurity.%

\paragraph{Ensembling}
To train an ensemble of independently MLE-optimal decision stumps, we sequentially train one stump for every feature $j_m \in [d]$ and construct an ensemble with equal weights, rejecting stumps with an entropy impurity $\imp_\text{entropy}$ above a predetermined threshold.

\subsection{Boosting Stump Ensembles for Certifiable Robustness} \label{sec:training_boosted}
\begin{wrapfigure}[15]{r}{0.40 \textwidth}
	\vspace{2.5mm}
\centering
	\includegraphics[width=1.0\linewidth]{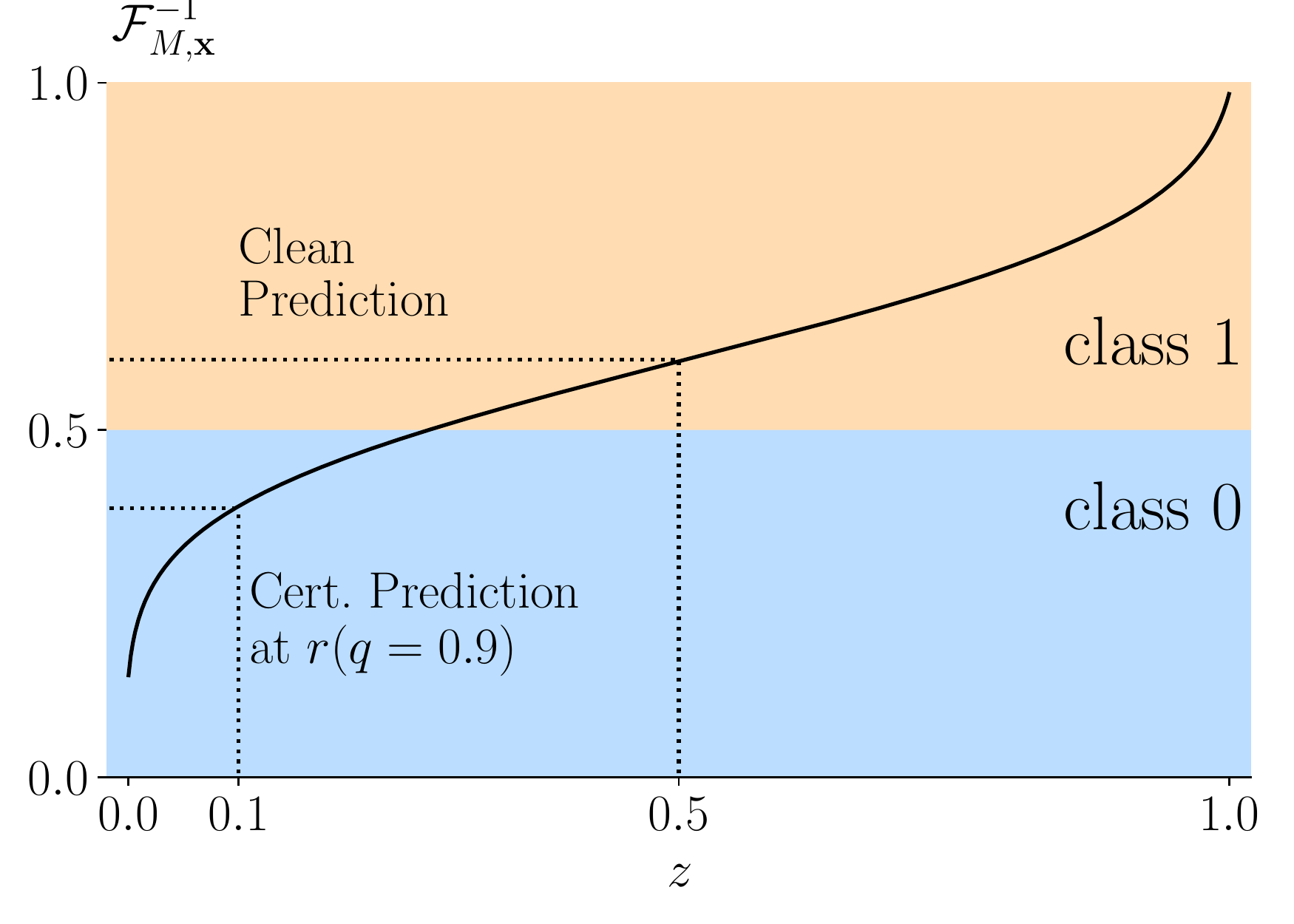}
	\vspace{-7mm}
 	\caption{\footnotesize Inverse CDF $\bar{\bc{F}}_M^{-1}$} \label{fig:cdf_loss}
	\vspace{-0mm}	

\end{wrapfigure}

Decision stumps trained this way maximize the expected likelihood under the chosen randomization scheme.
Assuming (due to the law of large numbers) a Gaussian output distribution, this corresponds to optimizing for the median output, which determines the clean prediction. 
However, certified correctness at a given radius $r$ is determined by the prediction $y'(\vx,r) = \bar{\bc{F}}_{m-1,\vx}^{-1}(z(r))$ at the  $z(r) := |y - \rho^{-1}(r)|$ percentile of the output distribution. Where we call $y'$ the \emph{certifiable prediction}, as certification is now equivalent to checking $y = \1_{y'(\vx,r) > 0.5}$ (\cref{eqn:cert_g}).
This difference is illustrated in \cref{fig:cdf_loss}, where the clean prediction is correct (class 1) while the certifiable prediction is incorrect.
To align our training objective better with certified accuracy, we propose two novel boosting schemes along the lines of the popular \treeboost \citep{Friedman2001Greedy} and \ada \citep{FreundS97}.

\paragraph{Gradient Boosting for Certifiable Robustness}
The key idea of gradient boosting is to compute the gradient of a loss function with respect to an ensemble's outputs and then add a model to the ensemble that makes a prediction along this gradient direction.
Implementing this idea, we adapt \treeboost \citep{Friedman2001Greedy} to propose \robtreeboost: At a high level, we add stumps to the ensemble, which aim to predict the residual between the target label and the current certifiable prediction.
Concretely, to add the $m$\th stump to our ensemble, we begin by computing the current ensemble's certifiable predictions $y'(r)$ at a target radius $r$ and then defining the pseudo labels $\tilde{y} = y - y'(r)$ as the residual between the target labels $y$ and the certifiable predictions $y'(r)$.
This yields a regression problem, which we tackle by choosing a feature $j_m$ and split threshold $v_m$ (approximately) minimizing MSE impurity under input randomization before computing $\gamma_{l,m}$ and $\gamma_{r,m}$ as approximate minimizers of the cross-entropy loss over the whole ensemble.
Please see \cref{sec:tree_boost_extra} for a more detailed discussion of \robtreeboost.

\paragraph{Adaptive Boosting for Certifiable Robustness}
The key idea of adaptive boosting is to build an ensemble by iteratively training models, weighted based on their error rate, while adapting sample weights based on whether they are classified correctly. We build on \ada \citep{FreundS97} to propose \robada: We construct an ensemble of $K$ stump ensembles via hard voting, where every ensemble is weighted based on its certifiable accuracy. To train a new ensemble, we increase the weights of all samples that are currently not classified certifiably correctly at a given radius $r$. We choose stump ensembles instead of individual stumps as base classifiers because single stumps often can not reach the success probabilities under input randomization required for certification.
To compute the certifiable radius for such an ensemble $\bar{F}_K$, we compute the certifiable radii $R^k$ of the individual stump ensembles $\bar{f}_M^k$, sort them in decreasing order such that $R^k \geq R^{k+1}$ and obtain the largest radius $R^k$ such that the weights of the first $k$ ensembles sum up to more than half of the total weights.
Please see \cref{sec:ada_boost_extra} for a more detailed discussion of \robada.

\vspace{-1.0mm}
\section{Experimental Evaluation}
\vspace{-1.5mm}
\label{sec:eval}

In this section, we empirically demonstrate the effectiveness of \tool in a wide range of settings. 
We show that \tool significantly outperforms the current state-of-the-art for certifying tree-based models on established benchmarks, using only numerical features (\cref{sec:eval-ds-for-stumps}), before highlighting its novel ability to obtain joint certificates on a set of new benchmarks (\cref{sec:eval-joint-certification}).
Finally, we perform an ablation study, investigating the effect of \tool's key components (\cref{sec:eval-ablation}).%

\paragraph{Experimental Setup}
We implement our approach in PyTorch \cite{PaszkeGMLBCKLGA19} and evaluate it on Intel Xeon Gold 6242 CPUs and an NVIDIA RTX 2080Ti. 
We compare to prior work on the \diabetes \cite{Smith1988UsingTA}, \breast \cite{Dua:2019}, \fmnists \cite{DBLP:journals/corr/abs-1708-07747}, \mnistof \cite{lecun2010mnist}, and \mnistts \cite{lecun2010mnist} datasets and are the first to provide joint certificates of categorical and numerical features, demonstrated on the \adult \cite{Dua:2019} and \credit \cite{Dua:2019} datasets.
For a more detailed description of the experimental setup, please refer to \cref{app:experimental-details}.

\begin{table}[tp]
	\centering
	\small
	\centering
	\caption{Natural accuracy (NAC) $[\%]$ and certified accuracy (CA) $[\%]$ with respect to $\ell_1$- and $\ell_2$-norm bounded perturbations. Results for \citet{WangZCBH20} as reported by them. Larger is better.}
	\vspace{1.5mm}
	\label{tab:baseline-both}
	\resizebox{0.99\columnwidth}{!}{
	\begin{tabular}{cccccccccc}
		\toprule
        \multirow{2.5}{*}{Perturbation} & \multirow{2.5}{*}{Dataset} &  \multirow{2.5}{*}{Radius $r$} & \textbf{Standard Training} & \multicolumn{2}{c}{\textbf{\citet{WangZCBH20}}} & \multicolumn{2}{c}{\textbf{Ours (Independent)}} & \multicolumn{2}{c}{\textbf{Ours (Boosting)}}\\
        \cmidrule(lr){4-4} \cmidrule(lr){5-6} \cmidrule(lr){7-8} \cmidrule(lr){9-10}
        & & & NAC & NAC & CA & NAC & CA & NAC & CA \\
        \midrule
		\multirow{5}{*}{$\ell_1$-norm}    & \breast & $1.0$  & 99.3 & 98.5 & 64.2 & \textbf{100.0}$\;\:$ & {81.0} & \textbf{100.0}$\;\:$ & \textbf{83.9} \\ %
                                     &\diabetes& $0.05$ & 74.7 & 72.7 & 68.2 & {76.0} & {69.5} & \textbf{77.9} & \textbf{72.1} \\ %
                                     &\fmnists & $0.5$  & \textbf{95.0} & 87.6 & 67.8 & 85.8 & 83.3 & 87.2 & \textbf{84.2} \\ %
                                     &\mnistof & $1.0$  & 99.1 & 95.5 & 83.8 & 96.6 & 94.1 & \textbf{99.3} & \textbf{98.1} \\ %
                                     &\mnistts & $1.0$  & 96.0 & 92.3 & 66.5 & 96.3 & 93.9 & \textbf{96.6} & \textbf{94.1} \\ %
     	\midrule
     	\multirow{5}{*}{$\ell_2$-norm}    & \breast & $0.7$  & 99.3 & 91.2 & 60.6 & \textbf{100.0}$\;\:$ & 75.2 & \textbf{100.0}$\;\:$ & \textbf{82.5} \\ %
                                     &\diabetes& $0.05$ & 74.7 & -       & -       & {77.3} & {68.2} & \textbf{79.9} & \textbf{71.4} \\ %
                                     &\fmnists & $0.4$  & \textbf{95.0} & 75.5 & 51.5 & 86.8 & 81.2 & 91.0 & \textbf{84.5} \\ %
                                     &\mnistof & $0.8$  & 99.1 & 95.6 & 63.4 & 95.8 & 91.6 & \textbf{99.2} & \textbf{96.3}\\ %
                                     &\mnistts & $0.8$  & 96.0 &  86.3 & 23.0 & \textbf{96.3} & \textbf{89.6} & \textbf{96.3} & \textbf{89.6}\\ %
		\bottomrule
	\end{tabular}
}
\vspace{-4mm}
\end{table}

\vspace{-1.0mm}
\subsection{Certification for Numerical Features} \label{sec:eval-ds-for-stumps}
\vspace{-1.0mm}
In \cref{tab:baseline-both}, we compare the certified accuracies obtained via \tool on ensembles of independently MLE optimal stumps (Independent) or boosted stump ensembles (Boosting) to the current state-of-the-art, \citet{WangZCBH20}, and standard training \citep{DBLP:journals/corr/abs-1201-0490} using established benchmarks \citep{WangZCBH20}. %

\paragraph{Independently MLE Optimal Stumps}
We first consider stump ensembles trained without boosting as described in \cref{sec:training_indp} and observe that \tool obtains higher certified accuracies in all settings and higher natural accuracies in most.
For example, on \mnistts, we increase the certified accuracy at an $\ell_2$ radius of $r_2=0.8$ from $23.0\%$ to $89.6\%$, almost quadrupling it compared to \citet{WangZCBH20}, while also improving natural accuracy from $86.3\%$ to $96.3\%$.

\paragraph{Boosting for Certified Accuracy}
Leveraging the boosting techniques introduced in \cref{sec:training_boosted}, \robtreeboost for \breast and \diabetes and \robada for \fmnists, \mnistof, and \mnistts, we increase certifiable and natural accuracies even further in most settings.
For example, compared to our independently trained stump ensemble, we improve the certified accuracy for \mnistof at an $\ell_1$-radius of $r_1 = 1.0$ from $94.1\%$ to $98.1\%$ and for \breast at an $\ell_2$-radius of $r_2 = 0.7$ from $75.2\%$ to $82.5\%$.

\vspace{-1.0mm}
\subsection{Joint Certificates for Categorical and Numerical Features}\label{sec:eval-joint-certification}
\vspace{-1.0mm}

\begin{wrapfigure}[13]{r}{0.42\textwidth}
	\centering
	\vspace{-5.0mm}
	\includegraphics[width=0.95\linewidth]{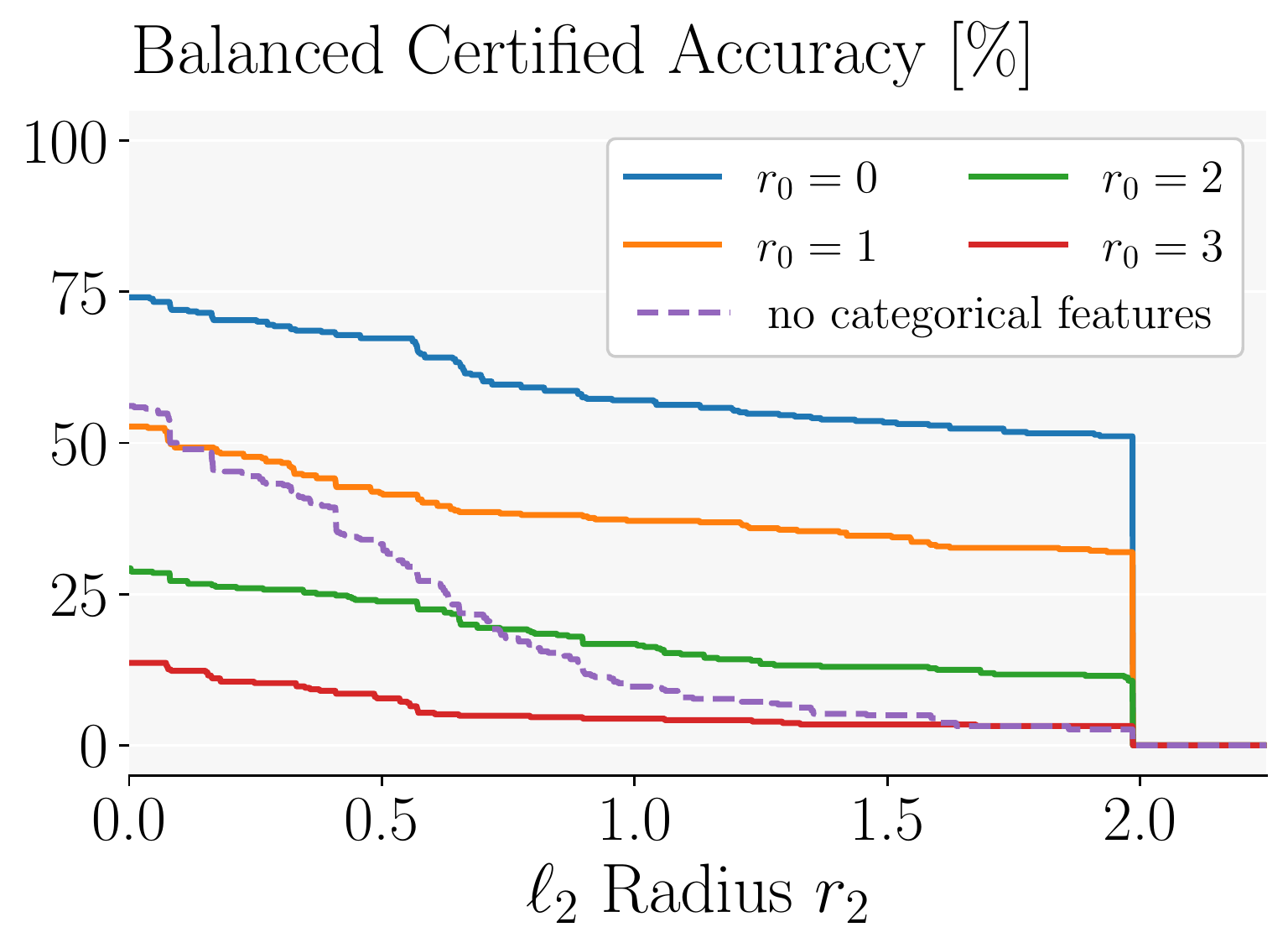}
	\vspace{-2.5mm}
	\caption{Effect of $\ell_0$-perturbations on $\ell_2$-robustness for \credit.}
	\label{fig:joint-ca}
\end{wrapfigure}
In \cref{tab:joint-balanced}, we compare models using only numerical, only categorical, or both types of features with regards to their balanced certified accuracy (BCA) (accounting for class frequency) at different combinations of $\ell_2$- and $\ell_0$-radii for numerical and categorical features, respectively.
We observe that models using both categorical and numerical features perform notably better on clean data, highlighting the importance of utilizing and thus also certifying them in combination.
Moreover, categorical features make the model significantly more robust to $\ell_2$ perturbations, e.g., at $\ell_2$-radii $\geq 0.75$, they improve certified accuracies, even when $2$  categorical features (of only $8$ and $7$ for \adult and \credit, respectively) are adversarially perturbed. 
We visualize this in \cref{fig:joint-ca}, showing BCA over $\ell_2$-perturbation radius and confirm that the model utilizing only numerical features (dotted line) loses accuracy much quicker with perturbation magnitude than the model leveraging categorical variables (solid lines). As we are the first to tackle this setting, we do not compare to other methods but provide more detailed experiments in \cref{app:eval-joint-robustness}.

\begin{table}[tp]
	\centering
	\small
	\centering
	\caption{Balanced certified accuracy (BCA) $[\%]$ under joint $\ell_0$- and $\ell_2$-perturbations of categorical and numerical features, respectively, depending on whether model uses categorical and/or numerical features. The balanced natural accuracy is the BCA at radius $r = 0.0$. Larger is better.}
	\label{tab:joint-balanced}
	\vspace{3mm}
   	\renewcommand{\arraystretch}{1.03}
	\resizebox{0.99\columnwidth}{!}{
    \begin{tabular}{ccccccccccc}
        \toprule
        \multirow{2.6}{*}{Dataset} & \multirowcell{2.6}{Categorical \\ Features} & \multirowcell{2.6}{$\ell_0$ Radius $r_0$} &\multirowcell{2.6}{ BCA without \\ Numerical Features} & \multicolumn{7}{c}{BCA with Numerical Features at $\ell_2$ Radius $r_2$}\\
        \cmidrule(lr){5-11}
        & & & & 0.00 & 0.25 & 0.50 & 0.75 & 1.00 & 1.25 & 1.50 \\
        \midrule
        \multirow{5.5}{*}{\adult} & no & - & - &74.9 & 65.7 & 42.4$\;\:$ & 27.4$\;\:$ & 14.5$\;\:$ & 8.9 & 5.1 \\
        \cmidrule(lr){2-2}
        & \multirow{4.0}{*}{yes} & 0 & 76.6 & 77.5 & 73.9 & 68.1$\;\:$ & 63.3$\;\:$ & 48.7$\;\:$ & 40.7$\;\:$ & 35.2$\;\:$ \\
                             &   & 1 & 57.4 &  66.0 & 61.7 & 53.9$\;\:$ & 47.4$\;\:$ & 34.3$\;\:$ & 26.6$\;\:$ & 21.8$\;\:$  \\
                             &   & 2 & 33.5 &  51.4 & 46.2 & 37.5$\;\:$ & 29.3$\;\:$ & 21.5$\;\:$ & 17.1$\;\:$ & 13.4$\;\:$  \\
                             &   & 3 & 8.9  &  36.7 & 31.4 & 24.1$\;\:$ & 15.4$\;\:$ & 10.3$\;\:$ & 8.1 & 5.7  \\
        \midrule
        \multirow{5.5}{*}{\credit} & no & - & - & 56.1 & 44.5 & 33.3 & 17.7$\;\:$ & 9.7 & 7.2 & 5.0 \\
        \cmidrule(lr){2-2}
        & \multirow{4.0}{*}{yes} & 0 & 70.7 & 74.1 & 70.3 & 67.3$\;\:$ & 59.7$\;\:$ & 57.1$\;\:$ & 54.9$\;\:$ & 53.4$\;\:$ \\
                               & & 1 & 48.2 & 52.7 & 47.7 & 41.7$\;\:$ & 38.3$\;\:$ & 37.1$\;\:$ & 35.1$\;\:$ & 34.7$\;\:$  \\
                               & & 2 & 26.4 & 29.3 & 26.0 & 23.8$\;\:$ & 19.2$\;\:$ & 16.8$\;\:$ & 13.5$\;\:$ & 13.0$\;\:$  \\
                               & & 3 & 7.8  & 13.6 & 10.3 & 7.8 & 4.9 & 4.4 & 3.9 & 3.4  \\
		\bottomrule
	\end{tabular}
}
\vspace{-2mm}
\end{table}

\vspace{-1.0mm}
\subsection{Ablation Study}
\vspace{-1.0mm}
\label{sec:eval-ablation}
We first illustrate the effectiveness of our derandomization approach, before demonstrating the benefit of training with our MLE optimality criterion and investigating the effect of the noise level on \tool.

\begin{wrapfigure}[13]{r}{0.40\textwidth}
	\centering
	\vspace{-4.5mm}
	\includegraphics[width=0.92\linewidth]{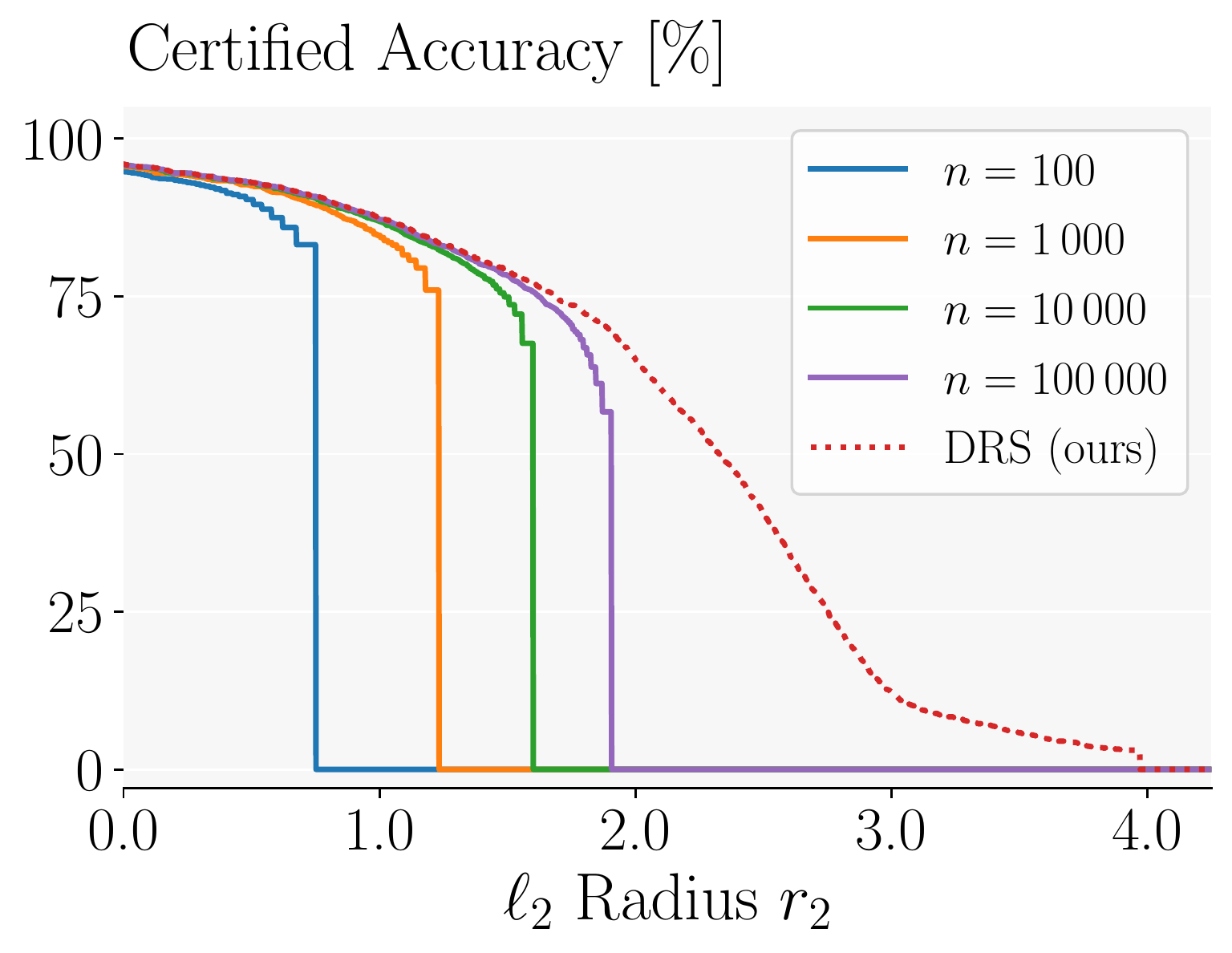}
	\vspace{-2.0mm}
	\caption{ \tool vs. \RS with various sample counts $n$ on \mnistof.}
	\label{fig:ablation-rs-ds}
\end{wrapfigure}
\paragraph{(De-)Randomized vs Randomized Smoothing}
In \cref{fig:ablation-rs-ds}, we compare \tool, (dotted line) and sampling-based \RS (solid lines), w.r.t. certified accuracy over $\ell_2$ radii.
We observe that the sampling-based estimation of the success probability in \RS significantly limits the obtained certifiable radii.
While this effect is particularly pronounced for small sample counts $n$, increasing the maximum certifiable radius, visible as the sudden drop in certifiable accuracy, requires an exponentially increasing number of samples, making the certification of large radii intractable.
\tool, in contrast, can compute exact success probabilities and thus deterministic guarantees for much larger radii, yielding a $33.1\%$ increase in ACR compared to using $n=100\,000$ samples. 
Additionally, \tool is multiple orders of magnitude faster than \RS, here, only requiring approximately $6.45 \cdot 10^{-4}$ s per sample.
For more extensive experiments, please refer to \cref{app:eval-ds-rs}.

\begin{wraptable}[12]{r}{0.47 \textwidth}
	\centering
	\small
	\vspace{3.0mm}
	\caption{Comparison of training with the exact distribution (MLE), randomly perturbed data (Sampling), or clean data (Default) on \breast for $\sigma=1$.}
	\vspace{-1.0mm}
	\label{tab:l2-mle-ablation}
	\resizebox{1.00\linewidth}{!}{
    \begin{tabular}{cccccc}
        \toprule
        \multirow{2.6}{*}{Method} & \multirow{2.6}{*}{ACR} & \multicolumn{4}{c}{Certified Accuracy $[\%]$ at Radius $r$}\\
        \cmidrule(lr){3-6}
        & & 0.0 & 0.25 & 0.5 & 0.75 \\
        \midrule
        MLE (Ours) & \textbf{0.675} & \textbf{100.0} & \textbf{97.1} & \textbf{86.1} & \textbf{30.7}  \\
        Sampling & 0.567 & 99.3 & 95.6 & 75.2 & 8.8 \\
        Default & 0.356 & 26.3 & 25.5 & 25.5 & 25.5 \\
        \midrule
	\end{tabular}
}
\end{wraptable}

\paragraph{MLE Optimality Criterion}
In \cref{tab:l2-mle-ablation}, we evaluate our robust MLE optimality criterion (MLE) by comparing it to the standard entropy criterion applied to samples drawn from the input randomization scheme (Sampling) or the clean data (Default).
We observe that the ensemble trained on the clean data (Default) suffers from a mode collapse when evaluated under noise.
In contrast, both approaches considering the input randomization perform much better, with our robust MLE approach outperforming sampling by a significant margin, especially at large radii.
For more extensive experiments, please refer to \cref{app:eval-mle-ablation}.

\begin{wrapfigure}[12]{r}{0.47\textwidth}
	\centering
	\vspace{-6.5mm}
	\includegraphics[width=0.92\linewidth]{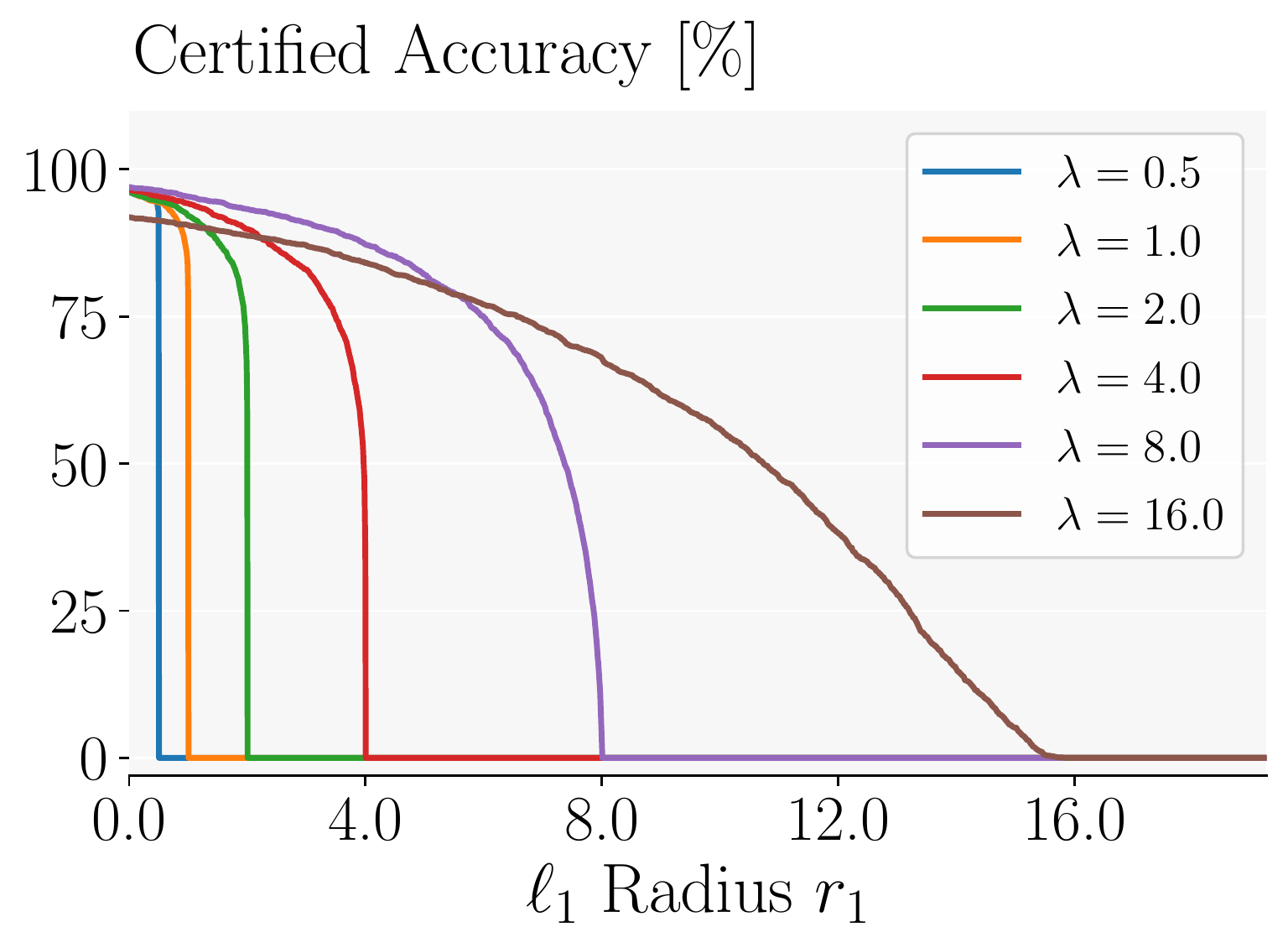}
	\vspace{-2.2mm}
	\caption{\footnotesize Comparing \tool for various noise levels $\lambda$ on \mnistof.}
	\label{fig:ablation-lambda-l1}
\end{wrapfigure}
\paragraph{Effect of Noise Level}
In \cref{fig:ablation-lambda-l1}, we compare the certified accuracy over $\ell_1$-radii for a range of different noise magnitudes $\lambda$ and ensembles of independently MLE optimal stumps.
We observe that at large perturbation magnitudes, we obtain stumps that `think outside the (hyper-)box', i.e., choose splits outside of the original data range, making their ensembles exceptionally robust, even at large radii.
In particular, we obtain a certifiable accuracy of $87.3\%$ at radius $r_1=4.0$, while the state-of-the-art achieves only $83.8\%$ at $r_1=1.0$ \citep{WangZCBH20}.
We provide more experiments for varying noise magnitudes in \cref{app:eval-noise-level}.

\section{Related Work} \label{sec:related}

\paragraph{(De-)Randomized Smoothing}
Probabilistic certification methods \citep{LiCWC19, LecuyerAG0J19,CohenRK19} are a popular approach for obtaining robustness certificates for a wide range of tasks \citep{BojchevskiKG20,GaoHG20,ChiangCAK0G20,FischerBV21, JiaCWG20}, threat models \citep{YangDHSR020,ZhangYGZ020,DvijothamHBKQGX20,LeeYCJ19,BojchevskiKG20,GaoHG20,FischerBV20,li2021tss, WangJCG21, schuchardt2021collective,LevineF20,Levine20DeRand,Levine21Improved}, and robustness-accuracy trade-offs \cite{horvath2022robust}.
These methods follow the general blueprint discussed in \cref{sec:background} and consider arbitrary base classifiers, though specially trained \citep{jeong2020consistency,zhai2020macer,salman2019provably}.
While recent work \citep{Horvath2022Boosting, DBLP:conf/iclr/YangLXK0L22} has found ensembles to be particularly suitable base classifiers, they use neural networks and can thus, in contrast to our work, not leverage their structure.
Specifically designed randomization schemes \citep{LevineF20,Levine21Improved} enable efficient enumeration and thus a deterministic certificate for, e.g., patch attacks or $\ell_1$-norm perturbations.
In contrast to these approaches, we permit arbitrary isotropic continuous randomization schemes, allowing us to leverage comprehensive results on \RS to obtain robustness guarantees against a wide range of $\ell_p$-norm bounded adversaries \cite{YangDHSR020}.

\paragraph{Certification and Training of Tree-Based Models}
In the setting of $\ell_\infty$ robustness, where every feature can be perturbed independently, various methods have been proposed to train \cite{ChenZBH19,Andriushchenko019,CalzavaraLTAO20,VosV21GROOT,pmlr-v162-guo22h} and certify \cite{Andriushchenko019,ChenZS0BH19,RanzatoZ20,TornblomN19} robust decision trees and stumps.
However, $\ell_\infty$ robust models are still vulnerable to other $\ell_p$ norm perturbations \citep{SchottRBB19,TramerB19}, which cover many realistic perturbations better and are the focus of this work.
There, the admissible perturbation of one feature depends on the perturbations of all others, making the above approaches leveraging their independence not applicable.

On the other hand, \citet{KantchelianTJ16} discuss complete robustness certification of tree ensembles in the $\ell_p$-norm setting via MILP.
However, this approach is intractable in most settings due to its Co-NP-complete complexity.
\citet{WangZCBH20} propose an efficient but incomplete DP-based certification algorithm for stump ensembles based on over-approximating the maximum perturbation effect in the $\ell_p$-norm setting. While similarly fast as our approach, we show empirically in \cref{sec:eval} that DRS obtains significantly stronger certificates. 
\citet{WangZCBH20} further introduce an incomplete certification algorithm for tree ensembles, which is based on computing the distance between the pre-image of all trees' leaves and the original sample. As they report significantly worse results using this approach than with stump ensembles, we omit a detailed comparison.

\section{Limitations and Societal Impact}
\label{sec:limitations-and-impact}

\paragraph{Limitations}
While able to handle arbitrary stump ensembles, and being extensible to arbitrary decision trees (see \cref{sec:appendix-trees}), \tool can not handle arbitrary ensembles of decision trees. %
However, as these have been shown to be significantly more sensitive to $\ell_p$-norm perturbations than stump ensembles \citep{WangZCBH20}, we believe this limitation to be of little practical relevance.
Further, like all Smoothing-based approaches, we construct a smoothed model from a base classifier and only obtain robustness guarantees for the former. 
In contrast to standard Randomized Smoothing approaches, we can, however, evaluate the smoothed model exactly and efficiently.

\paragraph{Societal Impact}
As our contributions improve certified accuracy and certification radii while retaining high natural accuracy, they could help make real-world AI systems more robust and thus generally amplify both any positive or negative societal effects. Further, while we achieve state-of-the-art results, these may not be sufficient to guarantee robustness in real-world deployment and could give practitioners a false sense of security, leading to them relying more on our models than is justified.

\section{Conclusion}
\label{sec:conclusion}

We propose \tool, a (De-)Randomized Smoothing approach to robustness certification, enabling joint deterministic certificates over numerical and categorical variables for decision stump ensembles by leveraging their structure to compute their exact output distributions for a given input randomization scheme. The key insight enabling this is that this output distribution can be efficiently computed by aggregating independent distributions associated with the individual features used by the ensemble.
We additionally propose a robust MLE optimality criterion for training individual decision stumps and two boosting schemes improving an ensemble's certifiable accuracy.
Empirically, we demonstrate that \tool significantly outperforms the state-of-the-art for tree-based models in a wide range of settings, obtaining up to 4-fold improvements in certifiable accuracy.

\message{^^JLASTBODYPAGE \thepage^^J}

\clearpage
\bibliography{references}
\bibliographystyle{IEEEtranN}

\newpage
\section*{Checklist}

\begin{enumerate}
\item For all authors...
\begin{enumerate}
  \item Do the main claims made in the abstract and introduction accurately reflect the paper's contributions and scope?
  \answerYes{We substantiate all claims theoretically (\cref{sec:det_smoothing} and \cref{sec:training}) or empirically (\cref{sec:eval}), as appropriate.}
  \item Did you describe the limitations of your work?
	\answerYes{We describe the limitations in \cref{sec:limitations-and-impact}.}
  \item Did you discuss any potential negative societal impacts of your work?
	\answerYes{We discuss societal impacts in \cref{sec:limitations-and-impact}.}
  \item Have you read the ethics review guidelines and ensured that your paper conforms to them?
    \answerYes{}
\end{enumerate}

\item If you are including theoretical results...
\begin{enumerate}
  \item Did you state the full set of assumptions of all theoretical results?
    \answerYes{See \cref{sec:det_smoothing}, \cref{sec:training}, and \cref{sec:appendix-additional-thery}.}
        \item Did you include complete proofs of all theoretical results?
    \answerYes{See \cref{sec:det_smoothing}, \cref{sec:training}, and \cref{sec:appendix-additional-thery}.}
\end{enumerate}

\item If you ran experiments...
\begin{enumerate}
  \item Did you include the code, data, and instructions needed to reproduce the main experimental results (either in the supplemental material or as a URL)?
    \answerYes{We will include them as supplemental material and will release them to the public upon publication. All datasets we use are publicly available.}
  \item Did you specify all the training details (e.g., data splits, hyperparameters, how they were chosen)?
    \answerYes{We provide full training details in \cref{app:experimental-details}.}
        \item Did you report error bars (e.g., with respect to the random seed after running experiments multiple times)?
    \answerYes{Both our training and certification methods are deterministic (up to floating point errors), thus our main experiments do not contain error bars. However, in \cref{app:additional_experiments} we report numerous error bars with respect to the data split via 5-fold cross-validation.}
        \item Did you include the total amount of compute and the type of resources used (e.g., type of GPUs, internal cluster, or cloud provider)?
    \answerYes{We provide resource type details in \cref{sec:eval} and timing details in \cref{app:experimental-details}.}
\end{enumerate}

\item If you are using existing assets (e.g., code, data, models) or curating/releasing new assets...
\begin{enumerate}
  \item If your work uses existing assets, did you cite the creators?
    \answerYes{}
  \item Did you mention the license of the assets?
    \answerYes{}
  \item Did you include any new assets either in the supplemental material or as a URL?
    \answerNA{}
  \item Did you discuss whether and how consent was obtained from people whose data you're using/curating?
    \answerNA{}
  \item Did you discuss whether the data you are using/curating contains personally identifiable information or offensive content?
    \answerNA{}
\end{enumerate}

\item If you used crowdsourcing or conducted research with human subjects...
\begin{enumerate}
  \item Did you include the full text of instructions given to participants and screenshots, if applicable?
    \answerNA{}
  \item Did you describe any potential participant risks, with links to Institutional Review Board (IRB) approvals, if applicable?
    \answerNA{}
  \item Did you include the estimated hourly wage paid to participants and the total amount spent on participant compensation?
    \answerNA{}
\end{enumerate}

\end{enumerate}

\message{^^JLASTREFERENCESPAGE \thepage^^J}

\ifbool{includeappendix}{%
	\clearpage
	\appendix
	
\section{Additional Theory}\label{sec:appendix-additional-thery}

In this section, we provide additional theoretical results, omitted in the main paper due to space constraints.
Concretely, in \cref{sec:pdf_computation} we provide a proof for \cref{thm:correctness} on the correctness of our PDF computation.
In \cref{app:mle_proof}, we show that $\gamma_l$, $\gamma_r$ and $v_m$, computed as outlined in \cref{sec:training}, are indeed jointly MLE-optimal.
Finally, we provide more details on \robtreeboost and \robada in \cref{sec:tree_boost_extra} and \cref{sec:ada_boost_extra}, respectively.

\subsection{PDF Computation}
\label{sec:pdf_computation}

Here, we provide a proof for \cref{thm:correctness} on the correctness of our efficient PDF-computation, restated below for convenience.
\correctness*

\begin{proof}
    Let the random variable $\Gamma^{(i)}$ be the prediction of the $i$-th meta-stump, then we have by definition of the meta-stump $\P[\Gamma^{(i)}=\Gamma_{i,j}] = \P_{x'_i \sim \dist(\vx)}[v_{i,j-1} < x'_i \leq v_{i,j}]$ (see \cref{sec:det_smoothing}).
    Note that, for presentational simplicity, we assume $\Gamma_{i,j} \neq \Gamma_{i,k}, \forall k \neq j$. %
    Now, we first show by induction that $\texttt{pdf}[i]$ computes the exact PDF of $\sum_{l=1}^{i}\Gamma^{(l)}$ (\cref{lemma:correctness}), before showing how the CDF of the meta-stump ensemble follows.

\begin{lemma}
    \label{lemma:correctness}
    \cref{alg:dp} computes $\texttt{pdf}[i][t] = \P\left[\sum_{l=1}^{i}\Gamma^{(l)}=t\right]$.
\end{lemma}
\begin{proof}
    We proceed by induction over $i$. In the base case, for $i=0$, we directly have $\texttt{pdf}[0][0]=1.0$ and $\texttt{pdf}[0][t]=0.0$ for $t>0$ by construction.
    Now the induction assumption is that $\texttt{pdf}[i-1][t] = \P\left[\sum_{l=1}^{i-1}\Gamma^{(l)}=t\right]$ for an arbitrary $i \leq d$ and all corresponding $t$.
    To compute the $\texttt{pdf}[i][t]$, we now have:
    \begin{align*}
        \texttt{pdf}[i][t] 
        &= \sum_{j=1}^{M_i}\texttt{pdf}[i-1][t-\Gamma_{i,j}] \cdot \P_{x'_i \sim \dist(\vx)}[v_{i,j-1} < x'_i \leq v_{i,j}]\\
        &= \sum_{j=1}^{M_i}\P\left[\left(\sum_{l=1}^{i-1}\Gamma^{(l)}\right)=t-\Gamma_{i,j}\right] \cdot \P_{x'_i \sim \dist(\vx)}[v_{i,j-1} < x'_i \leq v_{i,j}]\\
        &= \sum_{j=1}^{M_i}\P\left[\left(\sum_{l=1}^{i}\Gamma^{(l)}\right)=t \; \bigg| \; \Gamma^{(i)}=\Gamma_{i,j}\right] \cdot \P\left[\Gamma^{(i)}=\Gamma_{i,j}\right]\\
        &= \P\left[\left(\sum_{l=1}^{i}\Gamma^{(l)}\right)=t\right]\\
    \end{align*}
    where we first use the definition $\texttt{pdf}[i][t]$ according to \cref{alg:dp}, followed by induction assumption, the independency of different meta-stumps and the the law of total probability over $j$.
\end{proof}

Now, we show how \cref{thm:correctness} directly follows from \cref{lemma:correctness}.
Recall that $\gamma_{i,j} = \frac{\Gamma_{i,j}}{\Delta}$, where $\Delta$ is the number of discretization steps.
Similarly to $\Gamma^{(i)}$, let $\gamma^{(i)}$ be the random variable describing the prediction of the $i$-th meta-stump.
Using \cref{lemma:correctness}, we obtain
\begin{align*}
    \bar{\bc{F}}_{M,\vx}(z) 
    &= \sum_{t=0}^{ \lfloor z M\Delta \rfloor  } \texttt{pdf}[d][t] \\
    &= \sum_{t=0}^{ \lfloor z M\Delta \rfloor  } \P\left[\sum_{i=1}^{d}\Gamma^{(i)}=t\right] \\
    &= \P\left[\sum_{i=1}^{d}\Gamma^{(i)} \leq \lfloor z M\Delta \rfloor \right] \\
    &= \P\left[\sum_{i=1}^{d}\frac{\gamma^{(i)}}{M} M \Delta \leq \lfloor z M\Delta \rfloor \right] \\
    &= \P\left[\sum_{i=1}^{d}\frac{\gamma^{(i)}}{M}\leq  z  \right]\\
    &= \P\left[\bar{f}_M(\vx) \leq  z \right].
\end{align*}
Where the second to last step follows from the discretization of the leaf predictions leading to a piece-wise constant CDF.
\end{proof}

\subsection{MLE-Optimal Stumps}
\label{app:mle_proof}

In this section, we extend the theory from \cref{sec:training_indp}, showing that the $v_m$, $\gamma_l$ and $\gamma_r$ computed as outlined there, are, in fact, jointly MLE-optimal.

Recall that an individual stump operating on feature $j_m$ is characterized by three parameters: $v_m$, $\gamma_l$ and $\gamma_r$.
In \cref{sec:training_indp}, we show how to choose MLE-optimal $\gamma_l$ and $\gamma_r$ given $v_m$.
It remains to show that if $v_m$ minimizes the entropy impurity $\imp_{\text{entropy}}$, then $\gamma_l^{\dist,\text{MLE}}$, $\gamma_l^{\dist,\text{MLE}}$, and $v_m$ are jointly MLE-optimal.

For an arbitrary split position $v_m$, we have the probabilities $p_{l,i}(v_m) = \P_{\vx' \sim \dist(\vx_i)}[x'_{j_m} \leq v_m]$ and $p_{r,i}(v_m) = 1 - p_{l,i}(v_m)$ of $\vx'_i$ lying to the left or the right of $v_m$, respectively, under the input randomization scheme $\dist$. For an i.i.d. dataset with $n$ samples $(\vx_i, y_i) \sim (\bc{X},\bc{Y})$, we define the probabilities  $p^y_{j}(v_m) = \frac{1}{n}\sum_{\{i | y_i=y\}} p_{j,i}(v_m)$ of picking the $j \in \{l, r\}$ leaf, conditioned on the target label, and $p_j(v_m) = p_j^0(v_m) +p_j^1(v_m)$ as their sum. Now, we compute the entropy impurity $\imp_\text{entropy}$ \citep{BustosKSSV04} as

\begin{align*}
	\imp_\text{entropy}(v_m) &= -\sum_{j \in \{l,r\}} p_{j}(v_m)  \sum_{y \in \{0,1\}}  \frac{p^y_{j}(v_m)}{p_{j}(v_m)} \log\left(\frac{p^y_{j}(v_m)}{p_{j}(v_m)}\right)\\
	&= -\sum_{j \in \{l,r\}} \sum_{y \in \{0,1\}}  p^y_{j}(v_m) \log\left(\frac{p^y_{j}(v_m)}{p_{j}(v_m)}\right).
\end{align*}

Similarly, let $\gamma_l^{\dist,\text{MLE}}(v_m)$ and $\gamma_l^{\dist,\text{MLE}}(v_m)$ be the MLE-optimal predictions given $v_m$, as computed in \cref{sec:training_indp}.
We formalize our statement as follows in \cref{thm:mle-jointly}:

\begin{theorem}
	\label{thm:mle-jointly}
    Given an i.i.d. dataset with $n$ samples $(\vx_i, y_i) \sim (\bc{X},\bc{Y})$, 
    let $v_m^* :=\argmin_{v_m}\imp_\text{entropy}(v_m)$, $\gamma_l(v_m^*) = \frac{p^1_{l}(v_m^*)}{p^1_{l}(v_m^*) + p^0_{l}(v_m^*)}$ and $\gamma_r(v_m^*) = \frac{p^1_{r}(v_m^*)}{p^1_{r}(v_m^*) + p^0_{r}(v_m^*)}$.
    Then $v_m^*, \gamma_l$ and $\gamma_r$ are jointly MLE-optimal with respect to that dataset.
\end{theorem}
\begin{proof}
    Similarly to \cref{sec:training_indp}, but also optimizing over $v_m$, we obtain:
    \begin{align*}
        v_{m}^{\dist\text{MLE}}, \gamma_l^{\dist\text{MLE}}, \gamma_r^{\dist\text{MLE}} &= \argmax_{v_{m}, \gamma_l, \gamma_r}\; \P[\bc{Y} \mid \dist(\bc{X}), f_m] \\
        &= \argmax_{v_{m}, \gamma_l, \gamma_r}\; \sum_{i=1}^{n} \E_{\vx' \sim \dist(\vx_i)} \left[\log \P[y_i \mid \vx', f_m]\right]\\
        &= \argmax_{v_{m}, \gamma_l, \gamma_r}\; \sum_{i \in \{i \mid y_i = 0\}}^{n} p_{l,i}(v_{m})  \log(1 - \gamma_l) + p_{r,i}(v_{m})  \log(1 - \gamma_r) \\
        & \qquad\quad\;\;\;\, + \sum_{i \in \{i \mid y_i = 1\}}^{n} p_{l,i}(v_{m}) \log(\gamma_l) + p_{r,i}(v_{m}) \log(\gamma_r) \\
        &= \argmax_{v_{m}, \gamma_l, \gamma_r}\; p^0_{l}(v_{m})  \log(1 - \gamma_l) + p^0_{r}(v_{m})  \log(1 - \gamma_r) \\
        & \qquad\quad\;\;\;\, + p^1_{l}(v_{m}) \log(\gamma_l) + p^1_{r}(v_{m}) \log(\gamma_r)
    \end{align*}
    As shown in \cref{sec:training_indp}, for a fixed $v_m$, the MLE-optimal estimates for $\gamma_l$ and $\gamma_r$ are $\gamma_l^{\dist\text{MLE}}(v_m) = \frac{p^1_{l}(v_m)}{p^1_{l}(v_m) + p^0_{l}(v_m)}$ and $\gamma_r^{\dist\text{MLE}}(v_m) = \frac{p^1_{r}(v_m)}{p^1_{r}(v_m) + p^0_{r}(v_m)}$.
    Hence, in the following, it is enough to optimize over $v_m$, substituting in $\gamma_l^{\dist\text{MLE}}(v_m)$ and $\gamma_r^{\dist\text{MLE}}(v_m)$.
    We obtain:
    \begin{align*}
        v_{m}^{\dist\text{MLE}} &= \argmax_{v_{m}}\; p^0_{l}(v_{m})  \log(1 - \gamma_l^{\dist,\text{MLE}}(v_m)) + p^0_{r}(v_{m})  \log(1 - \gamma_r^{\dist,\text{MLE}}(v_m)) \\
        & \qquad\quad\;\;\;\, + p^1_{l}(v_{m}) \log(\gamma_l^{\dist,\text{MLE}}(v_m)) + p^1_{r}(v_{m}) \log(\gamma_r^{\dist,\text{MLE}}(v_m)) \\
        &= \argmax_{v_{m}}\; p^0_{l}(v_{m})  \log \left(1 - \frac{p^1_{l}(v_m)}{p^1_{l}(v_m) + p^0_{l}(v_m)}\right) + p^0_{r}(v_{m})  \log \left(1 - \frac{p^1_{r}(v_m)}{p^1_{r}(v_m) + p^0_{r}(v_m)}\right) \\
        & \qquad\quad\;\;\;\, + p^1_{l}(v_{m}) \log \left(\frac{p^1_{l}(v_m)}{p^1_{l}(v_m) + p^0_{l}(v_m)}\right) + p^1_{r}(v_{m}) \log \left(\frac{p^1_{r}(v_m)}{p^1_{r}(v_m) + p^0_{r}(v_m)}\right) \\
        &= \argmax_{v_{m}}\; p^0_{l}(v_{m})  \log \left(1 - \frac{p^1_{l}(v_m)}{p_{l}(v_m)}\right) + p^0_{r}(v_{m})  \log \left(1 - \frac{p^1_{r}(v_m)}{p_{r}(v_m)}\right) \\
        & \qquad\quad\;\;\;\, + p^1_{l}(v_{m}) \log \left(\frac{p^1_{l}(v_m)}{p_{l}(v_m)}\right) + p^1_{r}(v_{m})\log \left(\frac{p^1_{r}(v_m)}{p_{r}(v_m)}\right) \\
        &= \argmax_{v_{m}}\; - \imp_\text{entropy}(v_m) \\
        &= \argmin_{v_{m}}\; \imp_\text{entropy}(v_m) \\
        &= v_m^*
    \end{align*}
    Thus, we have that the triplet $v_m^* :=\argmin_{v_m}\imp_\text{entropy}(v_m)$, $\gamma_l(v_m^*) = \frac{p^1_{l}(v_m^*)}{p^1_{l}(v_m^*) + p^0_{l}(v_m^*)}$ and $\gamma_r(v_m^*) = \frac{p^1_{r}(v_m^*)}{p^1_{r}(v_m^*) + p^0_{r}(v_m^*)}$ is jointly MLE-optimal.
\end{proof}

\subsection{Gradient Boosting for Certifiable Robustness}
\label{sec:tree_boost_extra}

Below, we describe \robtreeboost, already outlined in \cref{sec:training_boosted}, in more detail.
Formally, we aim to minimize the cross-entropy loss between the certifiable prediction at the $q$\th percentile $\bc{F}_{m-1,\vx_i}^{-1}(q)$ and the one-hot target probability given by the label $y$, where we choose $q = \rho^{-1}(r)$ for some target radius r.
Concretely, to add the $m$\th stump to our ensemble, we begin by computing the certifiable prediction $y_i'$:
\begin{equation} \label{eqn:cert_pred}
	y'_i =
	\begin{cases}
		\bar{\bc{F}}_{m-1,\vx_i}^{-1}(q) \qquad &\text{ if } y = 0 \\
		\bar{\bc{F}}_{m-1,\vx_i}^{-1}(1-q) \qquad &\text{ if } y = 1.
	\end{cases}
\end{equation}
Now, we define the pseudo label $\tilde{y}$ as the residual between the target label $y$ and the certifiable prediction $y'$, scaled to $[0,1]$ as $\tilde{y}_i = \frac{1}{2} + \frac{y_i - y'_i}{2}$. 
Subsequently, we select feature $j_m$ and split position $v_m$ that minimize the mean squared error impurity (MSE) under the randomization scheme for these pseudo-labels.
As before, we define the mean squared error impurity $\imp_\text{MSE}$ in terms of the branching probabilities $p_{l,i} = \P_{\vx' \sim \dist(\vx_i)}[x'_{j_m} \leq v_m]$ and $p_{r,i} = 1 - p_{l,i}$:
\begin{equation}
\mu_j = \frac{\sum_{i=1}^n p_{j,i} \, \tilde{y}_i}{\sum_{i=1}^n p_{j,i}} \qquad\qquad \imp_\text{MSE} = \frac{\sum_{i=1}^n \sum_{j \in \{l,r\}} p_{j,i} (\tilde{y}_i - \mu_j)^2}{n}.
\end{equation}
The optimal leaf predictions can now be computed approximately \citep{Friedman2001Greedy} to 
\begin{equation}
\gamma_{l} = \frac{\sum_{i=1}^n p_{l,i} \, \tilde{y}'_i}{\sum_i p_{l,i} \, |2\tilde{y}'_i-1|(1-|2\tilde{y}'_i-1|)},
\end{equation}
and $\gamma_r$ analogously. 
We initialize this boosting process with an ensemble of individually MLE-optimal stumps and repeat this boosting step until we have added as many stumps as desired.

\subsection{Adaptive Boosting for Certifiable Robustness}
\label{sec:ada_boost_extra}
Below, we describe \robada, already outlined in \cref{sec:training_boosted}, in more detail.
Our goal is to obtain a weighted ensemble $\bar{F}_K$ 
\begin{equation}
\bar{F}_K(\vx) = \frac{1}{\sum_{k=1}^K \alpha^k} \sum_{k=1}^K \alpha^k \1_{\P_{\vx' \sim \dist(\vx)}[\bar{f}_M^k(\vx')>0.5]>0.5},
\end{equation}
consisting of $K$ stump ensembles $\bar{f}_M^k$, that is certifiably robust at a pre-determined radius $r$.
Here, $\bar{F}_K(\vx) \colon \R^d \to [0, 1]$ is a soft-classifier, that predicts class 1 for outputs $> 0.5$ and class 0 else.

To train the $K$ constituting ensembles such that the overall ensemble $\bar{F}_K$ is certifiably robust at radius $r$, we proceed as follows:
First, we initialize the weights of all samples $\vx_i$ to $w_i^1=\frac{1}{n}$.
Then, for $k = 1$ to $K$, we iteratively fit a new stump ensemble $\bar{f}_M^k$ as described in \cref{sec:training_indp} using the sample weights $w_i^k$. 
Then, similar to \citet{FreundS97} although targeting certifiability instead of accuracy, we update the sample weights as follows:
First, we compute whether the newly trained $k$-th ensemble $\bar{f}_M^k$ is certifiably correct ($c_i$) for each sample $\vx_i$ in the training set:
\begin{equation} \label{eqn:cert_pred_ada}
	{c}_i =
	\begin{cases}
		\1_{\P_{\vx' \sim \dist(\vx_i)}[\bar{f}_M^k(\vx')\leq 0.5] > \rho^{-1}_\vx(r)} \qquad &\text{ if } y = 0 \\
		\1_{\P_{\vx' \sim \dist(\vx_i)}[\bar{f}_M^k(\vx')> 0.5] > \rho^{-1}_\vx(r)} \qquad &\text{ if } y = 1.
	\end{cases}
\end{equation}
Then, we determine the certifiable error $err^k$, and the model weight $\alpha^k$ of $f_m^k$ as:
\begin{equation*}
err^k = \frac{\sum_{i=1}^n w_i (1-c_i)}{\sum_{i=1}^n w_i}  \qquad \qquad \alpha^k = \log\frac{1-err^k}{err^k}
\end{equation*}
and update the sample weights for the next iteration to:
\begin{equation*}
    w_i^{k+1} = \frac{w_i^k \exp(\alpha^k (1-c_i))}{\sum_{i=1}^n w_i^k \exp(\alpha^k (1-c_i))}
\end{equation*}
before training the next ensemble.
This way, we are minimizing the overall loss for certified predictions at radius $r$.

To certify $\bar{F}_K$ at a specific radius $r$, we now have to show that we can certify individual ensembles corresponding to at least half the total weights, or more formally (here, without loss of generality assuming a label of $y=1$):
\begin{equation}
	\sum_{k=1}^K \alpha^k \1_{\P_{\vx' \sim \dist(\vx)}[\bar{f}_M^k(\vx') > 0.5] > \rho^{-1}_\vx(r)} > \frac{\sum_{k=1}^K |\alpha^k|}{2}.
\end{equation}
To compute the certifiable radius for $\bar{F}_k$, we compute the certifiable radii $R^k$ of the individual ensembles, sort them in decreasing order  such that $R^k \geq R^{k+1}$ and obtain the largest radius $R^k$ such that $\sum_{l=1}^{k} \alpha^l > \frac{\sum_{l=1}^{K} |\alpha^l|}{2}$.
Intuitively, we need to find a subset of models such that their weighted predictions for class 1 reach at least half the possible weight, accounting for negative weights.

\section{Experimental Details}
\label{app:experimental-details}

Here, we describe our experimental setup in greater detail.
Note that we also publish all code, models, and instructions required to reproduce our results at \url{https://github.com/eth-sri/drs}.

\subsection{Datasets}
\label{app:deails-datases}

In this section, we describe the datasets we use in detail.

\paragraph{Datasets with Numerical Features}
We conduct experiments focusing on numerical features only on all the datasets considered by prior work \cite{WangZCBH20,Andriushchenko019}.
More concretely, we use the tabular datasets \breast \cite{Dua:2019} and \diabetes \cite{Smith1988UsingTA}, where we follow prior work \citep{WangZCBH20} in using the first 80\% of the samples as train set and the remaining 20\% as test set, normalizing the data to $[0,1]$, and the vision datasets \mnistof \cite{lecun2010mnist}, \mnistts \cite{lecun2010mnist}, and \fmnists \cite{DBLP:journals/corr/abs-1708-07747}, where we use all samples of the right classes from the train and test sets.

Additionally, we consider the \spambase \cite{Dua:2019} dataset, where the task is to predict whether an email is spam (binary classification) given $57$ numerical features.
We normalize all features using the mean and standard deviation of the training data before applying any perturbations.

\paragraph{Datasets with Numerical and Categorical Features}
We conduct our experiments on the joint certification of numerical and categorical features using the popular \adult \cite{Dua:2019}, \credit \cite{Dua:2019}, \mammal \cite{Dua:2019}, and \bank \cite{Dua:2019} datasets.
By default, we use the first 70\% of the samples as the train set, and the remaining 30\% as the test set.
For error bound experiments (in \cref{app:eval-joint-robustness}), we use 5-fold cross-validation over the whole datasets, and report the mean and standard deviations over the 5 folds.
Here, we normalize the numerical features using the mean and standard deviation of the training data, before applying any perturbations.

The \adult \cite{Dua:2019} dataset is a societal dataset based on the 1994 US Census database.
It contains eight categorical and six numerical variables for each individual.
The cardinalities of the categorical variables range from 2 to 42 (concretely, they are $9, 16, 7, 15, 6, 5, 2,$ and $42$).
The task is to predict whether an individual's salary is below or above $50$k USD.

The \credit \cite{Dua:2019} dataset is a financial dataset containing $13$ categorical and $7$ numerical features.
The cardinalities of the categorical features range from $2$ to $10$ (concretely, they are $4, 5, 10, 5, 5, 4, 3, 4, 3, 3, 4, 2,$ and $2$).
The task is to predict whether a customer has a low or high risk to default on a loan.

The \mammal \cite{Dua:2019} dataset is medial dataset where the goal is to predicting whether breast biopsies are needed.
It consists of $3$ numerical and $2$ categorical where the categorical features have cardinalities $4$ and $5$.

The \bank \cite{Dua:2019} dataset is a financial dataset, consisting of $9$ categorical and $7$ numerical features.
The task is to predict whether a client will subscribe to a bank term deposit or not, given the features.

Some datasets exhibit a significant class imbalance, with the minority class constituting $24.6\%$ of the \adult and $29.6\%$ of the \credit train set. %
Therefore, we report balanced certified accuracy, computed as the arithmetic mean of the per class certified accuracies.

\paragraph{Dataset with Categorical Features}
The \mushroom \cite{Dua:2019} dataset contains $22$ categorical features encoding physical features of mushrooms with the goal to predicting whether a mushroom is edible or poisonous.

\begin{wraptable}[12]{r}{0.5 \textwidth}
	\centering
	\small
	\vspace{-5mm}
	\caption{Noise magnitudes used for \cref{tab:baseline-both}.}
	\vspace{-2mm}
	\label{tab:details-noise-levels}
	\resizebox{1.00\linewidth}{!}{
    \begin{tabular}{cccc}
        \toprule
        Method & Dataset & $\lambda \; (\text{for } \ell_1)$ & $\sigma \; (\text{for } \ell_2)$ \\
        \midrule
        \multirow{5}{*}{Independent} & \breast & 2.00 & 4.00 \\
        & \diabetes & 0.35 & 0.25 \\
        & \mnistof & 4.00 & 0.25 \\
        & \mnistts & 4.00 & 0.25 \\
        & \fmnists & 4.00 & 0.25 \\
        \cmidrule(lr){1-4}
        \multirow{5}{*}{Boosting} & \breast & 2.00& 0.25  \\
        & \diabetes& 0.28& 0.15\\
        & \mnistof & 4.00 & 0.25 \\
        & \mnistts & 4.00 & 0.25 \\
        & \fmnists & 4.00 & 0.25 \\
        \bottomrule
	\end{tabular}
}
\end{wraptable}

\subsection{Training Details}\label{app:details-training}
The key (and for independent training, the only) hyper-parameter of our approach is the noise magnitude, $\lambda$ for $\ell_1$-certification and $\sigma$ for $\ell_2$-certification.
In \cref{tab:details-noise-levels}, we report the noise levels chosen for the different datasets.
We discuss the effect of different noise magnitudes in \cref{app:eval-noise-level} and observe that results are generally quite stable across a wide range of noise magnitudes.
Unless otherwise stated, we determine the split position $v_m$ via linear search using increments of size $0.01$ and discretize leaf predictions $\gamma$ using $100$ steps (i.e., $\Delta=100$).

\begin{wraptable}[5]{r}{0.5 \textwidth}
	\centering
	\small
	\vspace{-4.3mm}
	\caption{\robtreeboost parameters.}
	\vspace{-2mm}
	\label{tab:details-treeboost}
	\resizebox{1.00\linewidth}{!}{
		\begin{tabular}{cccc}
			\toprule
			Parameter & Perturbation & \breast & \diabetes \\
			\midrule
			\multirow{2}{*}{Percentile $q$}& $\ell_1$ & 0.60 & 0.70 \\
			& $\ell_2$ & 0.98 & 0.95 \\
			\cmidrule(lr){1-4}
			\multirow{2}{*}{Additional stumps $n_b$}& $\ell_1$ & 30 & 15 \\
			& $\ell_2$ & 40 & 100 \\
			\bottomrule
		\end{tabular}
	}
\end{wraptable}
\paragraph{\robtreeboost}
We initialize \robtreeboost with an ensemble of independently trained stumps and add a further $n_b$ stumps as described in \cref{sec:training_boosted} using the $q$\th percentile to compute the certifiable predictions. We chose $q$ and $n_b$ as shown in \cref{tab:details-treeboost}.

\paragraph{\robada}
To evaluate \robada, we consider ensembles of $K=20$ individual stump ensembles in each of our experiments.
We choose the same noise magnitudes as for independently trained stumps, described in \cref{tab:details-noise-levels}.

\paragraph{Joint Certification}
For joint certification, we use ensembles of independently trained decision stumps, one for each feature.
The stump corresponding to categorical features maps a categorical value to either $0.375$ or $0.625$ (which are the same distance from the decision threshold $0.5$), depending on whether the majority of the samples with this categorical value have class $0$ or $1$, respectively. Note that permitting arbitrary leaf predictions slightly improves clean accuracy, but significantly worsens worst-case behaviour. Choosing leaf predictions further from the decision threshold gives more emphasis to categorical variables compared to numerical ones.
The stumps for the numerical features are learned individually, as described in \cref{sec:training_indp}.
For $\ell_1$, we used the noise magnitude $\lambda = 2.0$ and for $\ell_2$-certification $\sigma=0.25$.

\subsection{Computational Resources and Experimental Timings}
\label{app:details-computation}

In this section, we describe the computational resources required for our experiments.
We run all our experiments using 24 cores of an Intel Xeon Gold 6242 CPUs and a single NVIDIA RTX 2080Ti and report timings for the full experiment in \cref{app:details-computation}. We show timings in \cref{tab:timings}.

\begin{wraptable}[10]{r}{0.5 \textwidth}
	\centering
	\small
	\vspace{-4.5mm}
	\caption{Experimental timings for whole datasets.}
	\vspace{-2mm}
	\label{tab:timings}
	\resizebox{1.00\linewidth}{!}{
		\begin{tabular}{cccccc}
			\toprule
			\multirow{2.5}{*}{Norm} & \multirow{2.5}{*}{Dataset} & \multicolumn{2}{c}{Independent} & \multicolumn{2}{c}{Boosting} \\
			\cmidrule(lr){3-4}
			\cmidrule(lr){5-6}
			&& Training & Certification & Training & Certification \\
			\midrule
			\multirow{5}{*}{$\ell_1$} & \breast & 2s & $<0.1$s & 14s & $<0.1$s \\
			& \diabetes & 2s & $<0.1$s & 2s & $<0.1$s \\
			& \mnistof & 32s & 5s & 13min & 27s \\
			& \mnistts & 29s & 4s & 11min & 16s \\
			& \fmnists & 31s & 6s & 13min & 40s \\
			\cmidrule(lr){1-6}
			\multirow{5}{*}{$\ell_2$} & \breast & 2s & $<0.1$s & 9s & $<0.1$s \\
			& \diabetes & 2s & $<0.1$s & 47s & $<0.1$s\\
			& \mnistof & 14s & 4s & 10min & 29s \\
			& \mnistts & 14s & 3s & 9min & 26s \\
			& \fmnists & 15s & 4s & 9min & 27s \\
			\bottomrule
		\end{tabular}
	}
\end{wraptable}

We observe that all certification is extremely quick with \fmnists taking the longest at $6$s for the whole test set and an ensemble of independently trained stumps in the $\ell_1$-setting, translating to $0.003$s per sample.
When evaluating models in single instead of double precision, we can, e.g., further reduce certification times from $3$s to $1.2$s for \mnistts.
The independently MLE-optimal training is similarly quick, allowing us to run all core experiments in less than $5$ minutes.
Only \robada takes more than one minute for an individual experiment, as it involves training and certifying $20$ stump ensembles.
For datasets combining categorical and numerical features, the training and certification for the categorical variables is almost instantaneous and dominated by that for the numerical features.
The latter requires $19.9$s and $47.0$s for the $\ell_1$ and $\ell_2$-experiment, respectively, on \adult and $1.5$s respectively $2.0$s on \credit.
We remark that computational efficiency was not a main focus of this work and we did not optimize runtimes.

\section{Additional Experiments}
\label{app:additional_experiments}

In this section, we extend our experimental evaluation from \cref{sec:eval}.
Concretely, in \cref{app:eval-joint-robustness}, we provide additional experiments on the joint certification of categorical and numerical variables.
In \cref{app:eval-ds-rs}, we compare \tool to \RS in more detail while in \cref{app:eval-mle-ablation}, we continue our investigation of our MLE optimality criterion.
Moreover, in \cref{app:eval-noise-level}, we provide additional experiments on the effect of the noise magnitudes $\lambda$ and $\sigma$ for $\ell_1$- and $\ell_2$-certification, respectively.
In \cref{app:discretization-size-ablation}, we analyze the impact of the discretization granularity and in \cref{app:binning-size-ablation}, we evaluate the effect of an approximate split position optimization.
Finally, in \cref{app:error-bounds}, we include error bound experiments for certification of numerical features via 5-fold cross-validation.%

\subsection{Additional Experiments on Joint Robustness Certificates}
\label{app:eval-joint-robustness}

\begin{table}[tp]
	\centering
	\small
	\centering
	\caption{Balanced certified accuracy (BCA) $[\%]$ under joint $\ell_0$- and $\ell_2$-perturbations of categorical and numerical features, respectively, depending on whether model uses categorical and/or numerical features. The balanced natural accuracy is the BCA at radius $r = 0.0$. Larger is better.}
	\label{tab:joint-cv-balanced-l2}
	\vspace{2mm}
	\resizebox{0.98\columnwidth}{!}{
    \begin{tabular}{ccccccccccc}
        \toprule
        \multirow{2.6}{*}{Dataset} & \multirowcell{2.6}{Categorical \\ Features} & \multirowcell{2.6}{$\ell_0$ Radius $r_0$} &\multirowcell{2.6}{ BCA without \\ Numerical Features} & \multicolumn{7}{c}{BCA with Numerical Features at $\ell_2$ Radius $r_2$}\\
        \cmidrule(lr){5-11}
        & & & & 0.00 & 0.25 & 0.50 & 0.75 & 1.00 & 1.25 & 1.50 \\
        \midrule
        \multirow{5.5}{*}{\adult} & no & - & - &74.3\textsubscript{$\pm$ 0.4} & 65.5\textsubscript{$\pm$ 0.3} & 42.3\textsubscript{$\pm$ 0.5} & 26.9\textsubscript{$\pm$ 0.4} & 13.7\textsubscript{$\pm$ 0.3} & 7.7\textsubscript{$\pm$ 0.5} & 4.4\textsubscript{$\pm$ 0.3} \\
        \cmidrule(lr){2-2}
        & \multirow{4.0}{*}{yes} & 0 & 76.2\textsubscript{$\pm$ 0.6} & 77.9\textsubscript{$\pm$ 0.4} & 74.2\textsubscript{$\pm$ 0.7} & 68.0\textsubscript{$\pm$ 0.6} & 62.9\textsubscript{$\pm$ 0.6} & 48.4\textsubscript{$\pm$ 0.4} & 39.6\textsubscript{$\pm$ 0.7} & 34.2\textsubscript{$\pm$ 0.4} \\
        & & 1 & 57.0\textsubscript{$\pm$ 0.8} & 66.2\textsubscript{$\pm$ 0.8} & 61.5\textsubscript{$\pm$ 0.9} & 52.8\textsubscript{$\pm$ 0.7} & 45.9\textsubscript{$\pm$ 0.7} & 33.1\textsubscript{$\pm$ 0.4} & 25.0\textsubscript{$\pm$ 0.5} & 20.5\textsubscript{$\pm$ 0.4} \\
        & & 2 & 32.9\textsubscript{$\pm$ 0.6} &  50.7\textsubscript{$\pm$ 0.6} & 45.4\textsubscript{$\pm$ 0.8} & 36.2\textsubscript{$\pm$ 0.5} & 27.8\textsubscript{$\pm$ 0.3} & 19.7\textsubscript{$\pm$ 0.3} & 15.2\textsubscript{$\pm$ 0.4} & 11.7\textsubscript{$\pm$ 0.3} \\
        & & 3 & 8.9\textsubscript{$\pm$ 0.2} & 35.9\textsubscript{$\pm$ 0.5} & 30.8\textsubscript{$\pm$ 0.6} & 23.4\textsubscript{$\pm$ 0.4} & 14.6\textsubscript{$\pm$ 0.4} & 9.7\textsubscript{$\pm$ 0.4} & 7.2\textsubscript{$\pm$ 0.3} & 5.1\textsubscript{$\pm$ 0.2} \\
        
        \midrule
        \multirow{5.5}{*}{\credit} & no & - & - &  59.4\textsubscript{$\pm$ 3.7} & 51.3\textsubscript{$\pm$ 3.0} & 39.6\textsubscript{$\pm$ 2.9} & 27.0\textsubscript{$\pm$ 5.3} & 19.5\textsubscript{$\pm$ 7.9} & 13.5\textsubscript{$\pm$ 6.3} & 6.7\textsubscript{$\pm$ 4.1}  \\
        \cmidrule(lr){2-2}
        & \multirow{4.0}{*}{yes} & 0 & 64.7\textsubscript{$\pm$ 4.2}  & 65.3\textsubscript{$\pm$ 4.3} & 64.0\textsubscript{$\pm$ 4.9} & 62.1\textsubscript{$\pm$ 4.6} & 58.4\textsubscript{$\pm$ 4.0} & 53.3\textsubscript{$\pm$ 3.8} & 51.5\textsubscript{$\pm$ 4.6} & 49.2\textsubscript{$\pm$ 4.7} \\
        & & 1 & 44.7\textsubscript{$\pm$ 3.5} & 48.2\textsubscript{$\pm$ 3.1} & 46.1\textsubscript{$\pm$ 3.1} & 42.1\textsubscript{$\pm$ 3.5} & 38.6\textsubscript{$\pm$ 3.5} & 35.4\textsubscript{$\pm$ 4.5} & 33.2\textsubscript{$\pm$ 5.2} & 31.1\textsubscript{$\pm$ 5.7} \\
        & & 2 & 26.7\textsubscript{$\pm$ 5.7} & 29.8\textsubscript{$\pm$ 4.4} & 27.9\textsubscript{$\pm$ 4.4} & 24.5\textsubscript{$\pm$ 5.7} & 21.3\textsubscript{$\pm$ 5.8} & 18.7\textsubscript{$\pm$ 5.8} & 16.7\textsubscript{$\pm$ 5.6} & 15.5\textsubscript{$\pm$ 6.0} \\
        & & 3 & 11.1\textsubscript{$\pm$ 4.3} & 14.2\textsubscript{$\pm$ 4.5} & 13.3\textsubscript{$\pm$ 3.8} & 11.4\textsubscript{$\pm$ 4.5} & 9.8\textsubscript{$\pm$ 3.6} & 8.5\textsubscript{$\pm$ 3.9} & 6.7\textsubscript{$\pm$ 3.7} & 6.5\textsubscript{$\pm$ 3.8} \\
        
        \midrule
        \multirow{4.5}{*}{\mammal} & no & - & - & 61.7\textsubscript{$\pm$ 2.8} & 61.6\textsubscript{$\pm$ 2.9} & 51.5\textsubscript{$\pm$ 3.0} & 13.8\textsubscript{$\pm$ 4.2} & 9.4\textsubscript{$\pm$ 5.0} & 8.6\textsubscript{$\pm$ 5.6} & 6.4\textsubscript{$\pm$ 6.3}  \\
        \cmidrule(lr){2-2}
        & \multirow{3.0}{*}{yes} & 0 & 79.0\textsubscript{$\pm$ 1.2} & 78.1\textsubscript{$\pm$ 2.6} & 78.1\textsubscript{$\pm$ 2.6} & 76.4\textsubscript{$\pm$ 2.0} & 51.1\textsubscript{$\pm$ 5.8} & 48.4\textsubscript{$\pm$ 9.5} & 47.6\textsubscript{$\pm$ 9.1} & 42.5\textsubscript{$\pm$ 12.0} \\
        & & 1 & 30.8\textsubscript{$\pm$ 3.3} & 46.7\textsubscript{$\pm$ 6.2} & 46.6\textsubscript{$\pm$ 6.2} & 36.5\textsubscript{$\pm$ 6.9} & 11.3\textsubscript{$\pm$ 3.9} & 7.7\textsubscript{$\pm$ 4.0} & 7.1\textsubscript{$\pm$ 4.5} & 4.9\textsubscript{$\pm$ 5.1} \\
        & & 2 & 0.0\textsubscript{$\pm$ 0.0} & 12.8\textsubscript{$\pm$ 2.2} & 12.7\textsubscript{$\pm$ 2.4} & 2.7\textsubscript{$\pm$ 2.0} & 2.6\textsubscript{$\pm$ 2.0} & 2.4\textsubscript{$\pm$ 2.0} & 2.1\textsubscript{$\pm$ 2.3} & 0.0\textsubscript{$\pm$ 0.0} \\

        \midrule
        \multirow{5.5}{*}{\bank} & no & - & - & 73.1\textsubscript{$\pm$ 2.2} & 63.0\textsubscript{$\pm$ 1.4} & 47.7\textsubscript{$\pm$ 2.1} & 31.9\textsubscript{$\pm$ 1.4} & 17.9\textsubscript{$\pm$ 3.8} & 12.5\textsubscript{$\pm$ 5.0} & 7.4\textsubscript{$\pm$ 4.7}   \\
        \cmidrule(lr){2-2}
        & \multirow{4.0}{*}{yes} & 0 & 62.8\textsubscript{$\pm$ 1.9} & 69.9\textsubscript{$\pm$ 1.8} & 65.4\textsubscript{$\pm$ 1.4} & 57.0\textsubscript{$\pm$ 0.6} & 48.8\textsubscript{$\pm$ 0.7} & 39.6\textsubscript{$\pm$ 1.6} & 30.4\textsubscript{$\pm$ 1.6} & 24.0\textsubscript{$\pm$ 1.7} \\
        & & 1 & 42.3\textsubscript{$\pm$ 1.5} & 53.6\textsubscript{$\pm$ 1.9} & 47.7\textsubscript{$\pm$ 2.1} & 40.1\textsubscript{$\pm$ 2.2} & 30.5\textsubscript{$\pm$ 2.0} & 21.7\textsubscript{$\pm$ 1.7} & 14.8\textsubscript{$\pm$ 1.9} & 9.8\textsubscript{$\pm$ 2.4} \\
        & & 2 & 21.2\textsubscript{$\pm$ 2.3} &  37.4\textsubscript{$\pm$ 2.5} & 31.5\textsubscript{$\pm$ 2.1} & 23.2\textsubscript{$\pm$ 2.0} & 14.9\textsubscript{$\pm$ 2.0} & 9.0\textsubscript{$\pm$ 2.3} & 6.1\textsubscript{$\pm$ 2.2} & 4.3\textsubscript{$\pm$ 2.5} \\
        & & 3 &  7.2\textsubscript{$\pm$ 2.3} &  21.8\textsubscript{$\pm$ 2.9} & 17.5\textsubscript{$\pm$ 2.7} & 11.0\textsubscript{$\pm$ 2.3} & 5.6\textsubscript{$\pm$ 1.3} & 3.0\textsubscript{$\pm$ 1.4} & 2.2\textsubscript{$\pm$ 1.4} & 1.0\textsubscript{$\pm$ 0.4} \\

		\bottomrule
	\end{tabular}
}
\vspace{-4mm}
\end{table}

\begin{table}[tp]
	\centering
	\small
	\centering
	\caption{Certified accuracy (CA) $[\%]$ under joint $\ell_0$- and $\ell_2$-perturbations of categorical and numerical features, respectively, depending on whether model uses categorical and/or numerical features. The natural accuracy is the CA at radius $r = 0.0$. Larger is better.}
	\label{tab:joint-cv-imbalanced-l2}
	\vspace{2mm}
	\resizebox{0.98\columnwidth}{!}{
    \begin{tabular}{ccccccccccc}
        \toprule
        \multirow{2.6}{*}{Dataset} & \multirowcell{2.6}{Categorical \\ Features} & \multirowcell{2.6}{$\ell_0$ Radius $r_0$} &\multirowcell{2.6}{ CA without \\ Numerical Features} & \multicolumn{7}{c}{CA with Numerical Features at $\ell_2$ Radius $r_2$}\\
        \cmidrule(lr){5-11}
        & & & & 0.00 & 0.25 & 0.50 & 0.75 & 1.00 & 1.25 & 1.50 \\
        \midrule
        \multirow{5.5}{*}{\adult} & no & - & - & 74.2\textsubscript{$\pm$ 0.5} & 65.7\textsubscript{$\pm$ 0.6} & 37.5\textsubscript{$\pm$ 1.0} & 23.2\textsubscript{$\pm$ 0.8} & 9.6\textsubscript{$\pm$ 0.3} & 4.0\textsubscript{$\pm$ 0.5} & 2.1\textsubscript{$\pm$ 0.2}  \\
        \cmidrule(lr){2-2}
        & \multirow{4.0}{*}{yes} & 0 & 69.7\textsubscript{$\pm$ 0.6} & 70.0\textsubscript{$\pm$ 0.5} & 65.8\textsubscript{$\pm$ 0.7} & 58.6\textsubscript{$\pm$ 0.7} & 53.7\textsubscript{$\pm$ 0.7} & 34.9\textsubscript{$\pm$ 0.7} & 24.4\textsubscript{$\pm$ 0.8} & 19.1\textsubscript{$\pm$ 0.4} \\
        & & 1 & 52.0\textsubscript{$\pm$ 0.8} & 58.3\textsubscript{$\pm$ 0.9} & 53.6\textsubscript{$\pm$ 1.0} & 43.8\textsubscript{$\pm$ 0.8} & 36.9\textsubscript{$\pm$ 0.8} & 21.3\textsubscript{$\pm$ 0.5} & 13.0\textsubscript{$\pm$ 0.6} & 10.2\textsubscript{$\pm$ 0.3} \\
        & & 2 & 27.5\textsubscript{$\pm$ 0.6} & 43.1\textsubscript{$\pm$ 0.7} & 38.3\textsubscript{$\pm$ 0.8} & 28.2\textsubscript{$\pm$ 0.6} & 19.4\textsubscript{$\pm$ 0.4} & 11.0\textsubscript{$\pm$ 0.4} & 7.5\textsubscript{$\pm$ 0.3} & 5.7\textsubscript{$\pm$ 0.2} \\
        & & 3 & 6.6\textsubscript{$\pm$ 0.2} & 28.7\textsubscript{$\pm$ 0.5} & 24.1\textsubscript{$\pm$ 0.7} & 16.5\textsubscript{$\pm$ 0.3} & 8.4\textsubscript{$\pm$ 0.3} & 4.8\textsubscript{$\pm$ 0.3} & 3.5\textsubscript{$\pm$ 0.2} & 2.4\textsubscript{$\pm$ 0.1} \\
        
        \midrule
        \multirow{5.5}{*}{\credit} & no & - & - & 63.7\textsubscript{$\pm$ 3.5} & 55.3\textsubscript{$\pm$ 3.9} & 43.5\textsubscript{$\pm$ 4.6} & 30.7\textsubscript{$\pm$ 6.0} & 21.7\textsubscript{$\pm$ 9.1} & 14.0\textsubscript{$\pm$ 7.2} & 7.7\textsubscript{$\pm$ 5.2}  \\
        \cmidrule(lr){2-2}
        & \multirow{4.0}{*}{yes} & 0 & 58.3\textsubscript{$\pm$ 9.5} & 59.3\textsubscript{$\pm$ 9.5} & 57.9\textsubscript{$\pm$ 10.1} & 55.8\textsubscript{$\pm$ 10.2} & 52.4\textsubscript{$\pm$ 9.6} & 48.2\textsubscript{$\pm$ 8.7} & 45.8\textsubscript{$\pm$ 9.2} & 43.1\textsubscript{$\pm$ 9.4} \\
        & & 1 & 38.2\textsubscript{$\pm$ 7.4} & 42.4\textsubscript{$\pm$ 8.5} & 40.4\textsubscript{$\pm$ 8.5} & 36.5\textsubscript{$\pm$ 9.0} & 33.5\textsubscript{$\pm$ 8.2} & 30.8\textsubscript{$\pm$ 8.0} & 28.4\textsubscript{$\pm$ 7.8} & 26.1\textsubscript{$\pm$ 7.2} \\
        & & 2 & 21.7\textsubscript{$\pm$ 5.0} & 25.0\textsubscript{$\pm$ 6.0} & 23.0\textsubscript{$\pm$ 5.9} & 20.4\textsubscript{$\pm$ 5.5} & 17.7\textsubscript{$\pm$ 5.4} & 15.3\textsubscript{$\pm$ 4.7} & 13.6\textsubscript{$\pm$ 5.0} & 12.3\textsubscript{$\pm$ 4.7} \\
        & & 3 &  8.4\textsubscript{$\pm$ 2.6} & 11.0\textsubscript{$\pm$ 2.5} & 10.3\textsubscript{$\pm$ 2.3} & 8.8\textsubscript{$\pm$ 2.8} & 7.5\textsubscript{$\pm$ 2.4} & 6.4\textsubscript{$\pm$ 2.2} & 5.1\textsubscript{$\pm$ 2.5} & 4.9\textsubscript{$\pm$ 2.5} \\
        
        \midrule
        \multirow{4.5}{*}{\mammal} & no & - & - &  61.7\textsubscript{$\pm$ 3.7} & 61.6\textsubscript{$\pm$ 3.9} & 51.6\textsubscript{$\pm$ 3.4} & 13.8\textsubscript{$\pm$ 4.4} & 9.6\textsubscript{$\pm$ 5.3} & 8.8\textsubscript{$\pm$ 6.0} & 6.6\textsubscript{$\pm$ 6.7}  \\
        \cmidrule(lr){2-2}
        & \multirow{3.0}{*}{yes} & 0 &  79.0\textsubscript{$\pm$ 1.3} & 77.9\textsubscript{$\pm$ 2.8} & 77.9\textsubscript{$\pm$ 2.8} & 76.3\textsubscript{$\pm$ 2.2} & 51.2\textsubscript{$\pm$ 5.0} & 48.7\textsubscript{$\pm$ 8.1} & 47.8\textsubscript{$\pm$ 7.8} & 42.6\textsubscript{$\pm$ 11.4} \\
        & & 1 & 31.8\textsubscript{$\pm$ 4.5} & 46.6\textsubscript{$\pm$ 6.2} & 46.5\textsubscript{$\pm$ 6.2} & 36.5\textsubscript{$\pm$ 6.9} & 11.3\textsubscript{$\pm$ 4.0} & 7.8\textsubscript{$\pm$ 4.3} & 7.2\textsubscript{$\pm$ 4.8} & 5.1\textsubscript{$\pm$ 5.5} \\
        & & 2 & 0.0\textsubscript{$\pm$ 0.0} & 12.8\textsubscript{$\pm$ 2.0} & 12.6\textsubscript{$\pm$ 2.1} & 2.6\textsubscript{$\pm$ 1.9} & 2.5\textsubscript{$\pm$ 1.9} & 2.4\textsubscript{$\pm$ 1.9} & 2.0\textsubscript{$\pm$ 2.2} & 0.0\textsubscript{$\pm$ 0.0} \\

        \midrule
        \multirow{5.5}{*}{\bank} & no & - & - & 68.6\textsubscript{$\pm$ 2.3} & 56.9\textsubscript{$\pm$ 2.1} & 41.1\textsubscript{$\pm$ 1.3} & 21.2\textsubscript{$\pm$ 2.0} & 5.6\textsubscript{$\pm$ 0.9} & 3.1\textsubscript{$\pm$ 1.3} & 1.7\textsubscript{$\pm$ 1.1} \\
        \cmidrule(lr){2-2}
        & \multirow{4.0}{*}{yes} & 0 & 64.4\textsubscript{$\pm$ 7.5} & 73.8\textsubscript{$\pm$ 5.8} & 68.6\textsubscript{$\pm$ 6.3} & 58.9\textsubscript{$\pm$ 7.2} & 47.1\textsubscript{$\pm$ 8.0} & 33.8\textsubscript{$\pm$ 7.4} & 20.7\textsubscript{$\pm$ 6.0} & 13.4\textsubscript{$\pm$ 4.2} \\
        & & 1 & 44.4\textsubscript{$\pm$ 7.9} & 58.8\textsubscript{$\pm$ 6.7} & 52.4\textsubscript{$\pm$ 6.7} & 41.7\textsubscript{$\pm$ 7.6} & 29.3\textsubscript{$\pm$ 7.3} & 16.9\textsubscript{$\pm$ 5.8} & 8.1\textsubscript{$\pm$ 2.6} & 3.8\textsubscript{$\pm$ 0.9} \\
        & & 2 & 23.2\textsubscript{$\pm$ 7.1} & 41.9\textsubscript{$\pm$ 6.9} & 35.9\textsubscript{$\pm$ 6.7} & 24.8\textsubscript{$\pm$ 6.6} & 13.4\textsubscript{$\pm$ 4.9} & 6.0\textsubscript{$\pm$ 2.7} & 2.6\textsubscript{$\pm$ 0.8} & 1.1\textsubscript{$\pm$ 0.6} \\
        & & 3 & 8.3\textsubscript{$\pm$ 4.7} & 24.9\textsubscript{$\pm$ 5.7} & 19.3\textsubscript{$\pm$ 5.2} & 10.7\textsubscript{$\pm$ 3.7} & 4.7\textsubscript{$\pm$ 1.9} & 1.6\textsubscript{$\pm$ 0.6} & 0.6\textsubscript{$\pm$ 0.3} & 0.2\textsubscript{$\pm$ 0.1} \\

		\bottomrule
	\end{tabular}
}
\vspace{-4mm}
\end{table}

\begin{table}[tp]
	\centering
	\small
	\centering
	\caption{Balanced certified accuracy (BCA) $[\%]$ under joint $\ell_0$- and $\ell_1$-perturbations of categorical and numerical features, respectively, depending on whether model uses categorical and/or numerical features. The balanced natural accuracy is the BCA at radius $r = 0.0$. Larger is better.}
	\label{tab:joint-cv-balanced-l1}
	\vspace{2mm}
	\resizebox{0.98\columnwidth}{!}{
    \begin{tabular}{ccccccccccc}
        \toprule
        \multirow{2.6}{*}{Dataset} & \multirowcell{2.6}{Categorical \\ Features} & \multirowcell{2.6}{$\ell_0$ Radius $r_0$} &\multirowcell{2.6}{ BCA without \\ Numerical Features} & \multicolumn{7}{c}{BCA with Numerical Features at $\ell_1$ Radius $r_1$}\\
        \cmidrule(lr){5-11}
        & & & & 0.00 & 0.25 & 0.50 & 0.75 & 1.00 & 1.25 & 1.50 \\
        \midrule
        \multirow{5.5}{*}{\adult} & no & - & - & 62.1\textsubscript{$\pm$ 0.3} & 58.8\textsubscript{$\pm$ 0.5} & 54.9\textsubscript{$\pm$ 0.4} & 51.5\textsubscript{$\pm$ 0.5} & 47.1\textsubscript{$\pm$ 0.5} & 44.6\textsubscript{$\pm$ 0.4} & 42.0\textsubscript{$\pm$ 0.3}  \\
        \cmidrule(lr){2-2}
        & \multirow{4.0}{*}{yes} & 0 & 76.2\textsubscript{$\pm$ 0.6}  & 76.9\textsubscript{$\pm$ 0.5} & 76.4\textsubscript{$\pm$ 0.5} & 75.6\textsubscript{$\pm$ 0.5} & 74.8\textsubscript{$\pm$ 0.6} & 73.6\textsubscript{$\pm$ 0.6} & 72.6\textsubscript{$\pm$ 0.6} & 71.1\textsubscript{$\pm$ 0.6} \\
        & & 1 & 57.0\textsubscript{$\pm$ 0.8} & 59.7\textsubscript{$\pm$ 0.6} & 59.1\textsubscript{$\pm$ 0.6} & 58.4\textsubscript{$\pm$ 0.6} & 57.6\textsubscript{$\pm$ 0.7} & 56.8\textsubscript{$\pm$ 0.7} & 55.8\textsubscript{$\pm$ 0.8} & 54.4\textsubscript{$\pm$ 0.8} \\
        & & 2 & 32.9\textsubscript{$\pm$ 0.6}  & 38.3\textsubscript{$\pm$ 0.5} & 37.3\textsubscript{$\pm$ 0.5} & 36.3\textsubscript{$\pm$ 0.4} & 35.6\textsubscript{$\pm$ 0.4} & 34.8\textsubscript{$\pm$ 0.5} & 34.2\textsubscript{$\pm$ 0.5} & 33.3\textsubscript{$\pm$ 0.5} \\
        & & 3 & 8.9\textsubscript{$\pm$ 0.2} & 17.0\textsubscript{$\pm$ 0.3} & 15.3\textsubscript{$\pm$ 0.3} & 14.3\textsubscript{$\pm$ 0.2} & 13.3\textsubscript{$\pm$ 0.3} & 12.3\textsubscript{$\pm$ 0.3} & 11.9\textsubscript{$\pm$ 0.3} & 11.5\textsubscript{$\pm$ 0.2} \\
        
        \midrule
        \multirow{5.5}{*}{\credit} & no & - & - & 58.2\textsubscript{$\pm$ 4.0} & 54.7\textsubscript{$\pm$ 4.0} & 51.4\textsubscript{$\pm$ 3.4} & 48.0\textsubscript{$\pm$ 2.6} & 43.0\textsubscript{$\pm$ 1.3} & 33.3\textsubscript{$\pm$ 1.5} & 26.9\textsubscript{$\pm$ 2.2}   \\
        \cmidrule(lr){2-2}
        & \multirow{4.0}{*}{yes} & 0 & 64.7\textsubscript{$\pm$ 4.2}  & 65.1\textsubscript{$\pm$ 4.4} & 64.5\textsubscript{$\pm$ 4.1} & 64.0\textsubscript{$\pm$ 3.8} & 63.4\textsubscript{$\pm$ 3.5} & 62.8\textsubscript{$\pm$ 3.6} & 61.2\textsubscript{$\pm$ 4.6} & 60.6\textsubscript{$\pm$ 4.9} \\
        & & 1 & 44.7\textsubscript{$\pm$ 3.5} & 46.1\textsubscript{$\pm$ 3.3} & 45.4\textsubscript{$\pm$ 2.8} & 45.2\textsubscript{$\pm$ 3.0} & 44.5\textsubscript{$\pm$ 3.3} & 43.9\textsubscript{$\pm$ 2.7} & 43.3\textsubscript{$\pm$ 2.9} & 42.2\textsubscript{$\pm$ 3.3} \\
        & & 2 & 26.7\textsubscript{$\pm$ 5.7} & 28.1\textsubscript{$\pm$ 5.6} & 27.8\textsubscript{$\pm$ 5.5} & 27.1\textsubscript{$\pm$ 5.9} & 26.2\textsubscript{$\pm$ 6.0} & 26.1\textsubscript{$\pm$ 6.0} & 25.3\textsubscript{$\pm$ 6.7} & 24.4\textsubscript{$\pm$ 6.4} \\
        & & 3 & 11.1\textsubscript{$\pm$ 4.3} & 12.7\textsubscript{$\pm$ 5.1} & 12.6\textsubscript{$\pm$ 4.9} & 11.9\textsubscript{$\pm$ 4.8} & 11.5\textsubscript{$\pm$ 4.8} & 11.0\textsubscript{$\pm$ 4.2} & 10.6\textsubscript{$\pm$ 4.3} & 10.2\textsubscript{$\pm$ 4.1} \\

        \midrule
        \multirow{4.5}{*}{\mammal} & no & - & - & 51.0\textsubscript{$\pm$ 1.2} & 49.3\textsubscript{$\pm$ 1.2} & 48.4\textsubscript{$\pm$ 1.3} & 48.4\textsubscript{$\pm$ 1.3} & 48.1\textsubscript{$\pm$ 1.1} & 45.9\textsubscript{$\pm$ 1.3} & 45.8\textsubscript{$\pm$ 1.2}  \\
        \cmidrule(lr){2-2}
        & \multirow{3.0}{*}{yes} & 0 &  79.0\textsubscript{$\pm$ 1.2} & 77.0\textsubscript{$\pm$ 1.9} & 76.8\textsubscript{$\pm$ 1.8} & 76.6\textsubscript{$\pm$ 1.9} & 76.6\textsubscript{$\pm$ 1.9} & 76.6\textsubscript{$\pm$ 1.9} & 74.4\textsubscript{$\pm$ 2.4} & 74.4\textsubscript{$\pm$ 2.4} \\
        & & 1 & 30.8\textsubscript{$\pm$ 3.3} & 41.0\textsubscript{$\pm$ 3.8} & 39.5\textsubscript{$\pm$ 3.2} & 38.9\textsubscript{$\pm$ 3.2} & 38.9\textsubscript{$\pm$ 3.2} & 38.6\textsubscript{$\pm$ 3.3} & 37.2\textsubscript{$\pm$ 3.8} & 37.0\textsubscript{$\pm$ 4.0} \\
        & & 2 & 0.0\textsubscript{$\pm$ 0.0} & 0.0\textsubscript{$\pm$ 0.0} & 0.0\textsubscript{$\pm$ 0.0} & 0.0\textsubscript{$\pm$ 0.0} & 0.0\textsubscript{$\pm$ 0.0} & 0.0\textsubscript{$\pm$ 0.0} & 0.0\textsubscript{$\pm$ 0.0} & 0.0\textsubscript{$\pm$ 0.0} \\

        \midrule
        \multirow{5.5}{*}{\bank} & no & - & - & 69.5\textsubscript{$\pm$ 2.5} & 64.7\textsubscript{$\pm$ 2.0} & 61.1\textsubscript{$\pm$ 1.8} & 56.7\textsubscript{$\pm$ 1.9} & 52.1\textsubscript{$\pm$ 1.7} & 47.3\textsubscript{$\pm$ 1.9} & 41.3\textsubscript{$\pm$ 1.3}   \\
        \cmidrule(lr){2-2}
        & \multirow{4.0}{*}{yes} & 0 &  62.8\textsubscript{$\pm$ 1.9} & 68.3\textsubscript{$\pm$ 0.6} & 66.4\textsubscript{$\pm$ 0.7} & 64.9\textsubscript{$\pm$ 0.7} & 63.2\textsubscript{$\pm$ 0.9} & 61.2\textsubscript{$\pm$ 0.5} & 59.0\textsubscript{$\pm$ 0.9} & 55.9\textsubscript{$\pm$ 1.7} \\
        & & 1 & 42.3\textsubscript{$\pm$ 1.5} & 49.1\textsubscript{$\pm$ 1.6} & 46.9\textsubscript{$\pm$ 1.4} & 44.8\textsubscript{$\pm$ 1.3} & 43.3\textsubscript{$\pm$ 1.6} & 40.6\textsubscript{$\pm$ 1.6} & 38.9\textsubscript{$\pm$ 1.9} & 36.3\textsubscript{$\pm$ 1.9} \\
        & & 2 & 21.2\textsubscript{$\pm$ 2.3} & 30.9\textsubscript{$\pm$ 1.6} & 29.0\textsubscript{$\pm$ 1.6} & 27.9\textsubscript{$\pm$ 2.2} & 25.9\textsubscript{$\pm$ 2.4} & 23.9\textsubscript{$\pm$ 1.9} & 22.4\textsubscript{$\pm$ 1.8} & 20.2\textsubscript{$\pm$ 1.9} \\
        & & 3 & 7.2\textsubscript{$\pm$ 2.3} & 15.5\textsubscript{$\pm$ 2.6} & 14.3\textsubscript{$\pm$ 2.9} & 13.0\textsubscript{$\pm$ 2.8} & 12.0\textsubscript{$\pm$ 2.8} & 10.9\textsubscript{$\pm$ 2.5} & 9.6\textsubscript{$\pm$ 2.5} & 7.9\textsubscript{$\pm$ 2.0} \\
        
		\bottomrule
	\end{tabular}
}
\vspace{-4mm}
\end{table}

\begin{table}[tp]
	\centering
	\small
	\centering
	\caption{Certified accuracy (CA) $[\%]$ under joint $\ell_0$- and $\ell_1$-perturbations of categorical and numerical features, respectively, depending on whether model uses categorical and/or numerical features. The natural accuracy is the CA at radius $r = 0.0$. Larger is better.}
	\label{tab:joint-cv-imbalanced-l1}
	\vspace{2mm}
	\resizebox{0.98\columnwidth}{!}{
    \begin{tabular}{ccccccccccc}
        \toprule
        \multirow{2.6}{*}{Dataset} & \multirowcell{2.6}{Categorical \\ Features} & \multirowcell{2.6}{$\ell_0$ Radius $r_0$} &\multirowcell{2.6}{ CA without \\ Numerical Features} & \multicolumn{7}{c}{CA with Numerical Features at $\ell_1$ Radius $r_1$}\\
        \cmidrule(lr){5-11}
        & & & & 0.00 & 0.25 & 0.50 & 0.75 & 1.00 & 1.25 & 1.50 \\
        \midrule
        \multirow{5.5}{*}{\adult} & no & - & - & 80.0\textsubscript{$\pm$ 0.3} & 77.3\textsubscript{$\pm$ 0.4} & 73.3\textsubscript{$\pm$ 0.4} & 69.8\textsubscript{$\pm$ 0.5} & 64.7\textsubscript{$\pm$ 0.6} & 61.5\textsubscript{$\pm$ 0.6} & 58.1\textsubscript{$\pm$ 0.7} \\
        \cmidrule(lr){2-2}
        & \multirow{4.0}{*}{yes} & 0 & 69.7\textsubscript{$\pm$ 0.6} & 69.8\textsubscript{$\pm$ 0.6} & 69.2\textsubscript{$\pm$ 0.6} & 68.3\textsubscript{$\pm$ 0.6} & 67.4\textsubscript{$\pm$ 0.7} & 65.8\textsubscript{$\pm$ 0.6} & 64.7\textsubscript{$\pm$ 0.7} & 63.3\textsubscript{$\pm$ 0.6} \\
        & & 1 & 52.0\textsubscript{$\pm$ 0.8} & 53.5\textsubscript{$\pm$ 0.7} & 52.9\textsubscript{$\pm$ 0.7} & 52.3\textsubscript{$\pm$ 0.7} & 51.5\textsubscript{$\pm$ 0.7} & 50.6\textsubscript{$\pm$ 0.8} & 49.7\textsubscript{$\pm$ 0.9} & 48.6\textsubscript{$\pm$ 0.9} \\
        & & 2 & 27.5\textsubscript{$\pm$ 0.6} & 31.8\textsubscript{$\pm$ 0.5} & 30.7\textsubscript{$\pm$ 0.5} & 29.7\textsubscript{$\pm$ 0.4} & 28.9\textsubscript{$\pm$ 0.4} & 28.1\textsubscript{$\pm$ 0.5} & 27.6\textsubscript{$\pm$ 0.4} & 27.1\textsubscript{$\pm$ 0.5} \\
        & & 3 & 6.6\textsubscript{$\pm$ 0.2} & 12.1\textsubscript{$\pm$ 0.3} & 10.8\textsubscript{$\pm$ 0.3} & 9.9\textsubscript{$\pm$ 0.2} & 9.1\textsubscript{$\pm$ 0.2} & 8.3\textsubscript{$\pm$ 0.2} & 8.0\textsubscript{$\pm$ 0.2} & 7.8\textsubscript{$\pm$ 0.1} \\
        
        \midrule
        \multirow{5.5}{*}{\credit} & no & - & - & 69.8\textsubscript{$\pm$ 2.2} & 66.3\textsubscript{$\pm$ 2.0} & 62.7\textsubscript{$\pm$ 1.8} & 59.1\textsubscript{$\pm$ 0.9} & 53.3\textsubscript{$\pm$ 2.0} & 42.3\textsubscript{$\pm$ 1.6} & 35.1\textsubscript{$\pm$ 3.5}  \\
        \cmidrule(lr){2-2}
        & \multirow{4.0}{*}{yes} & 0 &  58.3\textsubscript{$\pm$ 9.5} & 58.0\textsubscript{$\pm$ 9.9} & 57.6\textsubscript{$\pm$ 9.8} & 57.0\textsubscript{$\pm$ 9.7} & 56.4\textsubscript{$\pm$ 9.5} & 55.8\textsubscript{$\pm$ 9.7} & 54.2\textsubscript{$\pm$ 10.5} & 53.6\textsubscript{$\pm$ 10.8} \\ 
        & & 1 & 38.2\textsubscript{$\pm$ 7.4} & 39.0\textsubscript{$\pm$ 7.7} & 38.5\textsubscript{$\pm$ 7.7} & 38.3\textsubscript{$\pm$ 7.7} & 37.7\textsubscript{$\pm$ 7.7} & 37.2\textsubscript{$\pm$ 7.8} & 36.7\textsubscript{$\pm$ 7.8} & 35.5\textsubscript{$\pm$ 7.9} \\
        & & 2 & 21.7\textsubscript{$\pm$ 5.0} & 22.4\textsubscript{$\pm$ 5.3} & 22.0\textsubscript{$\pm$ 5.0} & 21.4\textsubscript{$\pm$ 5.4} & 20.8\textsubscript{$\pm$ 5.7} & 20.6\textsubscript{$\pm$ 5.7} & 20.0\textsubscript{$\pm$ 6.0} & 19.1\textsubscript{$\pm$ 5.8} \\
        & & 3 & 8.4\textsubscript{$\pm$ 2.6} & 9.5\textsubscript{$\pm$ 2.9} & 9.4\textsubscript{$\pm$ 2.7} & 8.9\textsubscript{$\pm$ 2.8} & 8.6\textsubscript{$\pm$ 2.9} & 8.3\textsubscript{$\pm$ 2.7} & 7.8\textsubscript{$\pm$ 2.9} & 7.5\textsubscript{$\pm$ 2.9} \\

        \midrule
        \multirow{4.5}{*}{\mammal} & no & - & - & 49.6\textsubscript{$\pm$ 3.7} & 48.0\textsubscript{$\pm$ 3.2} & 47.0\textsubscript{$\pm$ 3.3} & 47.0\textsubscript{$\pm$ 3.3} & 46.8\textsubscript{$\pm$ 3.1} & 44.6\textsubscript{$\pm$ 3.1} & 44.5\textsubscript{$\pm$ 3.1}  \\
        \cmidrule(lr){2-2}
        & \multirow{3.0}{*}{yes} & 0 & 79.0\textsubscript{$\pm$ 1.3} & 76.6\textsubscript{$\pm$ 1.7} & 76.5\textsubscript{$\pm$ 1.5} & 76.3\textsubscript{$\pm$ 1.6} & 76.3\textsubscript{$\pm$ 1.6} & 76.3\textsubscript{$\pm$ 1.6} & 74.2\textsubscript{$\pm$ 2.4} & 74.2\textsubscript{$\pm$ 2.4} \\
        & & 1 & 31.8\textsubscript{$\pm$ 4.5} & 40.0\textsubscript{$\pm$ 5.4} & 38.5\textsubscript{$\pm$ 4.6} & 37.8\textsubscript{$\pm$ 4.7} & 37.8\textsubscript{$\pm$ 4.7} & 37.6\textsubscript{$\pm$ 4.7} & 36.1\textsubscript{$\pm$ 4.8} & 36.0\textsubscript{$\pm$ 5.0} \\
        & & 2 & 0.0\textsubscript{$\pm$ 0.0} & 0.0\textsubscript{$\pm$ 0.0} & 0.0\textsubscript{$\pm$ 0.0} & 0.0\textsubscript{$\pm$ 0.0} & 0.0\textsubscript{$\pm$ 0.0} & 0.0\textsubscript{$\pm$ 0.0} & 0.0\textsubscript{$\pm$ 0.0} & 0.0\textsubscript{$\pm$ 0.0} \\

        \midrule
        \multirow{5.5}{*}{\bank} & no & - & - & 80.4\textsubscript{$\pm$ 3.5} & 77.4\textsubscript{$\pm$ 3.0} & 75.1\textsubscript{$\pm$ 3.2} & 72.2\textsubscript{$\pm$ 3.1} & 68.4\textsubscript{$\pm$ 2.5} & 63.7\textsubscript{$\pm$ 2.2} & 57.0\textsubscript{$\pm$ 1.1}  \\
        \cmidrule(lr){2-2}
        & \multirow{4.0}{*}{yes} & 0 & 64.4\textsubscript{$\pm$ 7.5} & 74.0\textsubscript{$\pm$ 5.4} & 72.6\textsubscript{$\pm$ 5.1} & 71.6\textsubscript{$\pm$ 5.0} & 70.2\textsubscript{$\pm$ 5.0} & 68.3\textsubscript{$\pm$ 5.4} & 66.3\textsubscript{$\pm$ 5.4} & 63.2\textsubscript{$\pm$ 5.7} \\
        & & 1 & 44.4\textsubscript{$\pm$ 7.9} & 56.5\textsubscript{$\pm$ 6.6} & 55.2\textsubscript{$\pm$ 6.3} & 53.9\textsubscript{$\pm$ 6.4} & 52.3\textsubscript{$\pm$ 6.3} & 50.2\textsubscript{$\pm$ 6.3} & 48.2\textsubscript{$\pm$ 6.2} & 44.9\textsubscript{$\pm$ 6.3} \\
        & & 2 & 23.2\textsubscript{$\pm$ 7.1} & 38.5\textsubscript{$\pm$ 6.9} & 37.2\textsubscript{$\pm$ 6.5} & 36.0\textsubscript{$\pm$ 6.2} & 34.6\textsubscript{$\pm$ 5.9} & 32.8\textsubscript{$\pm$ 5.8} & 30.7\textsubscript{$\pm$ 5.7} & 27.7\textsubscript{$\pm$ 5.4} \\
        & & 3 & 8.3\textsubscript{$\pm$ 4.7} & 20.2\textsubscript{$\pm$ 6.6} & 19.4\textsubscript{$\pm$ 6.4} & 18.4\textsubscript{$\pm$ 6.0} & 17.2\textsubscript{$\pm$ 5.6} & 16.0\textsubscript{$\pm$ 5.0} & 14.6\textsubscript{$\pm$ 4.6} & 12.3\textsubscript{$\pm$ 4.2} \\

		\bottomrule
	\end{tabular}
}
\vspace{-4mm}
\end{table}

\begin{table}[tp]
	\centering
	\small
	\centering
	\caption{Certified accuracy (CA) $[\%]$ and balanced certified accuracy (BCA) $[\%]$ under $\ell_0$-perturbations of categorical features. Larger is better.}
	\label{tab:categorical}
	\vspace{2mm}
	\resizebox{0.4\columnwidth}{!}{
    \begin{tabular}{cccc}
        \toprule
        Dataset & $\ell_0$ Radius $r_0$ & CA & BCA \\
        \midrule
        \multirow{4.5}{*}{\mushroom} & 0 & 90.6\textsubscript{$\pm$ 0.7} & 90.4\textsubscript{$\pm$ 0.9}  \\
        & 1 & 87.1\textsubscript{$\pm$ 1.7} & 86.9\textsubscript{$\pm$ 1.6}  \\
        & 2 & 81.2\textsubscript{$\pm$ 3.6} & 81.1\textsubscript{$\pm$ 3.4} \\
        & 3 & 70.5\textsubscript{$\pm$ 5.7} & 70.7\textsubscript{$\pm$ 5.5}  \\
		\bottomrule
	\end{tabular}
}
\vspace{-4mm}
\end{table}

In \cref{tab:joint-cv-balanced-l2}, we report the mean and standard deviation (over a 5-fold cross-validation) of the balanced certified accuracies at a range of $\ell_2$ radii over the numerical features given varying perturbation levels of the categorical features for all datasets containing both numerical and categorical features.
We report the corresponding imbalanced certified accuracies in \cref{tab:joint-cv-imbalanced-l2} and similar results for $\ell_1$ radii in \cref{tab:joint-cv-balanced-l1,tab:joint-cv-imbalanced-l1}. 

We again observe that models utilizing both categorical and numerical features outperform those using only either one on clean data.
Interestingly, the slower drop in certified accuracy with increasing perturbation of the numerical features is much more pronounced in the $\ell_1$-setting, and much higher certified accuracies are obtained even at large radii.
For example, on \adult, considering only numerical features leads to a BCA of $62.1\%_{\pm 0.3\%}$ at radius $r_1=0.0$ dropping to $42.0\%_{\pm 0.3\%}$ at $r_1=1.5$.
In contrast, when also utilizing categorical features, the BCA at $r_1=0.0$ is $76.9\%_{\pm 0.5\%}$, only dropping to $71.1\%_{\pm 0.6\%}$ at $r_1=1.5$, when no categorical variable is perturbed ($r_0 = 0$).
Similarly, when at most one categorical variable is perturbed, the BCA at $r_1=0.0$ is $59.7\%_{\pm 0.6\%}$ and only drops to $54.4\%_{\pm 0.8\%}$ radius $r_1=1.5$.
This highlights again that, when available, utilizing categorical features in addition to numerical ones is essential to improve accuracy and make models more certifiably robust.

While standard deviations are generally moderately low, the sensitivity to different train/test-splits is particularly small for datasets with many samples like \adult (nearly $50'000$ samples).

\tool is also applicable to data sets involving only categorical features as ca be seen in \cref{tab:categorical}, where we report results on \mushroom.
As expected, we observe that both balanced and imbalanced certifiable accuracy decrease as we permit more and more categorical features to be perturbed.

\subsection{Derandomized vs. Randomized Smoothing}
\label{app:eval-ds-rs}
\begin{table}[tp]
	\centering
	\small
	\centering
	\caption{We compare certifying the same stump ensembles via Deterministic Smoothing (\tool) and Randomized Smoothing (\RS) with respect to the average certified radius (ACR) and the certified accuracy [\%] at numerous radii $r$ on \mnistof for $\ell_1$ ($\lambda=4.0$) and $\ell_2$ ($\sigma=0.5$) norm perturbations. Larger is better.}
	\vspace{3mm}
	\label{tab:ablation_drs_vs_rs}
	\renewcommand{\arraystretch}{1.1}
	\resizebox{0.8\columnwidth}{!}{
    \begin{tabular}{ccccccccccc}
        \toprule
        \multirow{2.6}{*}{Norm} & \multirow{2.6}{*}{Method} & \multirow{2.6}{*}{ACR} & \multicolumn{8}{c}{Certified Accuracy at Radius r}\\
        \cmidrule(lr){4-11}
        & & & 0.0 & 0.50 & 1.00 & 1.50 & 2.00 & 2.50 & 3.00 & 3.50 \\
        \midrule
        \multirow{5.0}{*}{$\ell_1$} & \RS ($n=100$) & 2.809 & 93.0 & 91.2 & 88.6 & 86.2 & 82.9 & 77.0 & 68.8 & 0.0 \\
        & \RS ($n=1000$) & 3.337 & 95.6 & 94.4 & 92.8 & 90.6 & 87.8 & 84.7 & 79.5 & 70.4 \\
        & \RS ($n=10 000$) & 3.430 & 96.0 & 95.3 & 93.7 & 91.6 & 89.4 & 85.8 & 82.1 & 73.8 \\
        & \RS ($n=100 000$) & 3.456 & 96.1 & 95.5 & 94.0 & 91.9 & \textbf{89.9} & 86.3 & 82.9 & 74.6 \\
        & \tool (ours) & \textbf{3.467} & \textbf{96.6} & \textbf{95.6} & \textbf{94.1} & \textbf{92.1} & \textbf{89.9} & \textbf{86.5} & \textbf{83.1} & \textbf{75.1} \\
        \midrule
        \multirow{5.0}{*}{$\ell_2$} & \RS ($n=100$) & 0.680 & 94.8 & 90.1 & 0.0 & 0.0 & 0.0 & 0.0 & 0.0 & 0.0 \\
        & \RS ($n=1000$) & 1.102 & 95.6 & 92.5 & 85.0 & 0.0 & 0.0 & 0.0 & 0.0 & 0.0 \\
        & \RS ($n=10 000$) & 1.403 & 95.9 & 92.9 & 86.9 & 75.0 & 0.0 & 0.0 & 0.0 & 0.0 \\
        & \RS ($n=100 000$) & 1.627 & 95.9 & \textbf{93.0} & 87.3 & 78.1 & 0.0 & 0.0 & 0.0 & 0.0 \\
        & \tool (ours) & \textbf{2.161} & \textbf{96.0} & \textbf{93.0} & \textbf{87.5} & \textbf{79.0} & \textbf{65.3} & \textbf{40.5} & \textbf{12.3} & \textbf{5.9} \\
		\bottomrule
	\end{tabular}
}
\end{table}

Here, we compare evaluating stump ensembles deterministically via \tool (\cref{sec:det_smoothing}) to sampling-based \RS \cite{CohenRK19}.
In \cref{tab:ablation_drs_vs_rs}, we provide quantitative results corresponding to \cref{fig:ablation-rs-ds}, expanded by an equivalent experiment for $\ell_1$-norm perturbations.
We observe that as sampling-based \RS uses increasingly more samples, it converges towards \tool. This convergence is much faster in the $\ell_1$-setting.
However, especially in the $\ell_2$-setting, a notable gap remains even when using as many as $100\,000$ samples.
This is expected as sampling-based \RS computes a lower confidence bound to the true success probability, which can be computed exactly with \tool. Thus the higher the desired confidence, the larger this gap will be. Further, if \RS were to yield a larger radius than \tool, this would actually be an error, occurring with probability $\alpha$, as \tool computes the true maximum certifiable radius.
This highlights another key difference: \RS provides probabilistic guarantees that hold with confidence $1-\alpha$, while \tool provides deterministic guarantees.
Moreover, for \RS, many samples have to be evaluated (typically $n=100\,000$), while \tool can efficiently compute the exact CDF.
We note that the much larger improvement in certified radii observed for $\ell_2$-norm perturbations is due to the significantly higher sensitivity of the certifiably radius w.r.t. the success probability (see \cref{tab:rs}).

\subsection{MLE Optimality Criterion}
\label{app:eval-mle-ablation}

\begin{table}[tp]
	\centering
	\small
	\centering
	\caption{We compare training stump ensembles optimally via MLE-optimal criterion, training them via noisy sampling (Sampling) and default training (Default) with respect to the average certified radius (ACR) and the certified accuracy [\%] on \mnistts at numerous radii $r$ on various norms for multiple noise magnitudes ($\lambda$ for $\ell_1$ and $\sigma$ for $\ell_2$). Larger is better.}
		\vspace{3mm}
	\label{tab:mle-ablation-extended}
	\renewcommand{\arraystretch}{1.2}
	\resizebox{0.8\columnwidth}{!}{
    \begin{tabular}{cccccccccccc}
        \toprule
        \multirow{2.6}{*}{Norm} & \multirow{2.6}{*}{$\lambda$ ($\ell_1$) or $\sigma$ ($\ell_2$)} & \multirow{2.6}{*}{Method} & \multirow{2.6}{*}{ACR} & \multicolumn{8}{c}{Certified Accuracy at Radius r}\\
        \cmidrule(lr){5-12}
        & & & & 0.0 & 0.5 & 1.0 & 1.5 & 2.0 & 2.5 & 3.0 & 3.5 \\
        \midrule
        \multirow{9.0}{*}{$\ell_1$} & \multirow{3.0}{*}{1.0} & Default & 0.519 & 51.9 & 51.9 & 51.9 & 0.0 & 0.0 & 0.0 & 0.0 & 0.0\\
        & & Sampling & 0.928 & \textbf{96.2} & 93.9 & 64.8 & 0.0 & 0.0 & 0.0 & 0.0 & 0.0 \\
        & & MLE (Ours) & \textbf{0.931} & \textbf{96.2} & \textbf{94.3} & \textbf{66.2} & 0.0 & 0.0 & 0.0 & 0.0 & 0.0  \\
        \cmidrule(lr){2-12}
        & \multirow{3.0}{*}{4.0} & Default & 2.074 & 51.9 & 51.9 & 51.9 & 51.9 & 51.9 & 51.9 & 51.9 & 51.9\\
        & & Sampling & 3.166 & \textbf{96.3} & 95.0 & 93.3 & 90.5 & 87.3 & 81.4 & 72.5 & 56.0 \\
        & & MLE (Ours) & \textbf{3.282} & \textbf{96.3} & \textbf{95.4} & \textbf{93.9} & \textbf{91.7} & \textbf{88.7} & \textbf{84.1} & \textbf{76.0} & \textbf{62.8}  \\
        \cmidrule(lr){2-12}
        & \multirow{3.0}{*}{16.0} & Default & 8.297 & 51.9 & 51.9 & 51.9 & 51.9 & 51.9 & 51.9 & 51.9 & 51.9  \\
        & & Sampling & 6.646 & \textbf{96.4} & 95.3 & 94.4 & 93.4 & 91.8 & 90.0 & 87.8 & 84.9 \\
        & & MLE (Ours) & \textbf{8.574} & 96.2 & \textbf{95.7} & \textbf{95.0} & \textbf{94.1} & \textbf{93.2} & \textbf{91.7} & \textbf{90.6} & \textbf{88.4} \\
        \midrule
        \multirow{9.0}{*}{$\ell_2$} & \multirow{3.0}{*}{0.25} & Default & 0.967 & 51.9 & 51.9 & 51.8 & 48.7 & 0.0 & 0.0 & 0.0 & 0.0\\
        & & Sampling & 1.628 & \textbf{96.3} & 92.8 & 85.9 & 71.7 & 0.0 & 0.0 & 0.0 & 0.0  \\
        & & MLE (Ours) & \textbf{1.642} & \textbf{96.3} & \textbf{93.0} & \textbf{86.3} & \textbf{73.0} & 0.0 & 0.0 & 0.0 & 0.0  \\
        \cmidrule(lr){2-12}
        & \multirow{3.0}{*}{1.0} & Default & \textbf{3.436} & 51.9 & 51.9 & 51.9 & 51.9 & \textbf{51.9} & \textbf{51.9} & \textbf{51.9} & \textbf{51.9} \\
        & & Sampling &  1.594 & \textbf{95.5} & 89.1 & 76.5 & 57.9 & 33.5 & 11.7 & 2.1 & 0.2 \\
        & & MLE (Ours) & 1.724 & \textbf{95.5} & \textbf{90.1} & \textbf{79.2} & \textbf{62.5} & 40.3 & 18.7 & 5.4 & 1.3   \\
        \cmidrule(lr){2-12}
        & \multirow{3.0}{*}{4.0} & Default & \textbf{12.167} & 51.9 & 51.9 & 51.9 & 51.9 & \textbf{51.9} & \textbf{51.9} & \textbf{51.9} & \textbf{51.9} \\
        & & Sampling & 1.095 & 89.2 & 72.9 & 50.9 & 32.6 & 15.8 & 4.0 & 0.5 & 0.0 \\
        & & MLE (Ours) & 1.652 & \textbf{95.1} & \textbf{88.7} & \textbf{76.5} & \textbf{59.2} & 36.6 & 16.3 & 4.9 & 1.5 \\
		\bottomrule
	\end{tabular}
}
\end{table}

In \cref{tab:mle-ablation-extended}, we compare our robust MLE optimality criterion (MLE) to applying the standard entropy criterion to samples drawn from the input randomization scheme (Sampling) or the clean data (Default).
We observe that training approaches accounting for randomness (i.e., Sampling and MLE) consistently outperform default training.
In some cases, default training even suffers from a mode collapse, always predicting the same class.
Amongst the two methods accounting for the input randomization, our MLE optimality criterion consistently outperforms samplings at all noise magnitudes and for both perturbation types.
This effect is particularly pronounced at large noise magnitudes, where sampling becomes less effective at capturing the input distribution.

\pagebreak
\subsection{Effect of Noise Level}
\label{app:eval-noise-level}

\begin{table}[tp]
	\centering
	\small
	\centering
	\caption{Comparison of average certified radius (ACR) and certified accuracy at various radii $r$ with respect to the $\ell_1$ norm for numerous datasets and noise magnitudes $\lambda$. Larger is better.}
	\vspace{2mm}
	\label{tab:ablation-noise-magnitude-l1}
	\resizebox{0.85\columnwidth}{!}{
    \begin{tabular}{cccccccccccccc}
        \toprule
        \multirow{2.6}{*}{Dataset} & \multirow{2.6}{*}{$\lambda$} & \multirow{2.6}{*}{ACR} & \multicolumn{11}{c}{Certified Accuracy at Radius r}\\
        \cmidrule(lr){4-14}
        & & & 0.0 & 1.0 & 2.0 & 3.0 & 4.0 & 5.0 & 6.0 & 7.0 & 8.0 & 9.0 & 10.0 \\
        \midrule
        \multirow{6.0}{*}{\fmnists} & 0.5 & 0.407 & 84.4 & 0.0 & 0.0 & 0.0 & 0.0 & 0.0 & 0.0 & 0.0 & 0.0 & 0.0 & 0.0  \\
        & 1.0 & 0.766 & 83.5 & 55.1 & 0.0 & 0.0 & 0.0 & 0.0 & 0.0 & 0.0 & 0.0 & 0.0 & 0.0 \\
        & 2.0 & 1.463 & 83.7 & 74.9 & 47.0 & 0.0 & 0.0 & 0.0 & 0.0 & 0.0 & 0.0 & 0.0 & 0.0 \\
        & 4.0 & 2.780 & \textbf{85.8} & 80.2 & 73.3 & 60.9 & 21.3 & 0.0 & 0.0 & 0.0 & 0.0 & 0.0 & 0.0 \\
        & 8.0 & 4.755 & 83.9 & 80.0 & 75.5 & 70.3 & 63.9 & 56.4 & 46.5 & 32.6 & 1.9 & 0.0 & 0.0 \\
        & 16.0 & \textbf{7.975} & 84.3 & \textbf{81.7} & \textbf{77.8} & \textbf{75.0} & \textbf{71.7} & \textbf{67.3} & \textbf{63.2} & \textbf{57.9} & \textbf{52.9} & \textbf{47.0} & \textbf{41.1} \\
        \midrule
        \multirow{6.0}{*}{\mnistof} & 0.5 & 0.476 & 96.3 & 0.0 & 0.0 & 0.0 & 0.0 & 0.0 & 0.0 & 0.0 & 0.0 & 0.0 & 0.0 \\
        & 1.0 & 0.934 & 96.3 & 77.0 & 0.0 & 0.0 & 0.0 & 0.0 & 0.0 & 0.0 & 0.0 & 0.0 & 0.0 \\
        & 2.0 & 1.808 & 96.2 & 92.1 & 62.8 & 0.0 & 0.0 & 0.0 & 0.0 & 0.0 & 0.0 & 0.0 & 0.0 \\
        & 4.0 & 3.467 & 96.6 & 94.1 & 89.9 & 83.1 & 39.1 & 0.0 & 0.0 & 0.0 & 0.0 & 0.0 & 0.0 \\
        & 8.0 & 6.472 & \textbf{97.0} & \textbf{95.4} & \textbf{93.3} & \textbf{91.0} & \textbf{87.4} & \textbf{82.2} & \textbf{75.1} & 60.7 & 4.4 & 0.0 & 0.0 \\
        & 16.0 & \textbf{8.957} & 90.4 & 88.6 & 86.6 & 83.5 & 80.3 & 77.4 & 72.9 & \textbf{67.4} & \textbf{61.9} & \textbf{56.2} & \textbf{49.6} \\
        \midrule
        \multirow{6.0}{*}{\mnistts} & 0.5 & 0.477 & 96.3 & 0.0 & 0.0 & 0.0 & 0.0 & 0.0 & 0.0 & 0.0 & 0.0 & 0.0 & 0.0 \\
        & 1.0 & 0.931 & 96.2 & 66.2 & 0.0 & 0.0 & 0.0 & 0.0 & 0.0 & 0.0 & 0.0 & 0.0 & 0.0 \\
        & 2.0 & 1.780 & 96.2 & 92.2 & 43.0 & 0.0 & 0.0 & 0.0 & 0.0 & 0.0 & 0.0 & 0.0 & 0.0  \\
        & 4.0 & 3.282 & 96.3 & 93.9 & 88.7 & 76.0 & 3.8 & 0.0 & 0.0 & 0.0 & 0.0 & 0.0 & 0.0 \\
        & 8.0 & 5.617 & \textbf{96.5} & 94.6 & 91.4 & 87.4 & 80.9 & 71.7 & 56.6 & 31.3 & 0.0 & 0.0 & 0.0 \\
        & 16.0 & \textbf{8.574} & 96.2 & \textbf{95.0} & \textbf{93.2} & \textbf{90.6} & \textbf{86.5} & \textbf{82.7} & \textbf{77.5} & \textbf{70.5} & \textbf{62.7} & \textbf{53.3} & \textbf{41.3} \\
		\bottomrule
	\end{tabular}
}
\vspace{-3mm}
\end{table}

\begin{table}[tp]
	\centering
	\small
	\centering
	\caption{Comparison of average certified radius (ACR) and certified accuracy at various radii $r$ with respect to the $\ell_2$ norm for numerous datasets and noise magnitudes $\sigma$. Larger is better.}
	\vspace{2mm}
	\label{tab:ablation-noise-magnitude-l2}
	\resizebox{0.8\columnwidth}{!}{
    \begin{tabular}{ccccccccccc}
        \toprule
        \multirow{2.6}{*}{Dataset} & \multirow{2.6}{*}{$\sigma$} & \multirow{2.6}{*}{ACR} & \multicolumn{8}{c}{Certified Accuracy at Radius r}\\
        \cmidrule(lr){4-11}
        & & & 0.0 & 0.5 & 1.0 & 1.5 & 2.0 & 2.5 & 3.0 & 3.5 \\
        \midrule
        \multirow{6.0}{*}{\fmnists} & 0.25 & 1.361 & \textbf{86.8} & \textbf{79.6} & 70.0 & \textbf{58.2} & 0.0 & 0.0 & 0.0 & 0.0 \\
        & 0.5 & 1.723 & 86.5 & 78.9 & \textbf{70.1} & 56.6 & 42.2 & 27.8 & 17.4 & 8.4 \\
        & 1.0 & 1.699 & 86.2 & 78.5 & 69.1 & 55.7 & 41.0 & 25.8 & 16.9 & 8.9 \\
        & 2.0 & 1.681 & 86.2 & 78.5 & 68.8 & 55.1 & 40.2 & 25.6 & 16.8 & 8.7 \\
        & 4.0 & \textbf{2.136} & 57.1 & 52.2 & 49.7 & 47.9 & \textbf{46.0} & \textbf{43.4} & \textbf{39.4} & \textbf{35.0} \\
        & 8.0 & 1.518 & 83.7 & 74.2 & 64.4 & 51.0 & 35.4 & 21.0 & 10.4 & 4.8 \\
        \midrule
        \multirow{6.0}{*}{\mnistof} & 0.25 & 1.737 & 95.8 & \textbf{93.6} & \textbf{89.0} & \textbf{82.8} & 0.0 & 0.0 & 0.0 & 0.0 \\
        & 0.5 & \textbf{2.161} & 96.0 & 93.0 & 87.5 & 79.0 & \textbf{65.3} & \textbf{40.5} & 12.3 & 5.9 \\
        & 1.0 & 2.044 & 96.0 & 92.7 & 86.1 & 75.6 & 57.3 & 28.4 & 11.4 & 6.7 \\
        & 2.0 & 2.012 & 95.8 & 92.7 & 85.8 & 74.9 & 56.2 & 26.9 & 10.3 & 6.0  \\
        & 4.0 & 1.875 & 94.8 & 87.2 & 71.8 & 48.1 & 34.7 & 29.9 & \textbf{23.8} & \textbf{15.7} \\
        & 8.0 & 1.808 & \textbf{96.1} & 90.2 & 80.3 & 62.9 & 36.4 & 20.4 & 13.3 & 7.7 \\
        \midrule
        \multirow{6.0}{*}{\mnistts} & 0.25 & 1.642 & \textbf{96.3} & \textbf{93.0} & \textbf{86.3} & \textbf{73.0} & 0.0 & 0.0 & 0.0 & 0.0 \\
        & 0.5 & \textbf{1.824} & 95.8 & 91.2 & 81.9 & 66.7 & 46.4 & 23.4 & 7.4 & 1.3 \\
        & 1.0 & 1.724 & 95.5 & 90.1 & 79.2 & 62.5 & 40.3 & 18.7 & 5.4 & 1.3 \\
        & 2.0 & 1.688 & 95.4 & 89.5 & 78.0 & 60.9 & 38.8 & 17.5 & 4.9 & 1.0  \\
        & 4.0 & 1.652 & 95.1 & 88.7 & 76.5 & 59.2 & 36.6 & 16.3 & 4.9 & 1.5 \\
        & 8.0 & 1.718 & 74.3 & 61.0 & 53.2 & 49.4 & \textbf{46.6} & \textbf{40.2} & \textbf{30.3} & \textbf{17.4} \\
		\bottomrule
	\end{tabular}
}
\vspace{-3mm}
\end{table}

\begin{wrapfigure}[11]{r}{0.40\textwidth}
	\centering
	\vspace{-12.5mm}
	\includegraphics[width=0.97\linewidth]{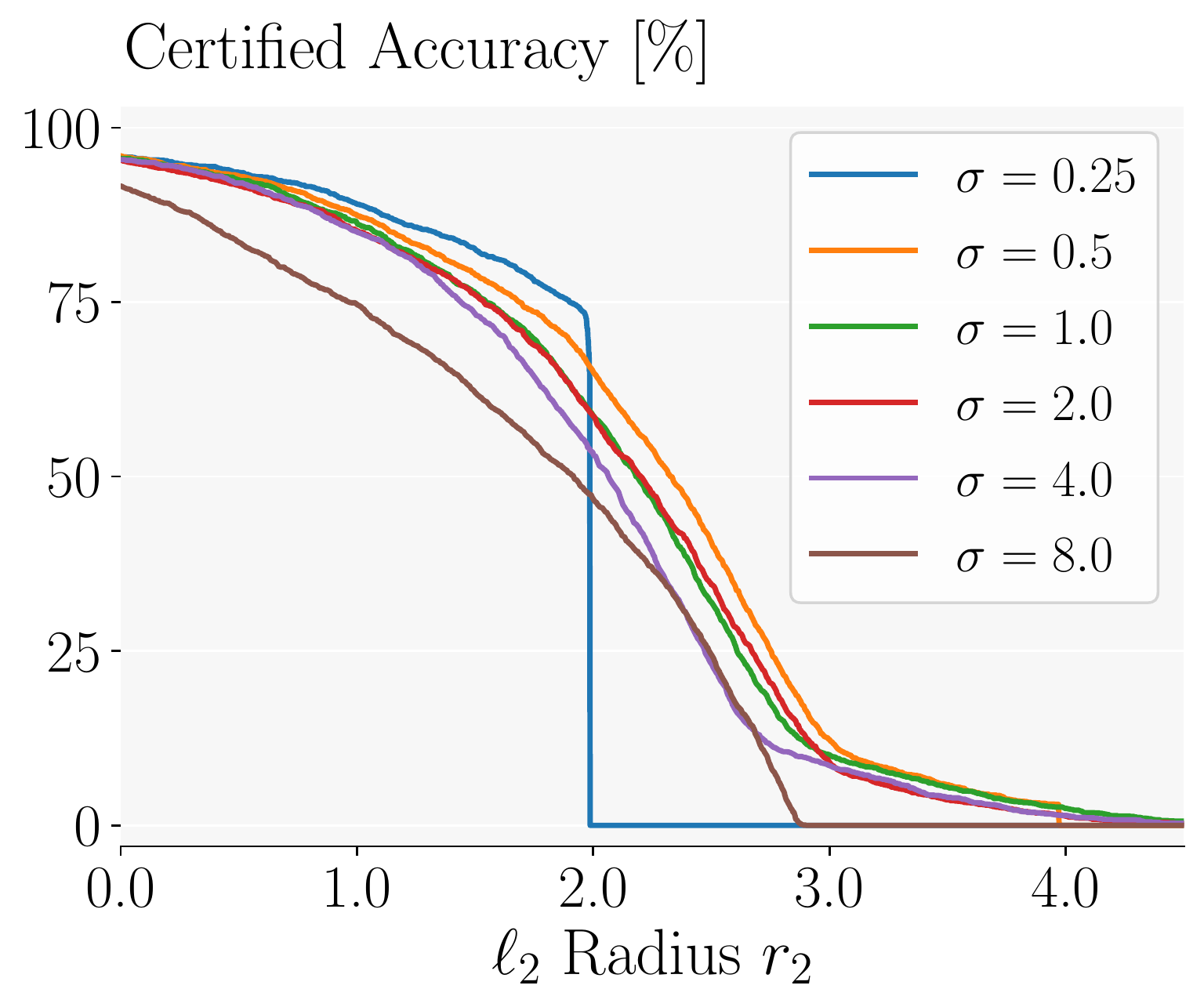}
	\vspace{-2.6mm}
	\caption{Comparing \tool for various noise levels $\sigma$ on \mnistof.}
	\label{fig:ablation-sigma-l2}
\end{wrapfigure}

Here, we provide additional experiments for varying noise magnitudes, $\lambda$ for $\ell_1$-certification, and $\sigma$ for $\ell_2$-certification.
In \cref{tab:ablation-noise-magnitude-l1,tab:ablation-noise-magnitude-l2}, we provide extensive experiments for the $\ell_1$- and $\ell_2$-setting, respectively, which we visualize in \cref{fig:ablation-lambda-l1,fig:ablation-sigma-l2}.

We observe that, in the $\ell_1$-setting, the natural accuracy (certified accuracy at radius $0$) is quite insensitive to an increase in noise magnitude. Consequently, large $\lambda$ lead to exceptionally large ACR and certified accuracies even at large radii, e.g., on \mnistts, we obtain a certified accuracy of $82.7\%$ at $\ell_1$-radius $r=5.0$.

In the $\ell_2$-setting, increasing the noise magnitude $\sigma$ generally leads to a more pronounced drop in natural and certified accuracy, and thus similar ACRs for various noise magnitudes.

\begin{wrapfigure}[15]{r}{0.39\textwidth}
	\centering
	\vspace{-5.2mm}
	\includegraphics[width=0.99\linewidth]{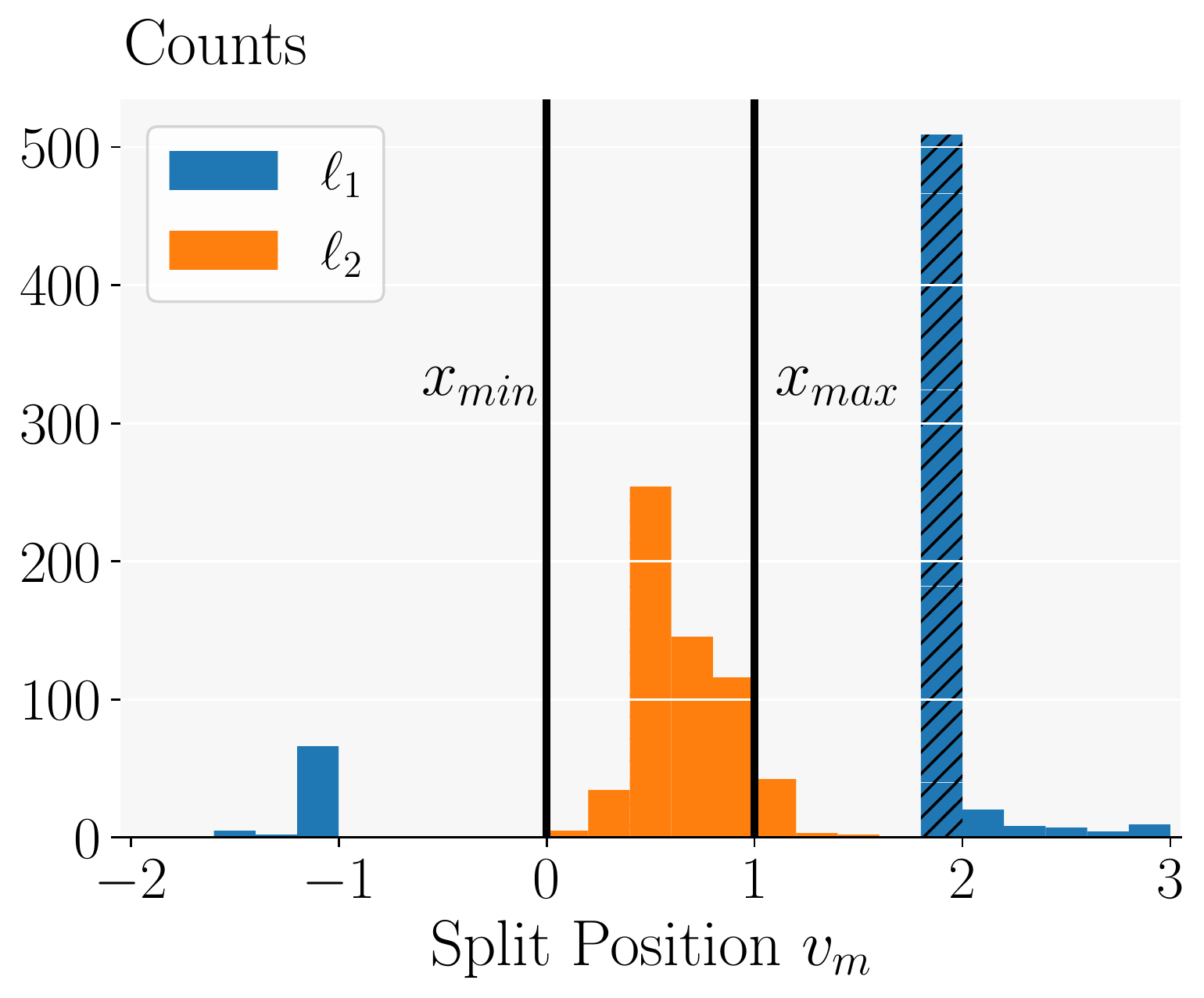}
	\vspace{-6.2mm}
	\caption{Comparing counts for values of $v_m$ on \mnistts for $\ell_1$ and $\ell_2$ norms with $\lambda=2.0$ and $\sigma=0.25$, respectively.}
	\label{fig:decision-thresholds}
\end{wrapfigure}

\paragraph{Thinking Outside the Box}
Analysing this surprising behaviour in the $\ell_1$-setting, we empirically find that despite the data being normalized to $[0,1]$, the MLE optimality criterion often yields split positions $v_m$ outside of $[0,1]$. Recall that there, uniformly distributed random noise is added to the original sample ($\vx' \sim \textit{Unif}([\vx-\lambda, \vx+\lambda]^d)$).
In \cref{fig:decision-thresholds}, we show a histogram of the split positions ($v_m$), illustrating this behaviour. In the $\ell_1$-setting and using $\lambda=2$, all split positions are either smaller than $-1$ or larger than $ 1.9$, which are exactly the borders of uniform distributions with $\lambda=2$ centered at the extremes of the image domain ($[0,1]$).
As all splits are outside the hyperbox constituting the original image domain, we refer to this behaviour as 'thinking outside the box'.
Intuitively, each unperturbed data point is on the same side of $v_m$ in this case, but when the randomization scheme is applied, a split outside of $[-1,2]$ leads to a probability mass of $0$ for an original feature value of $0$ or $1$, while for the other, the probability mass can be as high as $\frac{1}{2\lambda}$. Therefore, such splits allow the smoothed model to still separate these cases well for randomized inputs.

While we observe this effect on all datasets in the $\ell_1$-setting given a sufficiently large $\lambda$, it does not appear in the $\ell_2$-setting. There, $v_m$'s are typically clustered closely around or inside $[0,1]$, as the Gaussian randomization applied here has unbounded support and does not permit for such a clean separation, regardless of the choice of $v_m$.

\subsection{Leaf Prediction Discretization}
\label{app:discretization-size-ablation}

\begin{table}[tp]
	\centering
	\small
	\centering
	\caption{We compare the performance of models for different number of discretization sizes with respect to average certified radius (ACR) and given certified accuracies (CA) [\%] on \mnistof. We utilize $\lambda=4.0$ for $\ell_1$ and $\sigma=0.5$ for $\ell_2$. We report mean and standard deviation over $5$-fold cross-validation. Larger mean is better.}
	\label{tab:ablation_discretizations}
	\vspace{2mm}
	\resizebox{0.9\columnwidth}{!}{
    \begin{tabular}{ccccccccccc}
        \toprule
        \multirow{2.6}{*}{Norm} & \multirow{2.6}{*}{Discretizations} & \multirow{2.6}{*}{ACR} & \multicolumn{8}{c}{Certified Accuracy [\%] at Radius r} \\
        \cmidrule(lr){4-11}
        & & & 0.00 & 0.50 & 1.00 & 1.50 & 2.00 & 2.50 & 3.00 & 3.50 \\
        \midrule
        \multirow{10.0}{*}{$\ell_1$}& 2 & 1.860\textsubscript{$\pm$ 0.025}  & 60.2\textsubscript{$\pm$ 1.3}  & 52.9\textsubscript{$\pm$ 1.5}  & 46.6\textsubscript{$\pm$ 1.3}  & 44.5\textsubscript{$\pm$ 0.6}  & 44.5\textsubscript{$\pm$ 0.6}  & 44.5\textsubscript{$\pm$ 0.6}  & 44.5\textsubscript{$\pm$ 0.6}  & 44.3\textsubscript{$\pm$ 0.6} \\
        & 3 & 2.220\textsubscript{$\pm$ 0.023}  & 55.5\textsubscript{$\pm$ 0.6}  & 55.5\textsubscript{$\pm$ 0.6}  & 55.5\textsubscript{$\pm$ 0.6}  & 55.5\textsubscript{$\pm$ 0.6}  & 55.5\textsubscript{$\pm$ 0.6}  & 55.5\textsubscript{$\pm$ 0.6}  & 55.5\textsubscript{$\pm$ 0.6}  & 55.5\textsubscript{$\pm$ 0.6} \\
        & 5 & 2.220\textsubscript{$\pm$ 0.023}  & 55.5\textsubscript{$\pm$ 0.6}  & 55.5\textsubscript{$\pm$ 0.6}  & 55.5\textsubscript{$\pm$ 0.6}  & 55.5\textsubscript{$\pm$ 0.6}  & 55.5\textsubscript{$\pm$ 0.6}  & 55.5\textsubscript{$\pm$ 0.6}  & 55.5\textsubscript{$\pm$ 0.6}  & 55.5\textsubscript{$\pm$ 0.6} \\
        & 10 & 2.191\textsubscript{$\pm$ 0.033}  & 72.6\textsubscript{$\pm$ 1.2}  & 68.1\textsubscript{$\pm$ 1.3}  & 63.8\textsubscript{$\pm$ 1.0}  & 58.6\textsubscript{$\pm$ 1.0}  & 53.7\textsubscript{$\pm$ 0.8}  & 48.6\textsubscript{$\pm$ 0.7}  & 45.1\textsubscript{$\pm$ 0.6}  & 44.3\textsubscript{$\pm$ 0.6} \\
        & 50 & 3.502\textsubscript{$\pm$ 0.014}  & 97.0\textsubscript{$\pm$ 0.4}  & 96.2\textsubscript{$\pm$ 0.4}  & 94.9\textsubscript{$\pm$ 0.5}  & 93.2\textsubscript{$\pm$ 0.5}  & 90.9\textsubscript{$\pm$ 0.5}  & 88.0\textsubscript{$\pm$ 0.5}  & 83.8\textsubscript{$\pm$ 0.4}  & 76.9\textsubscript{$\pm$ 0.2} \\
        & 100 & 3.425\textsubscript{$\pm$ 0.015}  & 96.2\textsubscript{$\pm$ 0.4}  & 94.7\textsubscript{$\pm$ 0.5}  & 93.0\textsubscript{$\pm$ 0.5}  & 91.0\textsubscript{$\pm$ 0.5}  & 88.7\textsubscript{$\pm$ 0.5}  & 85.7\textsubscript{$\pm$ 0.3}  & 81.5\textsubscript{$\pm$ 0.4}  & 74.2\textsubscript{$\pm$ 0.5} \\
        & 500 & 3.375\textsubscript{$\pm$ 0.016}  & 95.1\textsubscript{$\pm$ 0.5}  & 93.6\textsubscript{$\pm$ 0.5}  & 91.8\textsubscript{$\pm$ 0.5}  & 89.8\textsubscript{$\pm$ 0.4}  & 87.4\textsubscript{$\pm$ 0.3}  & 84.4\textsubscript{$\pm$ 0.5}  & 80.0\textsubscript{$\pm$ 0.5}  & 72.5\textsubscript{$\pm$ 0.3} \\
        & 1'000 & 3.367\textsubscript{$\pm$ 0.016}  & 94.9\textsubscript{$\pm$ 0.6}  & 93.5\textsubscript{$\pm$ 0.6}  & 91.6\textsubscript{$\pm$ 0.5}  & 89.7\textsubscript{$\pm$ 0.5}  & 87.2\textsubscript{$\pm$ 0.3}  & 84.2\textsubscript{$\pm$ 0.5}  & 79.8\textsubscript{$\pm$ 0.6}  & 72.3\textsubscript{$\pm$ 0.3} \\
        & 5'000 & 3.364\textsubscript{$\pm$ 0.016}  & 94.9\textsubscript{$\pm$ 0.6}  & 93.4\textsubscript{$\pm$ 0.5}  & 91.6\textsubscript{$\pm$ 0.5}  & 89.6\textsubscript{$\pm$ 0.4}  & 87.2\textsubscript{$\pm$ 0.3}  & 84.2\textsubscript{$\pm$ 0.5}  & 79.7\textsubscript{$\pm$ 0.5}  & 72.4\textsubscript{$\pm$ 0.3} \\
        & 10'000 & 3.364\textsubscript{$\pm$ 0.016}  & 94.9\textsubscript{$\pm$ 0.6}  & 93.4\textsubscript{$\pm$ 0.5}  & 91.6\textsubscript{$\pm$ 0.5}  & 89.6\textsubscript{$\pm$ 0.4}  & 87.2\textsubscript{$\pm$ 0.3}  & 84.2\textsubscript{$\pm$ 0.5}  & 79.7\textsubscript{$\pm$ 0.6}  & 72.3\textsubscript{$\pm$ 0.3} \\
        \midrule
        \multirow{10.0}{*}{$\ell_2$} & 2 &1.766\textsubscript{$\pm$ 0.023}  & 44.5\textsubscript{$\pm$ 0.6}  & 44.5\textsubscript{$\pm$ 0.6}  & 44.5\textsubscript{$\pm$ 0.6}  & 44.5\textsubscript{$\pm$ 0.6}  & 44.5\textsubscript{$\pm$ 0.6}  & 44.5\textsubscript{$\pm$ 0.6}  & 44.5\textsubscript{$\pm$ 0.6}  & 44.5\textsubscript{$\pm$ 0.6}  \\
        & 3 & 2.204\textsubscript{$\pm$ 0.023}  & 55.5\textsubscript{$\pm$ 0.6}  & 55.5\textsubscript{$\pm$ 0.6}  & 55.5\textsubscript{$\pm$ 0.6}  & 55.5\textsubscript{$\pm$ 0.6}  & 55.5\textsubscript{$\pm$ 0.6}  & 55.5\textsubscript{$\pm$ 0.6}  & 55.5\textsubscript{$\pm$ 0.6}  & 55.5\textsubscript{$\pm$ 0.6} \\
        & 5 & 2.204\textsubscript{$\pm$ 0.023}  & 55.5\textsubscript{$\pm$ 0.6}  & 55.5\textsubscript{$\pm$ 0.6}  & 55.5\textsubscript{$\pm$ 0.6}  & 55.5\textsubscript{$\pm$ 0.6}  & 55.5\textsubscript{$\pm$ 0.6}  & 55.5\textsubscript{$\pm$ 0.6}  & 55.5\textsubscript{$\pm$ 0.6}  & 55.5\textsubscript{$\pm$ 0.6} \\
        & 10 & 2.044\textsubscript{$\pm$ 0.007}  & 95.9\textsubscript{$\pm$ 0.5}  & 92.2\textsubscript{$\pm$ 0.5}  & 85.9\textsubscript{$\pm$ 0.2}  & 75.8\textsubscript{$\pm$ 0.4}  & 56.9\textsubscript{$\pm$ 1.1}  & 27.8\textsubscript{$\pm$ 0.7}  & 14.1\textsubscript{$\pm$ 0.4}  & 7.6\textsubscript{$\pm$ 0.3} \\
        & 50 & 2.110\textsubscript{$\pm$ 0.006}  & 95.6\textsubscript{$\pm$ 0.5}  & 92.0\textsubscript{$\pm$ 0.6}  & 86.3\textsubscript{$\pm$ 0.3}  & 77.7\textsubscript{$\pm$ 0.4}  & 63.1\textsubscript{$\pm$ 0.6}  & 38.0\textsubscript{$\pm$ 0.9}  & 11.4\textsubscript{$\pm$ 0.3}  & 5.4\textsubscript{$\pm$ 0.3} \\
        & 100 & 2.120\textsubscript{$\pm$ 0.005}  & 95.3\textsubscript{$\pm$ 0.5}  & 91.8\textsubscript{$\pm$ 0.6}  & 86.2\textsubscript{$\pm$ 0.4}  & 77.7\textsubscript{$\pm$ 0.5}  & 64.0\textsubscript{$\pm$ 0.7}  & 39.8\textsubscript{$\pm$ 0.7}  & 11.4\textsubscript{$\pm$ 0.3}  & 5.0\textsubscript{$\pm$ 0.3} \\
        & 500 & 2.125\textsubscript{$\pm$ 0.005}  & 95.1\textsubscript{$\pm$ 0.5}  & 91.8\textsubscript{$\pm$ 0.5}  & 86.2\textsubscript{$\pm$ 0.3}  & 77.8\textsubscript{$\pm$ 0.4}  & 64.3\textsubscript{$\pm$ 0.6}  & 40.7\textsubscript{$\pm$ 0.7}  & 11.5\textsubscript{$\pm$ 0.4}  & 4.8\textsubscript{$\pm$ 0.2} \\
        & 1000 & 2.126\textsubscript{$\pm$ 0.005}  & 95.1\textsubscript{$\pm$ 0.5}  & 91.7\textsubscript{$\pm$ 0.6}  & 86.2\textsubscript{$\pm$ 0.3}  & 77.8\textsubscript{$\pm$ 0.4}  & 64.4\textsubscript{$\pm$ 0.6}  & 40.8\textsubscript{$\pm$ 0.7}  & 11.6\textsubscript{$\pm$ 0.3}  & 4.8\textsubscript{$\pm$ 0.2} \\
        & 5'000 & 2.126\textsubscript{$\pm$ 0.005}  & 95.1\textsubscript{$\pm$ 0.5}  & 91.7\textsubscript{$\pm$ 0.6}  & 86.2\textsubscript{$\pm$ 0.3}  & 77.8\textsubscript{$\pm$ 0.4}  & 64.3\textsubscript{$\pm$ 0.6}  & 40.8\textsubscript{$\pm$ 0.8}  & 11.6\textsubscript{$\pm$ 0.4}  & 4.8\textsubscript{$\pm$ 0.2} \\
        & 10'000 & 2.126\textsubscript{$\pm$ 0.005}  & 95.1\textsubscript{$\pm$ 0.5}  & 91.7\textsubscript{$\pm$ 0.6}  & 86.2\textsubscript{$\pm$ 0.3}  & 77.8\textsubscript{$\pm$ 0.4}  & 64.3\textsubscript{$\pm$ 0.6}  & 40.8\textsubscript{$\pm$ 0.8}  & 11.6\textsubscript{$\pm$ 0.3}  & 4.8\textsubscript{$\pm$ 0.2} \\
	\bottomrule
	\end{tabular}
}
\vspace{-3mm}
\end{table}

\begin{table}[tp]
	\centering
	\small
	\centering
	\caption{We compare the performance of models for different number of discretization sizes with respect to average certified radius (ACR) and given certified accuracies (CA) [\%] on \breast. We utilize $\lambda=2.0$ for $\ell_1$ and $\sigma=4.0$ for $\ell_2$. We report mean and standard deviation over $5$-fold cross-validation. Larger mean is better.}
	\label{tab:ablation_discretizations_tabular}
	\vspace{2mm}
	\resizebox{0.9\columnwidth}{!}{
    \begin{tabular}{ccccccccc}
        \toprule
        \multirow{2.6}{*}{Norm} & \multirow{2.6}{*}{Discretizations} & \multirow{2.6}{*}{ACR} & \multicolumn{6}{c}{Certified Accuracy [\%] at Radius r} \\
        \cmidrule(lr){4-9}
        & & & 0.00 & 0.10 & 0.2 & 0.3 & 0.4 & 0.5 \\
        \midrule
        \multirow{12.0}{*}{$\ell_1$}& 2 &  1.396\textsubscript{$\pm$ 0.039}  & 95.2\textsubscript{$\pm$ 1.4}  & 94.1\textsubscript{$\pm$ 1.4}  & 92.7\textsubscript{$\pm$ 1.4}  & 90.6\textsubscript{$\pm$ 1.8}  & 89.0\textsubscript{$\pm$ 1.9}  & 87.1\textsubscript{$\pm$ 0.8} \\
        & 3 & 1.298\textsubscript{$\pm$ 0.044}  & 65.0\textsubscript{$\pm$ 2.2}  & 65.0\textsubscript{$\pm$ 2.2}  & 65.0\textsubscript{$\pm$ 2.2}  & 65.0\textsubscript{$\pm$ 2.2}  & 65.0\textsubscript{$\pm$ 2.2}  & 65.0\textsubscript{$\pm$ 2.2} \\
        & 5 & 1.292\textsubscript{$\pm$ 0.045}  & 67.6\textsubscript{$\pm$ 2.8}  & 66.9\textsubscript{$\pm$ 2.7}  & 66.3\textsubscript{$\pm$ 2.3}  & 65.9\textsubscript{$\pm$ 2.5}  & 65.4\textsubscript{$\pm$ 2.0}  & 65.0\textsubscript{$\pm$ 2.2} \\
        & 10 & 1.395\textsubscript{$\pm$ 0.038}  & 95.2\textsubscript{$\pm$ 1.4}  & 94.1\textsubscript{$\pm$ 1.4}  & 92.8\textsubscript{$\pm$ 1.2}  & 90.8\textsubscript{$\pm$ 1.8}  & 89.0\textsubscript{$\pm$ 1.9}  & 87.1\textsubscript{$\pm$ 0.8}\\
        & 50 & 1.396\textsubscript{$\pm$ 0.039}  & 95.2\textsubscript{$\pm$ 1.4}  & 94.1\textsubscript{$\pm$ 1.4}  & 92.7\textsubscript{$\pm$ 1.4}  & 90.6\textsubscript{$\pm$ 1.8}  & 89.0\textsubscript{$\pm$ 1.9}  & 87.1\textsubscript{$\pm$ 0.8}\\
        & 100 & 1.396\textsubscript{$\pm$ 0.039}  & 95.2\textsubscript{$\pm$ 1.4}  & 94.1\textsubscript{$\pm$ 1.4}  & 92.7\textsubscript{$\pm$ 1.4}  & 90.6\textsubscript{$\pm$ 1.8}  & 89.0\textsubscript{$\pm$ 1.9}  & 87.1\textsubscript{$\pm$ 0.8}\\
        & 500 & 1.396\textsubscript{$\pm$ 0.039}  & 95.2\textsubscript{$\pm$ 1.4}  & 94.1\textsubscript{$\pm$ 1.4}  & 92.7\textsubscript{$\pm$ 1.4}  & 90.6\textsubscript{$\pm$ 1.8}  & 89.0\textsubscript{$\pm$ 1.9}  & 87.1\textsubscript{$\pm$ 0.8}\\
        & 1'000 & 1.396\textsubscript{$\pm$ 0.039}  & 95.2\textsubscript{$\pm$ 1.4}  & 94.1\textsubscript{$\pm$ 1.4}  & 92.7\textsubscript{$\pm$ 1.4}  & 90.6\textsubscript{$\pm$ 1.8}  & 89.0\textsubscript{$\pm$ 1.9}  & 87.1\textsubscript{$\pm$ 0.8}\\
        & 5'000 & 1.396\textsubscript{$\pm$ 0.039}  & 95.2\textsubscript{$\pm$ 1.4}  & 94.1\textsubscript{$\pm$ 1.4}  & 92.7\textsubscript{$\pm$ 1.4}  & 90.6\textsubscript{$\pm$ 1.8}  & 89.0\textsubscript{$\pm$ 1.9}  & 87.1\textsubscript{$\pm$ 0.8} \\
        & 10'000 & 1.396\textsubscript{$\pm$ 0.039}  & 95.2\textsubscript{$\pm$ 1.4}  & 94.1\textsubscript{$\pm$ 1.4}  & 92.7\textsubscript{$\pm$ 1.4}  & 90.6\textsubscript{$\pm$ 1.8}  & 89.0\textsubscript{$\pm$ 1.9}  & 87.1\textsubscript{$\pm$ 0.8} \\
        & 50'000 & 1.396\textsubscript{$\pm$ 0.039}  & 95.2\textsubscript{$\pm$ 1.4}  & 94.1\textsubscript{$\pm$ 1.4}  & 92.7\textsubscript{$\pm$ 1.4}  & 90.6\textsubscript{$\pm$ 1.8}  & 89.0\textsubscript{$\pm$ 1.9}  & 87.1\textsubscript{$\pm$ 0.8} \\
        & 100'000 & 1.396\textsubscript{$\pm$ 0.039}  & 95.2\textsubscript{$\pm$ 1.4}  & 94.1\textsubscript{$\pm$ 1.4}  & 92.7\textsubscript{$\pm$ 1.4}  & 90.6\textsubscript{$\pm$ 1.8}  & 89.0\textsubscript{$\pm$ 1.9}  & 87.1\textsubscript{$\pm$ 0.8} \\
        \midrule
        \multirow{12.0}{*}{$\ell_2$} & 2 & 20.672\textsubscript{$\pm$ 0.704}  & 65.0\textsubscript{$\pm$ 2.2}  & 65.0\textsubscript{$\pm$ 2.2}  & 65.0\textsubscript{$\pm$ 2.2}  & 65.0\textsubscript{$\pm$ 2.2}  & 65.0\textsubscript{$\pm$ 2.2}  & 65.0\textsubscript{$\pm$ 2.2}\\
        & 3 & 1.533\textsubscript{$\pm$ 0.080}  & 65.0\textsubscript{$\pm$ 2.2}  & 65.0\textsubscript{$\pm$ 2.2}  & 65.0\textsubscript{$\pm$ 2.2}  & 65.0\textsubscript{$\pm$ 2.2}  & 65.0\textsubscript{$\pm$ 2.2}  & 65.0\textsubscript{$\pm$ 2.2}\\
        & 5 & 1.533\textsubscript{$\pm$ 0.080}  & 65.0\textsubscript{$\pm$ 2.2}  & 65.0\textsubscript{$\pm$ 2.2}  & 65.0\textsubscript{$\pm$ 2.2}  & 65.0\textsubscript{$\pm$ 2.2}  & 65.0\textsubscript{$\pm$ 2.2}  & 65.0\textsubscript{$\pm$ 2.2}\\
        & 10 & 20.672\textsubscript{$\pm$ 0.704}  & 65.0\textsubscript{$\pm$ 2.2}  & 65.0\textsubscript{$\pm$ 2.2}  & 65.0\textsubscript{$\pm$ 2.2}  & 65.0\textsubscript{$\pm$ 2.2}  & 65.0\textsubscript{$\pm$ 2.2}  & 65.0\textsubscript{$\pm$ 2.2}\\
        & 50 & 0.644\textsubscript{$\pm$ 0.158}  & 89.2\textsubscript{$\pm$ 1.2}  & 80.5\textsubscript{$\pm$ 4.7}  & 66.7\textsubscript{$\pm$ 16.8}  & 59.8\textsubscript{$\pm$ 20.1}  & 56.8\textsubscript{$\pm$ 18.2}  & 54.5\textsubscript{$\pm$ 17.2}\\
        & 100 & 0.653\textsubscript{$\pm$ 0.075}  & 93.3\textsubscript{$\pm$ 5.7}  & 90.6\textsubscript{$\pm$ 7.4}  & 88.3\textsubscript{$\pm$ 8.2}  & 85.2\textsubscript{$\pm$ 8.0}  & 80.6\textsubscript{$\pm$ 7.9}  & 73.8\textsubscript{$\pm$ 5.5} \\
        & 500 & 0.624\textsubscript{$\pm$ 0.020}  & 95.9\textsubscript{$\pm$ 1.7}  & 93.9\textsubscript{$\pm$ 1.9}  & 92.4\textsubscript{$\pm$ 2.5}  & 89.3\textsubscript{$\pm$ 2.9}  & 83.9\textsubscript{$\pm$ 1.8}  & 77.3\textsubscript{$\pm$ 3.6}\\
        & 1000 & 0.609\textsubscript{$\pm$ 0.015}  & 96.1\textsubscript{$\pm$ 1.3}  & 94.8\textsubscript{$\pm$ 1.8}  & 92.7\textsubscript{$\pm$ 2.1}  & 90.0\textsubscript{$\pm$ 1.3}  & 85.9\textsubscript{$\pm$ 2.0}  & 77.9\textsubscript{$\pm$ 2.2}\\
        & 5'000 & 0.597\textsubscript{$\pm$ 0.015}  & 96.7\textsubscript{$\pm$ 1.7}  & 95.0\textsubscript{$\pm$ 1.7}  & 92.8\textsubscript{$\pm$ 1.8}  & 91.1\textsubscript{$\pm$ 1.4}  & 86.3\textsubscript{$\pm$ 2.1}  & 76.6\textsubscript{$\pm$ 3.5} \\
        & 10'000 & 0.598\textsubscript{$\pm$ 0.017}  & 96.5\textsubscript{$\pm$ 1.8}  & 95.0\textsubscript{$\pm$ 1.4}  & 92.8\textsubscript{$\pm$ 1.8}  & 91.1\textsubscript{$\pm$ 1.4}  & 86.3\textsubscript{$\pm$ 2.7}  & 76.3\textsubscript{$\pm$ 3.1} \\
        & 50'000 & 0.597\textsubscript{$\pm$ 0.016}  & 96.7\textsubscript{$\pm$ 1.7}  & 94.9\textsubscript{$\pm$ 1.6}  & 92.8\textsubscript{$\pm$ 1.8}  & 91.2\textsubscript{$\pm$ 1.6}  & 86.2\textsubscript{$\pm$ 2.9}  & 76.6\textsubscript{$\pm$ 3.5} \\
        & 100'000 & 0.597\textsubscript{$\pm$ 0.016}  & 96.7\textsubscript{$\pm$ 1.7}  & 94.9\textsubscript{$\pm$ 1.6}  & 92.8\textsubscript{$\pm$ 1.8}  & 91.2\textsubscript{$\pm$ 1.6}  & 86.2\textsubscript{$\pm$ 2.9}  & 76.6\textsubscript{$\pm$ 3.5} \\
        \bottomrule
	\end{tabular}
}
\vspace{-3mm}
\end{table}

In the main paper, all experiments are conducted with leaf predictions discretized to $\Delta = 100$ values to enable our efficient CDF computation.
In this section, we investigate the effect of this discretization.
Concretely, we report results on \mnistof and \breast using a range of discretization-granularities from $2$ to $10\,000$ and $2$ to $100\,000$ in \cref{tab:ablation_discretizations} and \cref{tab:ablation_discretizations_tabular}, respectively.
While using a very coarse discretization can lead to a mode collapse (explaining the very high ACRs observed for $\ell_2$ perturbations in \cref{tab:ablation_discretizations_tabular}) and generally degraded performance, we observe that for sufficiently fine discretizations (typically $\Delta\geq50$) the results converge as the discretization is refined further. As the discretization effect on the ensemble's output is bounded by $\frac{M}{2\Delta}$, we conclude that these fine discretizations closely approximate the non-discretized case. We choose $\Delta = 100$ such that our discretized smoothed models generally approximately recover the behavior of the non-discretized models while allowing for fast computations of the ensemble PDF. 

While performance improves monotonically with finer discretizations for \mnistof in the $\ell_2$ setting and for \breast in the $\ell_1$ setting, it seems to peak and then declines again for \mnistof in the $\ell_1$ setting and for \breast in the $\ell_2$ setting. For \breast, we observe significantly larger standard deviations at coarse discretizations leading to overlapping $\pm 1$ standard deviation intervals for all discretization levels not suffering from a mode collapse and thus statistically insignificant results. 
For \mnistof in the $\ell_1$, the performance peak at $\Delta=50$ is statistically significant. We hypothesize that the coarser regularizations have a beneficial regularizing effect in this setting.

\subsection{Split Position Search Granularity}
\label{app:binning-size-ablation}

\begin{table}[tp]
	\centering
	\small
	\centering
	\caption{We compare the performance of models for different binning sizes with respect to average certified radius (ACR) and given certified accuracies (CA) [\%] on \mnistof. We utilize $\lambda=4.0$ for $\ell_1$ and $\sigma=0.5$ for $\ell_2$. Larger is better.}
	\label{tab:ablation-binning-size}
	\vspace{2mm}
	\resizebox{0.8\columnwidth}{!}{
    \begin{tabular}{ccccccccccc}
        \toprule
        \multirow{2.6}{*}{Norm} & \multirow{2.6}{*}{Binning Size} & \multirow{2.6}{*}{ACR} & \multicolumn{8}{c}{Certified Accuracy [\%] at Radius r} \\
        \cmidrule(lr){4-11}
        & & & 0.00 & 0.50 & 1.00 & 1.50 & 2.00 & 2.50 & 3.00 & 3.50 \\
        \midrule
        \multirow{11.0}{*}{$\ell_1$} & 4.0 & 3.452 & 96.2 & 95.2 & 93.6 & 91.7 & 89.4 & 86.3 & 82.7 & 75.3  \\
        & 2.0 & 3.452 & 96.2 & 95.2 & 93.6 & 91.7 & 89.4 & 86.3 & 82.7 & 75.3 \\
        & 1.0 & 3.468 & 96.5 & 95.6 & 94.1 & 92.0 & 89.9 & 86.5 & 83.1 & 75.2  \\
        & 0.5 & 3.465 & 96.6 & 95.5 & 94.2 & 91.9 & 89.9 & 86.6 & 83.1 & 75.0 \\
        & 0.1 & 3.462 & 96.6 & 95.5 & 94.2 & 92.0 & 89.9 & 86.4 & 83.0 & 74.8 \\
        & 0.05 & 3.466 & 96.5 & 95.6 & 94.1 & 91.9 & 89.9 & 86.5 & 83.1 & 75.1 \\
        & 0.01 & 3.467 & 96.6 & 95.6 & 94.1 & 92.1 & 89.9 & 86.5 & 83.1 & 75.1 \\
        & 0.005 & 3.467 & 96.6 & 95.6 & 94.1 & 92.1 & 89.9 & 86.5 & 83.1 & 75.1 \\
        & 0.001 & 3.467 & 96.5 & 95.6 & 94.1 & 92.1 & 90.0 & 86.5 & 83.1 & 75.1 \\
        & 0.0005 & 3.466 & 96.5 & 95.5 & 94.1 & 92.1 & 90.0 & 86.5 & 83.0 & 75.1 \\
        & 0.0001 & 3.467 & 96.5 & 95.5 & 94.2 & 92.1 & 89.9 & 86.5 & 83.1 & 75.1 \\
        \midrule
        \multirow{11.0}{*}{$\ell_2$} & 4.0 & 0.584 & 56.0 & 55.9 & 31.5 & 0.0 & 0.0 & 0.0 & 0.0 & 0.0 \\
        & 2.0 & 1.888 & 95.4 & 91.3 & 83.6 & 73.9 & 55.7 & 22.4 & 4.0 & 0.8 \\
        & 1.0 & 1.980 & 95.7 & 92.0 & 84.1 & 74.0 & 58.8 & 35.2 & 4.8 & 0.6 \\
        & 0.5 & 2.119 & 95.8 & 93.0 & 87.2 & 78.5 & 62.6 & 35.7 & 11.5 & 6.5 \\
        & 0.1 & 2.161 & 96.0 & 93.1 & 87.5 & 79.1 & 65.2 & 40.8 & 12.3 & 5.9 \\
        & 0.05 & 2.160 & 96.0 & 93.1 & 87.5 & 79.1 & 65.3 & 41.4 & 12.4 & 5.8 \\
        & 0.01 & 2.161 & 96.0 & 93.0 & 87.5 & 79.0 & 65.3 & 40.5 & 12.3 & 5.9 \\
        & 0.005 & 2.163 & 96.0 & 93.1 & 87.5 & 79.0 & 65.3 & 40.9 & 12.4 & 5.9 \\
        & 0.001 & 2.164 & 96.0 & 93.0 & 87.5 & 79.0 & 65.3 & 41.1 & 12.5 & 5.8 \\
        & 0.0005 & 2.164 & 96.0 & 93.0 & 87.5 & 79.0 & 65.3 & 41.2 & 12.5 & 5.8 \\
        & 0.0001 & 2.163 & 96.0 & 93.0 & 87.5 & 79.0 & 65.3 & 41.1 & 12.5 & 5.8 \\
		\bottomrule
	\end{tabular}
}
\end{table}

In our main paper, all experiments are conducted using a step size of $0.01$ to conduct the line search for the optimal split position $v_m$.
In \cref{tab:ablation-binning-size}, we report results for search granularities from $4.0$ to $10^{-4}$ and observe that a step size of $0.1$ is sufficiently fine and reducing it further does not improve the performance of the obtained models.
This suggest that our approximate optimization based on line search comes very close to the finding the true optimal split position and thus jointly MLE optimal $v_m$ and $ \gamma$.

\subsection{Error Bounds}
\label{app:error-bounds}

\begin{table}[tp]
	\centering
	\small
	\centering
	\caption{Average certified accuracy (ACR) Certified accuracy (CA) $[\%]$ at various radii with respect to $\ell_1$- and $\ell_2$-norm bounded perturbations on various datasets. Larger is better.}
	\label{tab:error-bounds-extensive}
	\vspace{2mm}
	\resizebox{0.99\columnwidth}{!}{
    \begin{tabular}{cccccccc}
        \toprule
        \multirow{2.5}{*}{Perturbation} & \multirow{2.5}{*}{Dataset} & \multirow{2.5}{*}{ACR} & \multicolumn{5}{c}{Radius $r$}\\
        \cmidrule(lr){4-8}
        & & & 0.00 & 0.10 & 0.25 & 0.50 & 1.00 \\
        \midrule
        \multirow{6.5}{*}{$\ell_1$} & \breast & 1.396\textsubscript{$\pm$ 0.039}  & 95.2\textsubscript{$\pm$ 1.4}  & 94.1\textsubscript{$\pm$ 1.4}  & 92.1\textsubscript{$\pm$ 1.3}  & 87.1\textsubscript{$\pm$ 0.8}  & 72.8\textsubscript{$\pm$ 2.8} \\
        & \diabetes  &  0.153\textsubscript{$\pm$ 0.010}  & 73.7\textsubscript{$\pm$ 3.1}  & 58.3\textsubscript{$\pm$ 2.9}  & 30.1\textsubscript{$\pm$ 4.0}  & 0.0\textsubscript{$\pm$ 0.0}  & 0.0\textsubscript{$\pm$ 0.0} \\
        & \spambase  & 2.541\textsubscript{$\pm$ 0.042}  & 89.1\textsubscript{$\pm$ 0.5}  & 88.2\textsubscript{$\pm$ 0.4}  & 87.4\textsubscript{$\pm$ 0.6}  & 85.2\textsubscript{$\pm$ 0.8}  & 80.6\textsubscript{$\pm$ 1.2} \\
        & \fmnists & 2.731\textsubscript{$\pm$ 0.027}  & 84.4\textsubscript{$\pm$ 0.8}  & 83.8\textsubscript{$\pm$ 0.8}  & 83.1\textsubscript{$\pm$ 0.8}  & 81.7\textsubscript{$\pm$ 0.8}  & 79.1\textsubscript{$\pm$ 0.7} \\
        & \mnistof & 3.425\textsubscript{$\pm$ 0.015}  & 96.2\textsubscript{$\pm$ 0.4}  & 95.9\textsubscript{$\pm$ 0.5}  & 95.4\textsubscript{$\pm$ 0.5}  & 94.7\textsubscript{$\pm$ 0.5}  & 93.0\textsubscript{$\pm$ 0.5} \\
        & \mnistts &  3.243\textsubscript{$\pm$ 0.008}  & 95.7\textsubscript{$\pm$ 0.3}  & 95.5\textsubscript{$\pm$ 0.3}  & 95.1\textsubscript{$\pm$ 0.3}  & 94.5\textsubscript{$\pm$ 0.3}  & 92.8\textsubscript{$\pm$ 0.3} \\
        \midrule
        \multirow{6.5}{*}{$\ell_2$} & \breast & 0.653\textsubscript{$\pm$ 0.075}  & 93.3\textsubscript{$\pm$ 5.7}  & 90.6\textsubscript{$\pm$ 7.4}  & 87.3\textsubscript{$\pm$ 8.6}  & 73.8\textsubscript{$\pm$ 5.5}  & 15.5\textsubscript{$\pm$ 21.5} \\
        & \diabetes  & 0.124\textsubscript{$\pm$ 0.005}  & 72.7\textsubscript{$\pm$ 3.5}  & 53.0\textsubscript{$\pm$ 3.4}  & 15.6\textsubscript{$\pm$ 2.6}  & 0.0\textsubscript{$\pm$ 0.0}  & 0.0\textsubscript{$\pm$ 0.0} \\
        & \spambase  & 0.884\textsubscript{$\pm$ 0.006}  & 89.7\textsubscript{$\pm$ 1.0}  & 87.4\textsubscript{$\pm$ 1.2}  & 83.7\textsubscript{$\pm$ 1.1}  & 73.6\textsubscript{$\pm$ 1.0}  & 40.9\textsubscript{$\pm$ 0.8} \\
        & \fmnists & 1.334\textsubscript{$\pm$ 0.012}  & 85.0\textsubscript{$\pm$ 0.8}  & 83.7\textsubscript{$\pm$ 0.9}  & 81.5\textsubscript{$\pm$ 0.7}  & 78.1\textsubscript{$\pm$ 0.5}  & 68.6\textsubscript{$\pm$ 0.7} \\
        & \mnistof & 1.720\textsubscript{$\pm$ 0.006}  & 95.3\textsubscript{$\pm$ 0.4}  & 94.8\textsubscript{$\pm$ 0.4}  & 94.0\textsubscript{$\pm$ 0.5}  & 92.3\textsubscript{$\pm$ 0.6}  & 87.9\textsubscript{$\pm$ 0.3} \\
        & \mnistts & 1.613\textsubscript{$\pm$ 0.007}  & 95.5\textsubscript{$\pm$ 0.2}  & 94.9\textsubscript{$\pm$ 0.3}  & 93.9\textsubscript{$\pm$ 0.2}  & 91.7\textsubscript{$\pm$ 0.2}  & 84.9\textsubscript{$\pm$ 0.7} \\
		\bottomrule
	\end{tabular}
}
\vspace{-4mm}
\end{table}

In \cref{tab:error-bounds-extensive} we report the mean and standard deviation of the certified accuracies at various radii across a 5-fold cross validation for datasets including only numerical features.
We observe, that our results are very stable with standard deviations of less than $1.0\%$ on the computer vision datasets, which have large sample sizes.
On the tabular datasets, which consist of much fewer samples, the dependence on the train/test-split is slightly larger with standard deviations reaching around $4.0\%$ in some settings. 
Only where a large noise magnitude ($\sigma=4$), very small sample sizes, and large radii come together (\breast for $\ell_2$ perturbations of $r=1.0$) do we observe a large sensitivity to the train/test-split and thus a high standard deviations of up to $21\%$.

\section{(De-)Randomized Smoothing for Decision Tree Ensembles}
\label{sec:appendix-trees}
While we focus on decision stump ensembles in the main paper and in particular in \cref{sec:det_smoothing}, our approach can easily be extended to ensembles of decision trees with arbitrary depths which do not use the same features in distinct decision trees.
In particular, our approach can be easily extended to arbitrary individual decision trees.

Recall that the key idea of our approach is to group individual decision stumps into independent meta-stumps, allowing us to represent the output of the overall smoothed ensemble as the sum of independent terms.
We can apply the same idea here by constructing meta-stumps over all features used in an individual tree.
As we do not permit features to be reused in multiple trees, every tree is independent of all others.

For every leaf $j$ of a decision tree $m$ with prediction $\gamma_{m,j}$, we can accumulate the constraints along the path from the root to the leaf of that tree as 
\begin{equation}
	\psi_j(\vx) = \bigwedge_i x_i > v_{j,i}^- \wedge x_i \leq v_{j,i}^+.
\end{equation}
Note that if $x_i$ is not constrained (in one direction) on the path to leaf $j$, we can simply set the corresponding threshold $v^{\{+,-\}}$ to $\pm \infty$.
This allows us to formally define a smoothed decision tree as
\begin{equation}
	g_m(\vx) = \sum_j \gamma_{m,j} \, \prod_{i} \P_{\vx_i' \sim \dist(\vx_i)}[v_{j,i}^- < x'_i \leq v_{j,i}^+].
\end{equation}

As all features are perturbed independently under the randomization scheme $\dist$, we can compute the probability of a perturbed sample satisfying $\psi_j$ and thus landing in leaf $j$ by factorization as 
\begin{align}
	\P_{\vx' \sim \dist(\vx)}[\psi_j(\vx')] = &\prod_{i} \underbrace{\P_{\vx' \sim \dist(\vx)}[x'_i > v_{j,i}^- \wedge x'_i \leq v_{j,i}^+]}_{p_{j,i}} \\
	= &\prod_{i} \P_{\vx' \sim \dist(\vx)}[x'_i \leq v_{j,i}^+] - \P_{\vx' \sim \dist(\vx)}[x'_i \leq v_{j,i}^-].
\end{align}
For many common randomization schemes where the (dimension-wise) CDF is available, this expression can be evaluated efficiently in closed form, e.g., when using a Gaussian distribution as the randomization scheme $\dist(\vx) = \bc{N}(\vx, \sigma \mI)$, typically used for $\ell_2$-norm certificates (see \cref{tab:rs}), and given the inverse Gaussian CDF $\Phi^{-1}$, we obtain
\begin{equation}
	\P_{\vx' \sim \dist(\vx)}[\psi_j(\vx')] = \prod_{i} \Phi^{-1}\left(\frac{v_{j,i}^+ - x_i}{\sigma}\right) - \Phi^{-1}\left(\frac{v_{j,i}^- - x_i}{\sigma}\right).
\end{equation}
We can now construct a meta-stump equivalent per decision tree, where the piece-wise constant regions with output $\gamma_{m,j}$ are now simply defined over multiple variables instead of over a single variable. 
We illustrate this in \cref{fig:tree}, where we show a decision tree (\cref{fig:tree}a)) on features $x_1$ and $x_2$ (partially truncating depth to avoid clutter) and the corresponding output landscape (\cref{fig:tree}b)). We can now compute the probability of $\vx' \sim \dist(\mathbf{0})$ falling into the blue region as the product of the probabilities of $x'_1$ lying in $[0.5, 1)$, $p_{j,1}$, and $x'_2$ lying in $[0.5, \infty)$, $p_{j,2}$. Proceeding similarly for the other regions, we can instantiate \cref{alg:dp}, as for regular meta-stumps, only replacing the probability computation as discussed above.

This allows us to adapt \cref{alg:dp} to iterate over the ensembled decision trees instead of over features.

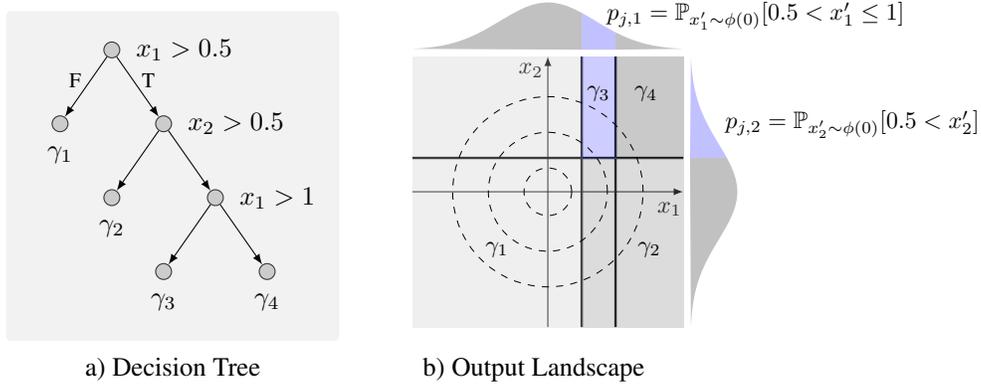
\begin{figure}
	\centering
	\begin{tikzpicture}
	\tikzset{>=latex}

	\def\x0{-3.5}
	\def\xx0{-0.0}
	\def\yy0{2.5}
	\def\dx{4.3}
	\def\ddx{1.5}
	\def\dddx{2.0}
	\def\ddy{0.9}

	\node (tree) [ fill=black!05, rectangle,
	minimum width=1.2cm, minimum height=0.80cm, align=center, scale=1.0, rounded corners=2pt,
	anchor=north] at (0, 0) {
		\begin{tikzpicture}[scale=1.0]
		\node (root)[draw=black!80, fill=black!20, circle, minimum size=6pt, inner sep=0pt, anchor=center] at (0, -1.0) {};	
		\node (A)[draw=black!80, fill=black!20, circle, minimum size=6pt, inner sep=0pt, anchor=center] at ($(-55:1.2) + (root)$) {};	
		\node (B)[draw=black!80, fill=black!20, circle, minimum size=6pt, inner sep=0pt, anchor=center] at ($(-125:1.2) + (root)$) {};	
		\node (AA)[draw=black!80, fill=black!20, circle, minimum size=6pt, inner sep=0pt, anchor=center] at ($(-55:1.2) + (A)$) {};
		\node (AB)[draw=black!80, fill=black!20, circle, minimum size=6pt, inner sep=0pt, anchor=center] at ($(-125:1.2) + (A)$) {};
		\node (AAA)[draw=black!80, fill=black!20, circle, minimum size=6pt, inner sep=0pt, anchor=center] at ($(-55:1.2) + (AA)$) {};
		\node (AAB)[draw=black!80, fill=black!20, circle, minimum size=6pt, inner sep=0pt, anchor=center] at ($(-125:1.2) + (AA)$) {};
		
		\draw[->]  (root) -- (A);
		\draw[->]  (root) -- (B);
		\draw[->]  (A) -- (AB);
		\draw[->]  (A) -- (AA);
		\draw[->]  (AA) -- (AAA);
		\draw[->]  (AA) -- (AAB);
		
		\node ()[anchor=west ] at ($(root) + (0.2,0)$) {$x_1 > 0.5$};
		\node ()[anchor=west ] at ($(A) + (0.2,0)$) {$x_2 > 0.5$};
		\node ()[anchor=west ] at ($(AA) + (0.2,0)$) {$x_1 > 1$};

		\node ()[anchor=west, scale=0.8] at ($(root) + (0.0,-0.4)$) {T};
		\node ()[anchor=east, scale=0.8] at ($(root) + (0.0,-0.4)$) {F};
		
		\node ()[anchor=north] at ($(B) + (0.0,0.0)$) {$\gamma_1$};
		\node ()[anchor=north] at ($(AB) + (0.0,0.0)$) {$\gamma_2$};
		\node ()[anchor=north] at ($(AAB) + (0.0,0.0)$) {$\gamma_3$};
		\node ()[anchor=north] at ($(AAA) + (0.0,0.0)$) {$\gamma_4$};
		\end{tikzpicture}
	};

	\node (region) [rectangle,	minimum width=1.2cm, minimum height=0.80cm, scale=0.9, rounded corners=2pt,
	anchor=north] at (7, 0.4) {
		\begin{tikzpicture}[scale=1.0]

		\def\a{2.0}
		\def\b{0.5}
		\def\c{1.0}
		
		\def\d{2.1}
		
		\def\fb{1/exp(((\b)^2)/2)}
		\def\fc{1/exp(((\c)^2)/2)}
		
		\def\normal{\x,{\d+0.7/exp(((\x)^2)/1)}}
		\fill [fill=black!60,rounded corners=0pt,opacity=0.4] plot[domain=-\a:\b] (\normal) -- ({\b},\d) -- cycle;
		\fill [fill=blue!60,rounded corners=0pt,opacity=0.4] ({\b},\d) -- plot[domain=\b:\c] (\normal) -- ({\c},\d) -- cycle;
		\fill [fill=black!60,rounded corners=0pt,opacity=0.4] ({\c},\d) -- plot[domain=\c:\a] (\normal) -- cycle;

		\def\normaly{{\d+0.7/exp(((\x)^2)/1)}, \x}
		\fill [fill=blue!60,rounded corners=0pt,opacity=0.4] (\d,\b) -- plot[domain=\b:\a] (\normaly) -- cycle;
		\fill [fill=black!60,rounded corners=0pt,opacity=0.4] plot[domain=-\a:\b] (\normaly) -- (\d,\b) -- cycle;
		
		\coordinate (xm) at ({-\a},{0});
		\coordinate (xp) at ({\a},{0});
		\coordinate (ym) at ({0},{-\a});
		\coordinate (yp) at ({0},{\a});
		
		\coordinate (pp) at ({\a},{\a});
		\coordinate (mm) at ({-\a},{-\a});
		\coordinate (mp) at ({-\a},{\a});
		\coordinate (pm) at ({\a},{-\a});
		
		\coordinate (bm) at ({\b},{-\a});
		\coordinate (bp) at ({\b},{\a});

		\coordinate (cm) at ({\c},{-\a});
		\coordinate (cp) at ({\c},{\a});
		
		\coordinate (mb) at ({-\a},{\b});
		\coordinate (pb) at ({\a},{\b});
		
		\coordinate (bb) at ({\b},{\b});
		\coordinate (cb) at ({\c},{\b});

		\draw[->]  (xm) -- (xp);
		\draw[->]  (ym) -- (yp);
		\node ()[] at ($(xp)+(-0.20,0.15)$) {$x_1$};
		\node ()[] at ($(yp)+(-0.25,0.2)$) {$x_2$};

		\draw[-, line width = 1pt]  (cm) -- (cp);
		\draw[-, line width = 1pt]  (bm) -- (bp);
		\draw[-, line width = 1pt]  (mb) -- (pb);

		\fill[fill=black!20,opacity=0.3,rounded corners=0pt] (mp) -- (bp) -- (bm) -- (mm) -- cycle;		
		\fill[fill=black!70,opacity=0.3,rounded corners=0pt] (pp) -- (cp) -- (cb) -- (pb) -- cycle;
		\fill[fill=blue!65,opacity=0.3,rounded corners=0pt] (bp) -- (bb) -- (cb) -- (cp) -- cycle;
		\fill[fill=black!45,opacity=0.3,rounded corners=0pt] (pb) -- (bb) -- (bm) -- (pm) -- cycle;
		
		\draw[style=dashed](0,0) circle (10pt);
		\draw[style=dashed](0,0) circle (25pt);
		\draw[style=dashed](0,0) circle (40pt);

		\node ()[anchor=center] at ($(mp)!0.50!(bm)+(0,-0.85)$) {$\gamma_1$};
		\node ()[anchor=center] at ($(pp)!0.50!(bb)+(0.2,0.2)$) {$\gamma_4$};
		\node ()[anchor=center] at ($(bp)!0.50!(cb)+(0,0.2)$) {$\gamma_3$};
		\node ()[anchor=center] at ($(pb)!0.50!(cm)+(0,-0.1)$) {$\gamma_2$};
		
		\node () [anchor=west] at ($({(\c+\b)/2},{\d + 0.5})$) {$p_{j,1} = \P_{x_1' \sim \dist(0)}[0.5 < x'_1 \leq 1]$};
		\node () [anchor=west] at ($({\d + 0.4},{\b+0.5})$) {$p_{j,2} = \P_{x_2' \sim \dist(0)}[0.5 < x'_2]$};

		\end{tikzpicture}
	};
  
  	\node (labela) [anchor=north] at ($(tree) + (.0,-2.3)$) {a) Decision Tree};
  	\node [anchor=north] at ($(labela.north) + (4.8,0)$) {b) Output Landscape};
  
\end{tikzpicture}
	\caption{Illustration of a meta-stump on a decision tree with feature re-use.}
	\label{fig:tree}
\end{figure}
}{}

\message{^^JLASTPAGE \thepage^^J}

\end{document}